\documentclass[11pt, letterpaper]{article}  

\usepackage[left=1in, right=1in, top=1in, bottom=1in]{geometry}
\usepackage[utf8]{inputenc}
\usepackage[T1]{fontenc}
\usepackage{bm}
\usepackage{type1cm}
\usepackage{lettrine}
\usepackage{moreverb}
\usepackage{mathtools}
\usepackage{amsmath,amssymb,amsfonts}%
\usepackage{amsthm}%
\usepackage{algorithm}
\usepackage{algorithmicx}
\usepackage{algpseudocode}
\usepackage{graphics}
\usepackage{graphicx}
\usepackage{subfigure}
\usepackage{caption}
\usepackage{extarrows}
\usepackage{color}
\usepackage{framed}
\usepackage{wrapfig}
\usepackage{bm}
\usepackage{mathrsfs}
\usepackage{mathabx}
\usepackage{multirow}
\usepackage{longtable}
\usepackage{hyperref}
\usepackage{paralist}
\usepackage{indentfirst}
\usepackage{relsize}
\usepackage{extarrows}
\usepackage{upgreek}
\usepackage{bm}
\usepackage{mwe}
\usepackage[dvipsnames,table]{xcolor}
\usepackage{booktabs}
\usepackage{authblk}
\usepackage{lettrine}
\usepackage{type1cm}
\usepackage{threeparttable}
\usepackage[sort&compress,numbers]{natbib}
\usepackage[figurename=Figure]{caption}
\usepackage[normalem]{ulem}

\usepackage{amsthm}

\theoremstyle{plain}
\newtheorem{theorem}{Theorem}

\theoremstyle{definition}

\theoremstyle{remark}

\graphicspath{ {./Figure/} }
\usepackage[font=footnotesize,labelfont=bf]{caption}

\providecommand{\keywords}[1]{\textbf{\textit{Keywords: }} #1}

\hypersetup{
bookmarks=true,
bookmarksopen=true,
bookmarksnumbered=true,
unicode=false,
pdftoolbar=true,
pdfmenubar=true,
pdffitwindow=false,
pdfstartview={FitH},
pdftitle={My title},
pdfauthor={Author},
pdfsubject={Subject},
pdfcreator={Creator},
pdfproducer={Producer},
pdfkeywords={keywords},
pdfnewwindow=true,
colorlinks=true,
linkcolor=blue,
citecolor=blue,
filecolor=blue,
urlcolor=blue
}

\begin{document}

\title{\textbf{Optimization and Generation in Aerodynamics Inverse Design}}

\author[1$\dagger$]{Huaguan Chen}
\author[1$\dagger$]{Ning Lin}
\author[1$\dagger$]{Luxi Chen}
\author[]{Jiacheng Cen}
\author[]{Rui Zhang}
\author[]{Wenbing Huang}
\author[]{Chongxuan Li}
\author[1$*$]{Hao Sun\vspace{6pt}}

\affil[]{\small Gaoling School of Artificial Intelligence, Renmin University of China, Beijing, China\vspace{18pt}}
\affil[$^\dagger$]{Equal contribution}
\affil[*]{Corresponding authors}

\date{}

\maketitle

\normalsize

\vspace{-12pt}

\begin{abstract}
	\small 
Aerodynamic inverse design is important for improving the efficiency of vehicles and aircraft, but practical design rarely seeks performance alone. In vehicle development, aerodynamic refinement should reduce drag while preserving visual features, because these features are closely tied to design language, brand recognition and user perception. Traditional CFD-driven optimization can improve aerodynamic objectives, but repeated flow solves, mesh updates and geometry correction make it too slow for broad design exploration. Deep learning offers a faster route, yet two issues remain unresolved: existing methods are still mainly driven by aerodynamic performance and lack an explicit mechanism for preserving visual features during optimization; and learning-based inverse design still lacks a coherent theoretical target linking optimization, generation and visual consistency. We address these limitations by formulating visual preservation and aerodynamic improvement as one probability target. Designs visually consistent with a reference shape or view define a learned visual design distribution, and aerodynamic cost reweights this distribution towards higher-performance regions. Under this cost-biased distribution, optimization performs gradient-based refinement from an initial geometry, whereas guided generation samples lower-cost 3D candidates from the same input view. 
OpenFOAM evaluation shows that visual-feature-preserving optimization reduces vehicle drag by 5.8\% relative to the initial vehicle and reduces the best valid aircraft drag-to-lift objective by 28.8\% relative to the initial aircraft, while preserving the input visual features. For view-based guided generation, the framework reduces vehicle drag by 3.0\% and the aircraft drag-to-lift objective by 68.6\% relative to direct generation from the same view, while maintaining visual consistency with the input concept. Wind-tunnel tests with 3D-printed vehicle prototypes provide an independent wake-level check, and controlled analyses explain the distributional mechanisms behind these improvements. We hope this work provides a theoretical foundation and practical route for visual-feature-preserving aerodynamic refinement and early-stage 3D design exploration.
\end{abstract}


\vspace{12pt} 

\section*{Introduction}

Aerodynamic performance is a central factor in the efficiency, range and usability of modern transportation systems. For road vehicles, aerodynamic drag becomes a major source of resistance at highway speeds and directly affects energy consumption, driving range and wind noise \cite{hucho2013aerodynamics,national2015cost}. For aircraft, lift and drag are among the primary quantities that determine cruise efficiency, payload capability and operational range \cite{anderson2011ebook,raymer2012aircraft}. Because vehicles and aircraft are deployed at large scale and operated over long lifetimes, even modest aerodynamic improvements can accumulate into substantial engineering, economic and environmental benefits \cite{national2015cost}. Aerodynamic design is therefore not only a classical problem in fluid mechanics, but also a practical route towards more energy-efficient transportation.

Finding aerodynamically optimal shapes is already a challenging high-dimensional problem \cite{martins2022aerodynamic}. Yet in practical vehicle design, the objective is rarely to pursue the lowest possible drag at the cost of altering the vehicle's visual character. Preserving the exterior appearance is often an equally, if not more, important design goal, because vehicle appearance is closely tied to brand identity, design lineage, market positioning, and users' visual preferences \cite{ranscombe2011characterizing}. In many real-world design processes, manufacturers may even accept a certain aerodynamic penalty in order to maintain a recognizable styling language. Therefore, after the overall design language has been established, aerodynamic refinement should not freely replace the original shape with a completely different optimum. A useful inverse-design method should instead improve aerodynamic performance while preserving the visual features and stylistic consistency of the original design.

This setting leads to two complementary tasks (Figure \ref{fig:overall}\textbf{a}). The first is visual-feature-preserving aerodynamic optimization. Given an existing 3D geometry, the goal is to introduce local and plausible modifications that improve an aerodynamic objective, such as drag coefficient or drag-to-lift ratio, while keeping the optimized shape visually close to the original design. This corresponds to late-stage aerodynamic refinement, where performance improvements should not disrupt the established design language. The second task is view-based aerodynamic generation. At the early design stage, vehicle concepts are often explored through renderings or 2D views before a complete three-dimensional model is constructed. A useful generative system should therefore take such a view as input and produce multiple plausible three-dimensional candidates that remain visually consistent with the concept while being biased towards improved aerodynamic performance.

Traditional aerodynamic shape optimization has provided powerful tools for improving engineering designs, but it is not well suited to broad and rapid design exploration \cite{martins2022aerodynamic,li2022machine}. Although adjoint methods and surrogate models can reduce the number of required evaluations \cite{martins2022aerodynamic,papadimitriou2008aerodynamic,queipo2005surrogate}, practical optimization still often involves repeated flow solutions, mesh updates and geometry correction. A single CFD-driven optimization campaign can require on the order of \(10^{2}\)--\(10^{5}\) CPU-hours \cite{elrefaie2024drivaernet,tran2024aerodynamics,aultman2021effects}, making it difficult to assess many candidate shapes within early-stage design cycles \cite{he2018aerodynamic,li2022machine}. In addition, many established workflows are local by construction: they improve a design within a prescribed parameterization around an existing geometry, but do not naturally produce diverse, visually consistent alternatives across a broader design space.

Deep learning provides a highly promising route for aerodynamic inverse design. Compared with traditional CFD-driven optimization, it can reduce the design cost by orders of magnitude, requiring less than \(0.1\) GPU-hours per candidate design in our implementation and enabling hundreds-fold acceleration (Figure~\ref{fig:overall}\textbf{b}).
However, current progress still leaves two key issues unresolved. First, visual-feature-preserving aerodynamic optimization remains largely unsolved. The challenge lies not only in how to quantify visual consistency, but also in how to introduce visual gradients into the optimization process. Existing methods are still mainly driven by aerodynamic performance \cite{tran2024aerodynamics,vatani2025tripoptimizer,hao3did,you2025physgen}, and therefore lack an explicit mechanism for preserving the input visual features. Second, despite the promise of deep learning, its use in aerodynamic inverse design is still at an early methodological stage. Existing formulations are often based on empirical design choices, while a coherent theoretical framework and a corresponding theory-guided implementation paradigm have not yet been established.

Here we address these two limitations from a probabilistic perspective. We treat the set of designs that are visually consistent with a reference shape or view as a learned visual design distribution, and then bias this distribution by aerodynamic cost. The result is a cost-biased distribution: designs have high probability only when they remain tied to the input visual features and have lower aerodynamic cost. Under this view, visual-feature-preserving optimization and view-based generation are not two separate problems. Optimization seeks a high-probability, low-cost point in this distribution, whereas generation samples multiple low-cost candidates from the same distribution. This gives learning-based aerodynamic inverse design a single theoretical target, while making visual preservation part of the target rather than a post-processing check.

This probability target is useful because it turns the design requirement into something the model can act on. A cost predictor alone can push a shape towards lower predicted drag or drag-to-lift ratio, but it has no reason to preserve the input visual features. The learned visual distribution supplies the missing counter-direction: it favours designs that stay close to the reference visual features, while the aerodynamic cost favours improved performance.
The framework therefore gives a principled way to balance the two forces that practical design needs. It also links late-stage refinement and early-stage view-to-3D exploration in one theory, instead of building separate empirical rules for optimization and generation.

We validate the framework on the two design tasks introduced above. For visual-feature-preserving optimization, OpenFOAM \cite{jasak2007openfoam} evaluation shows that the method reduces vehicle drag by 5.8\% while retaining the recognizable vehicle appearance, and reduces the best valid aircraft drag-to-lift objective by 28.8\% while avoiding the negative-lift failures produced by cost-only optimization. For view-based guided generation, OpenFOAM evaluation shows that the proposed guidance strategy reduces vehicle drag by 3.0\% and the aircraft drag-to-lift objective by 68.6\% relative to direct generation from the same input view, showing that the same probabilistic target supports both late-stage refinement and early-stage view-to-3D exploration. Wind-tunnel tests with 3D-printed vehicle prototypes provide an independent physical check at the wake-visualization level. Controlled two-dimensional experiments and parameter sweeps further analyse the mechanism and robustness of the framework. Finally, \textcolor{blue}{Supplementary Note~\ref{sec:si_offline_rl}} extends the same guidance principle to offline reinforcement learning, where it improves performance, stability and runtime over competing guidance schemes.

\section*{Unified distributional view of aerodynamic inverse design}

\begin{figure}[t!]
  \centering
   \includegraphics[width=0.99\linewidth]{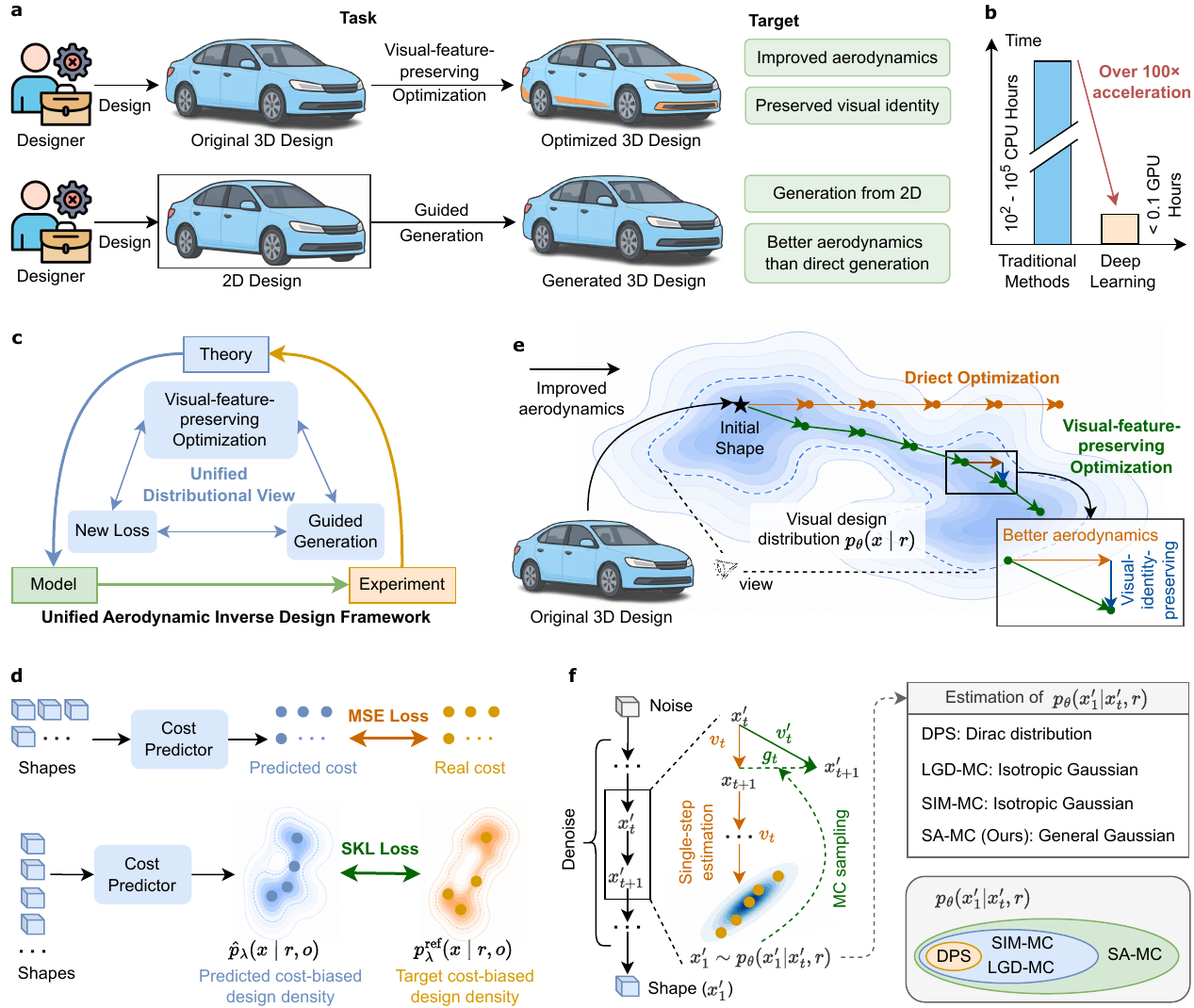}
\caption{\textbf{Unified optimization and generation framework for aerodynamic inverse design.}
\textbf{a}, Overview of the two target design tasks. For a given 3D vehicle, visual-feature-preserving optimization improves aerodynamic performance while keeping the overall appearance nearly unchanged. For an early-stage 2D design view, guided generation produces a plausible 3D design with better aerodynamic properties than direct generation.
\textbf{b}, Computational advantage of the proposed framework. Compared with conventional CFD-based optimization, the learning-based pipeline enables rapid design exploration and achieves over 100-fold acceleration.
\textbf{c}, Conceptual structure of the proposed method. A unified distributional view connects the theoretical formulation, the learning model and experimental validation, covering the new loss, visual-feature-preserving optimization and guided generation.
\textbf{d}, Visual-feature-preserving optimization in the learned design-space probability density. Direct optimization may improve the aerodynamic objective but can move the design away from visually consistent regions. The proposed visual-feature-preserving optimization instead follows directions that improve aerodynamics while preserving the learned visual design distribution.
\textbf{e}, Training objective for the aerodynamic cost predictor. Standard MSE matches individual predicted and real costs, whereas the proposed SKL loss aligns the estimated and real cost-induced distributions, making the predictor better suited for inverse design.
\textbf{f}, Covariance-aware guided generation. Existing training-free guidance methods rely on simplified conditional-distribution approximations, such as a Dirac distribution or isotropic Gaussian. The proposed SA-MC method efficiently estimates a general Gaussian approximation, enabling more accurate guidance for high-dimensional 3D shape generation.}
   \vspace{0pt} 
   \label{fig:overall}
\end{figure}

To address these limitations, we formulate aerodynamic inverse design under a unified probabilistic perspective (Figure~\ref{fig:overall}\textbf{c}), where visual-feature-preserving optimization and guided generation are treated as two instances of searching for high-performance designs under a learned design prior. The complete probabilistic view is shown in Extended Data Figure~\ref{fig:probability}. Let $r$ denote the visual representation of the reference design, $x$ denote a design representation, $p_{\mathrm{data}}(x\mid r)$ the distribution of designs visually consistent with $r$, and $c(x,o)$ the aerodynamic cost under operating condition $o$. The ideal inverse-design target can then be written as a cost-biased design distribution,
\begin{equation}
    p_\lambda(x\mid r,o)
    ~ \propto ~
    p_{\mathrm{data}}(x\mid r)
    \exp\left[-\lambda c(x,o)\right],
    \label{eq:cost_biased_design_distribution}
\end{equation}
where $\lambda$ controls the strength of aerodynamic preference. High-probability regions of this conditional distribution correspond to designs that preserve the visual representation $r$ while achieving favourable aerodynamic performance under the operating condition $o$. This provides a common formulation for the two tasks introduced above: visual-feature-preserving optimization seeks a high-density design of $p_\lambda(x\mid r,o)$, whereas guided generation samples diverse candidates from $p_\lambda(x\mid r,o)$.

In practice, the visual-conditional design distribution is represented by a learned prior $p_\theta(x\mid r)$, and the aerodynamic cost is approximated by a predictor $\hat{c}_\phi(x,o)$. The learned target distribution used by our framework is therefore
\begin{equation}
    \hat{p}_\lambda(x\mid r,o)
    ~ \propto ~
    p_\theta(x\mid r)
    \exp\left[-\lambda \hat{c}_\phi(x,o)\right].
    \label{eq:learned_cost_biased_distribution}
\end{equation}
Here, the learned prior constrains the search to visually consistent design regions, while the exponential cost term tilts this prior towards aerodynamically favourable regions. \textcolor{blue}{Supplementary Note~\ref{sec:si_derivations}} presents the formal statements and proofs of all key theoretical results supporting this work.

\paragraph{Distributional training of the cost predictor.}
The same distributional formulation motivates how the aerodynamic cost predictor should be trained. Since visual-feature-preserving optimization seeks a high-probability-density point of $\hat{p}_\lambda(x\mid r,o)$ and guided generation samples candidates from $\hat{p}_\lambda(x\mid r,o)$, the predictor should recover the distribution induced by the aerodynamic cost rather than only match cost values pointwise (Figure~\ref{fig:overall}\textbf{d}).

Under the learned visual prior, the simulated cost induces the reference tilted distribution
\begin{equation}
    p^{\mathrm{ref}}_\lambda(x\mid r,o)
    ~ \propto ~
    p_\theta(x\mid r)
    \exp\left[-\lambda c(x,o)\right],
    \label{eq:simulated_tilted_distribution}
\end{equation}
whereas the predicted cost induces $\hat{p}_\lambda(x\mid r,o)$ in equation~\eqref{eq:learned_cost_biased_distribution}. We train the predictor with a symmetric-KL-based objective,
\begin{equation}
    \mathcal{L}_{\mathrm{SKL}}
    =
    D_{\mathrm{KL}}
    \left(
    p^{\mathrm{ref}}_\lambda(x\mid r,o)
    \,\|\, 
    \hat{p}_\lambda(x\mid r,o)
    \right)
    +
    D_{\mathrm{KL}}
    \left(
    \hat{p}_\lambda(x\mid r,o)
    \,\|\,
    p^{\mathrm{ref}}_\lambda(x\mid r,o)
    \right).
    \label{eq:skl_distribution_alignment}
\end{equation}
This shifts the role of the predictor from pointwise regression to distributional alignment, making it better matched to the probability landscape actually used by downstream search and sampling. The theoretical connection and distinction between the SKL loss and the MSE loss are discussed in \textcolor{blue}{Supplementary Note~\ref{subsec:sKL-mse}}.

\paragraph{Visual-feature-preserving aerodynamic optimization.}
For visual-feature-preserving optimization, the goal is to find a high-density design of the learned target distribution rather than to minimize aerodynamic cost alone (Extended Data Figure~\ref{fig:probability} \textbf{b}). This can be written as
\begin{equation}
    x^\star
    =
    \arg\min_x
    \left[
    -\log \hat{p}_\lambda(x\mid r,o)
    \right].
    \label{eq:shape_preserving_objective}
\end{equation}
Using the definition of $\hat{p}_\lambda(x\mid r,o)$, its gradient is
\begin{equation}
    \nabla_x \log \hat{p}_\lambda(x\mid r,o)
    =
    \nabla_x \log p_\theta(x\mid r)
    -
    \lambda \nabla_x \hat{c}_\phi(x,o).
    \label{eq:target_density_gradient}
\end{equation}
Therefore, minimizing $-\log \hat{p}_\lambda(x\mid r,o)$ moves the design along two coupled directions: increasing the conditional density of visually consistent designs and decreasing the aerodynamic cost.

This view clarifies why direct cost minimization is insufficient for practical aerodynamic refinement. The second term in equation~\eqref{eq:target_density_gradient} is the aerodynamic cost gradient, which drives the design towards lower predicted cost. The first term is the conditional density gradient, which encourages the design to remain in regions that are visually consistent with the reference representation $r$. A direct optimizer that follows only $\nabla_x \hat{c}_\phi(x,o)$ may reduce the predicted aerodynamic cost, but can also move the design away from the visual features of the reference design. In contrast, visual-feature-preserving optimization follows the two coupled directions in Figure~\ref{fig:overall}\textbf{e}: a cost gradient for aerodynamic improvement and a density gradient for visual consistency. The resulting trajectory favours local, visually consistent refinement rather than unconstrained movement towards a purely aerodynamic optimum.

\paragraph{Guided aerodynamic generation.}
Let $x'_1$ denote the final clean design produced by aerodynamic guidance. The goal of guided generation is to make this guided output follow the learned cost-biased distribution,
\begin{equation}
    x'_1 \sim \hat{p}_\lambda(x'_1\mid r,o)
    ~ \propto ~
    p_\theta(x'_1\mid r)
    \exp\left[-\lambda \hat{c}_\phi(x'_1,o)\right],
    \label{eq:generation_target_distribution}
\end{equation}
where the prime denotes the guided trajectory. To see why guidance is needed, first consider the original flow without aerodynamic guidance. Generation starts from a source sample $x_0$, typically drawn from a simple noise distribution, and evolves through intermediate states $x_t$ with $t\in[0,1]$ under the learned velocity field $v_t(x_t\mid r)$ until it reaches the clean design $x_1$. Following this original velocity field samples the visual-conditional design prior,
\begin{equation}
    x_1 \sim p_\theta(x_1\mid r).
    \label{eq:unguided_generation_prior}
\end{equation}
This direct sample is consistent with the input representation $r$, but it is not biased towards lower aerodynamic cost. To obtain $x'_1$ from the target distribution in equation~\eqref{eq:generation_target_distribution}, we instead evolve a guided trajectory $x'_t$ using a guided velocity field $v'_t(x'_t\mid r,o)$ (Extended Data Figure~\ref{fig:probability}\textbf{c}).

The guidance term is defined as the correction from the original velocity field to the guided one,
\begin{equation}
    g_t(x'_t,r,o)
    :=
    v'_t(x'_t\mid r,o)
    -
    v_t(x'_t\mid r)
    ~ \propto ~
    \nabla_{x'_t}
    \log
    \mathbb{E}_{p_\theta(x_1\mid x'_t,r)}
    \left[
    \exp\left(-\lambda \hat{c}_\phi(x_1,o)\right)
    \right].
    \label{eq:guidance_correction}
\end{equation}
Here, $p_\theta(x_1\mid x'_t,r)$ denotes the distribution of possible unguided clean completions reachable from the current guided state $x'_t$ under the original flow. This guidance evaluates the expected aerodynamic preference of those possible clean completions and steers the guided trajectory towards regions whose future clean designs are likely to have lower cost. 
The derivation of equation (\ref{eq:guidance_correction}) follows the flow-matching guidance result \cite{feng2025guidance}.

The remaining challenge is how to estimate the conditional clean-design distribution $p_\theta(x_1\mid x'_t,r)$ in equation~\eqref{eq:guidance_correction}. Existing training-free guidance methods mainly differ in this approximation (Figure~\ref{fig:overall}\textbf{f}). DPS~\cite{chung2022diffusion} uses a single-point approximation, whereas LGD-MC~\cite{song2023loss} and SIM-MC~\cite{feng2025guidance} use an isotropic Gaussian approximation. However, these approximations do not provide the optimal Gaussian match. The optimal Gaussian approximation should match both the conditional mean and covariance~\cite{boystweedie}. Computing the covariance directly, however, requires a high-dimensional Jacobian, which is prohibitively expensive for general high-dimensional generation and has so far only been made practical for linear inverse problems with sparse operators~\cite{boystweedie}. For general problems, including 3D shape generation, this remains difficult. We therefore introduce SA-MC (Secant-Approximation Monte Carlo), which provides a practical general-Gaussian approximation through Monte Carlo sampling and efficient covariance estimation, enabling covariance-aware guidance for high-dimensional aerodynamic design. 
This summary is based on the unified theoretical perspective developed in \textcolor{blue}{Supplementary Note~\ref{subsec:si_guidance_statement}}, which shows that the essential differences among these guidance methods lie in their approximations to \(p_\theta(x_1\mid x'_t,r)\).

\section*{Latent-space implementation of the framework}

\begin{figure}[t!]
\centering
\includegraphics[width=0.99\linewidth]{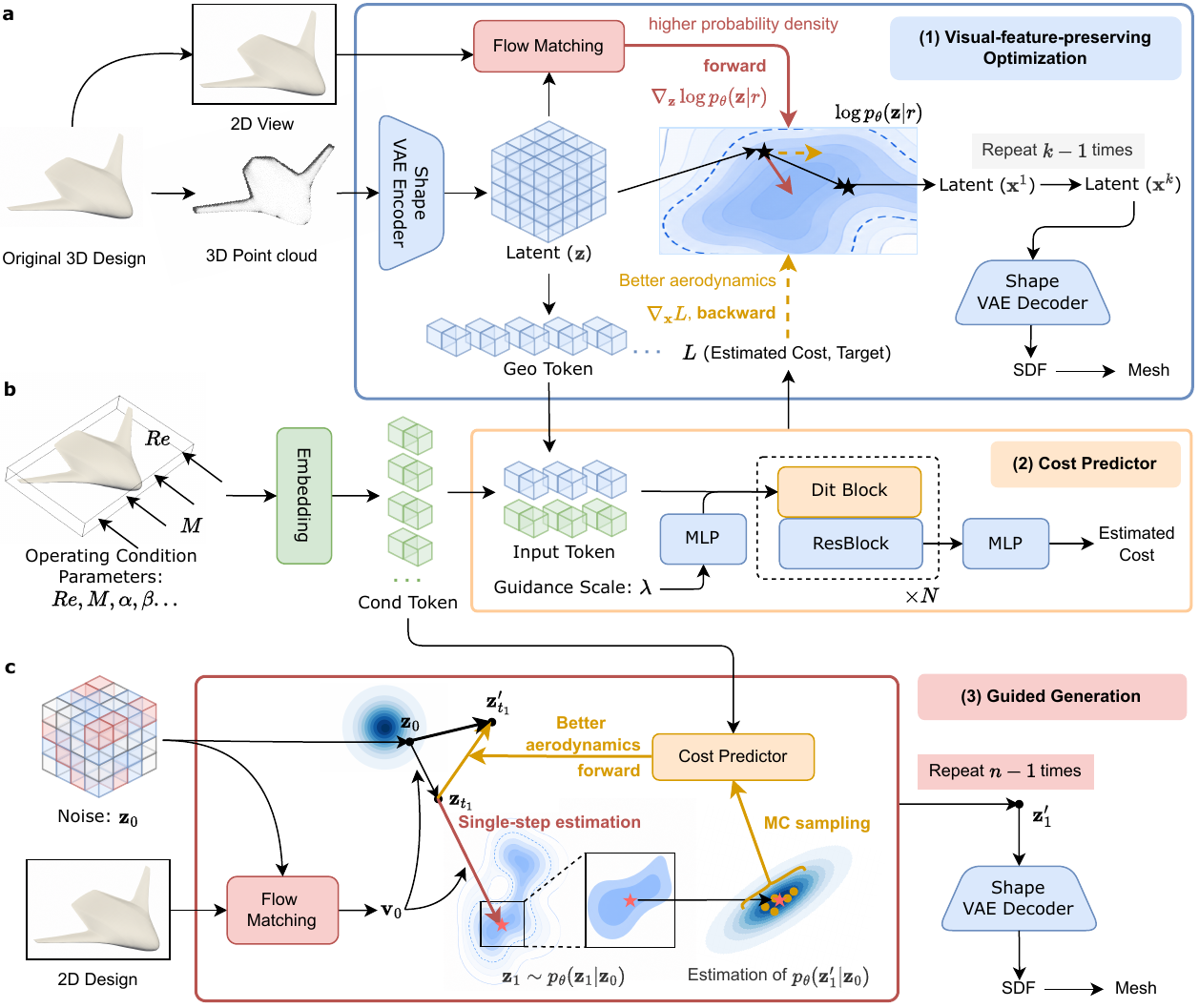}
\caption{\textbf{Latent-space implementation of visual-feature-preserving optimization and guided generation.}
\textbf{a}, Latent implementation of visual-feature-preserving optimization. An input 3D design is represented by a point cloud and encoded into a latent variable $\mathbf{z}$ by the Shape-VAE encoder. The pretrained flow-matching prior provides the visual-density direction $\nabla_{\mathbf{z}}\log p_\theta(\mathbf{z}\mid r)$ through its forward process, while the aerodynamic cost predictor provides the cost direction by back-propagating the predicted cost. The latent variable is iteratively updated by combining these two directions and then decoded into an SDF and mesh.
\textbf{b}, Condition-aware aerodynamic cost predictor. Operating conditions, such as Reynolds number, Mach number and angle of attack, are embedded as condition tokens and combined with geometry tokens from the latent design. The guidance scale $\lambda$ is also embedded, allowing one predictor to provide cost terms for different aerodynamic-bias strengths. The predictor outputs the estimated aerodynamic cost used by both optimization and generation.
\textbf{c}, Latent guided generation. Starting from noise $\mathbf{z}_0$ and a 2D design view, the flow-matching prior gives the unguided transport direction. SA-MC estimates the distribution of reachable clean latent designs through Monte Carlo sampling, scores the sampled clean designs with the cost predictor, and constructs a covariance-aware guidance direction towards lower-cost regions. The guided latent trajectory is finally decoded into a 3D design.}
\label{fig:all_model}
\end{figure}

The probabilistic formulation above defines aerodynamic inverse design through the learned cost-biased distribution $\hat{p}_\lambda(x\mid r,o)$. We implement this target in the latent space of Hunyuan3D 2.1 \cite{hunyuan3d2025hunyuan3d}, as shown in Figure~\ref{fig:all_model}. In this section, the abstract design variable $x$ is replaced by the Shape-VAE latent variable $\mathbf{z}$. This latent implementation is important for both efficiency and stability. Instead of directly modifying meshes, point clouds or dense SDF fields, the framework updates or samples a compact latent variable and then decodes the final latent design into an SDF and mesh.

The implementation uses three coupled modules. The Shape-VAE maps 3D geometry to and from the latent space. The DiT-based flow-matching model defines the visual-conditional design prior, which describes latent designs consistent with the input representation $r$. The aerodynamic cost predictor estimates the cost under operating condition $o$ and guidance scale $\lambda$. We keep the pretrained generative model fixed and adapt it through the learned aerodynamic cost, rather than retraining the 3D generator with aerodynamic data. This avoids the expensive training or fine-tuning of a large-scale 3D generative model, while preserving the broad visual and shape knowledge learned from large-scale 3D data. Aerodynamic adaptation is instead achieved by learning a lightweight condition-aware cost predictor and using it to bias the fixed latent design distribution.

In latent space, the learned cost-biased distribution is
\begin{equation}
    \hat{p}_\lambda(\mathbf{z}\mid r,o)
    ~ \propto ~
    p_\theta(\mathbf{z}\mid r)
    \exp\left[-\lambda \hat{c}_\phi(\mathbf{z},o;\lambda)\right],
    \label{eq:latent_target_distribution}
\end{equation}
where $p_\theta(\mathbf{z}\mid r)$ is the visual-conditional latent prior provided by the flow-matching model, and $\hat{c}_\phi(\mathbf{z},o;\lambda)$ is the predicted aerodynamic cost. The input representation $r$ can be an existing 3D design for optimization or a 2D design view for generation, and $o$ denotes operating conditions such as Reynolds number, Mach number and angle of attack. The first term in equation~\eqref{eq:latent_target_distribution} keeps the design tied to the visual input, while the exponential cost term reweights the distribution towards lower-cost regions. \textcolor{blue}{Supplementary Note~\ref{subsec:si_cost_biased_statement}} gives the corresponding theoretical statement.

\paragraph{Condition-aware aerodynamic cost predictor.}
The cost predictor provides the aerodynamic term in equation~\eqref{eq:latent_target_distribution} (Figure~\ref{fig:all_model}\textbf{b}). Given a latent design $\mathbf{z}$ and an operating condition $o$, such as Reynolds number, Mach number and angle of attack, the predictor outputs a scalar aerodynamic cost for downstream optimization and generation. Since the SKL objective in equation~\eqref{eq:skl_distribution_alignment} is defined over the full cost-biased design distribution, we implement it with a mini-batch probability loss during training.

A further issue is that the target distribution itself depends on the guidance scale $\lambda$: different values of $\lambda$ correspond to different strengths of aerodynamic preference and therefore induce different optimization objectives. To avoid training separate predictors for different $\lambda$ values, we condition the predictor on $\lambda$ and write its output as
$\hat{c}_\phi(\mathbf{z},o;\lambda)$. This allows one predictor to provide aerodynamic terms that are consistent with different guidance strengths in both visual-feature-preserving optimization and guided generation. Further implementation details are provided in Methods and in \textcolor{blue}{Supplementary Note~\ref{subsec:si_loss_statement}}.

\paragraph{Latent visual-feature-preserving optimization.}
Figure~\ref{fig:all_model}\textbf{a} shows the latent optimization workflow. Starting from an existing 3D design, we first convert the geometry into a point-cloud representation and encode it into a Shape-VAE latent variable $\mathbf{z}$. Visual-feature-preserving optimization then updates $\mathbf{z}$ by following the log-density gradient of the latent target distribution,
\begin{equation}
    \nabla_{\mathbf{z}}
    \log
    \hat{p}_\lambda(\mathbf{z}\mid r,o)
    =
    \nabla_{\mathbf{z}}
    \log
    p_\theta(\mathbf{z}\mid r)
    -\lambda
    \nabla_{\mathbf{z}}
    \hat{c}_\phi(\mathbf{z},o;\lambda).
    \label{eq:latent_optimization_gradient}
\end{equation}
The first term is obtained from the pretrained flow-matching prior through its forward process and points towards higher-probability latent regions under the input visual representation. The second term is obtained by back-propagating through the aerodynamic cost predictor and points towards lower predicted cost.

The update therefore combines two directions that correspond to the two requirements of the design task. The cost-gradient direction improves aerodynamic performance, while the visual-density direction discourages the latent variable from moving away from the visual features of the reference design. After repeating the update for $k-1$ steps, the optimized latent variable is decoded by the Shape-VAE decoder into an SDF and mesh. This implements the Density optimization rule used in the vehicle and aircraft experiments: the optimizer does not follow the aerodynamic cost alone, but seeks a low-cost design that remains in the learned visual design distribution. \textcolor{blue}{Supplementary Note~\ref{subsec:si_latent_opt_statement}} details the latent optimization workflow, and \textcolor{blue}{Supplementary Note~\ref{sec:si_algorithms}} gives pseudocode and time-complexity analysis.

\paragraph{Latent guided generation.}
Figure~\ref{fig:all_model}\textbf{c} shows the latent guided-generation workflow. Starting from a 2D design view and an initial noise latent $\mathbf{z}_0$, the flow-matching model provides the unguided transport direction. Without aerodynamic guidance, this process samples from the visual-conditional prior $p_\theta(\mathbf{z}_1\mid r)$ and produces a clean latent design consistent with the input view, but not explicitly biased towards lower aerodynamic cost. Guided generation instead modifies the flow trajectory so that the final clean latent variable is sampled from the cost-biased target distribution in equation~\eqref{eq:latent_target_distribution}.

The guidance correction is written as
\begin{equation}
    \mathbf{g}_t
    ~ \propto ~
    \nabla_{\mathbf{z}_t}
    \log
    \mathbb{E}_{p_\theta(\mathbf{z}_1\mid \mathbf{z}_t,r)}
    \left[
    \exp\left(-\lambda \hat{c}_\phi(\mathbf{z}_1,o;\lambda)\right)
    \right].
    \label{eq:latent_guidance}
\end{equation}
Here, $p_\theta(\mathbf{z}_1\mid \mathbf{z}_t,r)$ denotes the distribution of possible clean latent designs reachable from the current intermediate state under the original flow. The main implementation challenge is to estimate this reachable clean-design distribution accurately enough in high-dimensional 3D latent space.

Although equation~\eqref{eq:latent_guidance} is written as a gradient with respect to $\mathbf{z}_t$, SA-MC does not evaluate it by directly differentiating the full expectation. Instead, it generates possible clean latent designs from the current state using the flow-matching model, scores them with the cost predictor, and constructs the guidance direction from cost-weighted sample statistics. This makes a covariance-aware, general Gaussian approximation of $p_\theta(\mathbf{z}_1\mid \mathbf{z}_t,r)$ feasible in high-dimensional 3D latent space, without computing high-dimensional Jacobians. The final guided latent variable is decoded by the Shape-VAE into a 3D design that remains visually consistent with the input representation while being biased towards improved aerodynamic performance. \textcolor{blue}{Supplementary Note~\ref{subsec:si_samc_statement}} gives the complete covariance-estimation workflow used by SA-MC. \textcolor{blue}{Supplementary Note~\ref{subsec:si_samc_algorithm}} provides the executable pseudocode and time-complexity analysis.

\section*{Vehicle aerodynamic validation}

\begin{figure}[htbp]
  \centering
  \includegraphics[width=0.82\linewidth]{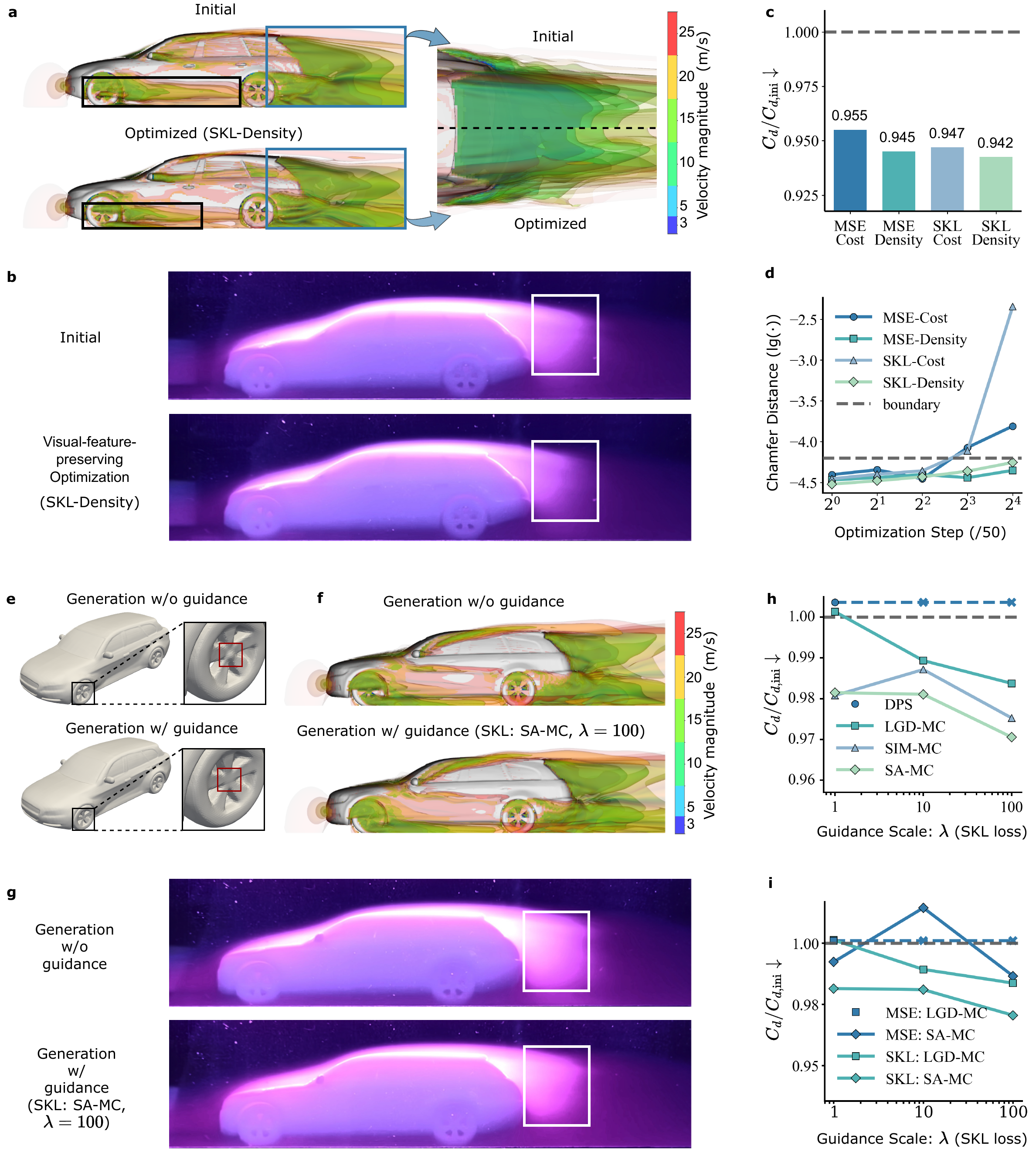}
\caption{\textbf{Vehicle aerodynamic validation.}
\textbf{a}, CFD comparison between the initial vehicle and the visual-feature-preserving optimized design. The optimized shape reduces adverse flow features in key regions while preserving the overall vehicle appearance.
\textbf{b}, Wind-tunnel visualization of the initial and optimized designs, showing a reduced wake region after visual-feature-preserving optimization.
\textbf{c}, Normalized drag coefficient for different predictor losses and optimization rules. ``Cost'' denotes cost-only optimization, whereas ``Density'' denotes the full visual-feature-preserving rule using the learned visual-design density. The grey dashed line marks the initial-vehicle baseline.
\textbf{d}, Chamfer distance during optimization. The grey dashed line marks the visual-deviation boundary; curves crossing this boundary correspond to updates that no longer preserve the input visual features. Visual-feature-preserving optimization keeps the update controlled, whereas cost-only optimization can deviate strongly from the original shape.
\textbf{e}, Generated 3D vehicles without and with aerodynamic guidance. Guidance reduces visible local distortions, especially abnormal protrusions around the wheel region, while maintaining the overall vehicle appearance.
\textbf{f}, CFD visualization of unguided and guided generation. SKL-based SA-MC guidance reduces the low-velocity wake compared with unguided generation.
\textbf{g}, Wind-tunnel visualization of generated designs, showing a smaller turbulent wake region after aerodynamic guidance.
\textbf{h}, Effect of guidance scale on normalized drag for different guidance methods under the SKL-trained predictor. The grey dashed line at 1.0 denotes direct generation without aerodynamic guidance; values below this line indicate drag reduction relative to direct generation. Crossed/dashed markers denote large-guidance settings that collapse or fail to produce usable vehicle outputs.
\textbf{i}, Joint comparison of predictor loss and guidance method. The grey dashed line denotes direct generation without guidance. The combination of SKL training and SA-MC guidance gives the most stable aerodynamic improvement across guidance scales.}
  \label{fig:vehicle}
\end{figure}

We first evaluate the framework on ground vehicles, where the aerodynamic objective is the drag coefficient $C_d$ and the design constraint is to preserve the visual features of the input vehicle. This experiment is organized as a validation of the three components of the proposed framework: the predictor loss, the visual-feature-preserving optimization rule and the guided generation strategy. In Figure~\ref{fig:vehicle}\textbf{c,d}, the labels ``Cost'' and ``Density'' denote two optimization rules. ``Cost'' means that the latent design is updated using only the aerodynamic cost term. ``Density'' denotes the full visual-feature-preserving optimization rule, in which the learned visual-design density is retained in the objective. Thus, the four settings separate the effect of the predictor loss, MSE or SKL, from the effect of the optimization rule, cost-only or visual-density-preserving.

For optimization, Figure~\ref{fig:vehicle}\textbf{a--d} shows why both the SKL loss and the visual-density term are needed. Starting from the same 3D vehicle, the SKL-Density setting modifies mainly local regions while preserving the recognizable vehicle appearance (Figure~\ref{fig:vehicle}\textbf{a}). The CFD field shows a reduced low-speed wake and weaker adverse-flow regions, and the 3D-printed wind-tunnel visualization gives the same qualitative trend with a smaller wake behind the optimized vehicle (Figure~\ref{fig:vehicle}\textbf{b}). In Figure~\ref{fig:vehicle}\textbf{c}, the grey dashed line is the drag of the initial vehicle, used as the normalization reference. All four optimization variants reduce drag below this reference, but SKL-Density gives the lowest value, $C_d/C_{d,\mathrm{init}}=0.942$, corresponding to a 5.8\% reduction. The comparison between MSE and SKL shows that the distribution-aware loss improves the downstream optimization target, while the comparison between Cost and Density shows that the visual-density term is not merely a regularizer for appearance: it also helps the optimizer remain in a useful region of the learned design distribution.

Figure~\ref{fig:vehicle}\textbf{d} explains the second part of the claim. The grey dashed line marks the visual-deviation boundary used in this analysis. Cost-only optimization can reduce drag, but it can also move the design beyond this boundary as the update proceeds, meaning that the optimized shape no longer preserves the input visual features. In contrast, the Density variants keep the Chamfer distance below the boundary throughout the optimization. The vehicle optimization result therefore supports the intended role of the framework: aerodynamic improvement is achieved while the update remains tied to the input visual features.

The generation panels test the sampling side of the same probability target. From the same input view, direct generation produces a vehicle-like 3D shape, but it is not explicitly biased towards lower drag. Aerodynamic guidance changes this by steering the generated candidate towards lower-cost regions while maintaining view consistency. In Figure~\ref{fig:vehicle}\textbf{e}, guidance also reduces visible local distortions around the wheel region. CFD visualization shows a smaller low-velocity wake for SKL-based SA-MC guidance than for direct generation (Figure~\ref{fig:vehicle}\textbf{f}), and the wind-tunnel visualization of generated vehicles shows a consistent reduction in the turbulent wake region (Figure~\ref{fig:vehicle}\textbf{g}).

The quantitative guidance results further validate the interaction between the loss and the sampler. In Figure~\ref{fig:vehicle}\textbf{h,i}, the grey dashed line denotes direct generation without guidance from the same view; values below this line indicate improvement over direct generation. Under the SKL-trained predictor, SA-MC remains below this baseline across the valid tested guidance scales and reaches $C_d/C_{d,\mathrm{direct}}=0.970$, a 3.0\% drag reduction relative to direct generation. Some alternative guidance methods become unstable at large guidance scales, as indicated by the failed or crossed high-scale markers. When the predictor is trained only with MSE, guidance becomes less reliable and can even move above the direct-generation baseline. These results show that the vehicle performance is not produced by a single module alone: the SKL loss defines a more useful cost-biased target, the Density optimization rule preserves visual features during refinement, and SA-MC uses the same target to make view-based generation controllably aerodynamic.

\section*{Aircraft aerodynamic validation}

\begin{figure}[htbp]
  \centering
   \includegraphics[width=0.82\linewidth]{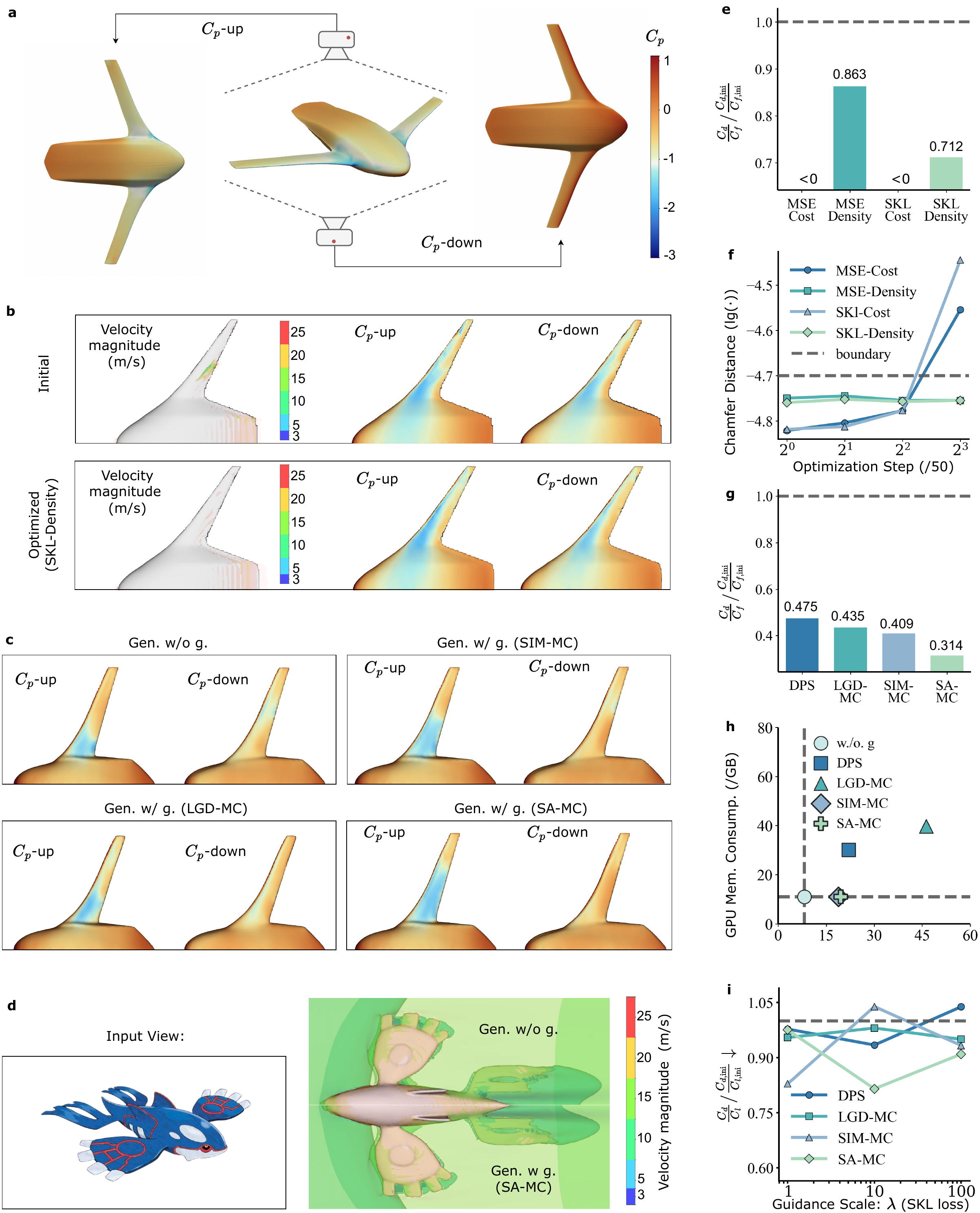}
\caption{\textbf{Aerodynamic validation on blended-wing aircraft.}
\textbf{a}, Definition of the upper- and lower-surface pressure-coefficient fields, denoted as $C_p$-up and $C_p$-down.
\textbf{b}, CFD comparison of the initial and optimized aircraft. The visual-feature-preserving optimized design reduces low-speed residual regions and yields smoother pressure distributions near the wing--body region.
\textbf{c}, Pressure-coefficient fields of generated aircraft without guidance and with different guidance methods.
\textbf{d}, Out-of-distribution view-based generation from a Kyogre input view. SA-MC guidance reduces the low-velocity wake compared with unguided generation.
\textbf{e}, Normalized drag-to-lift objective for different predictor losses and optimization rules. ``Cost'' denotes cost-only optimization, whereas ``Density'' denotes the full visual-feature-preserving rule using the learned visual-design density. The grey dashed line marks the initial-aircraft baseline.
\textbf{f}, Chamfer distance during optimization. The grey dashed line marks the visual-deviation boundary; visual-feature-preserving optimization keeps the optimized aircraft closer to the initial design, while cost-only optimization can cross this boundary.
\textbf{g}, Comparison of training-free guidance methods. The grey dashed line at 1.0 denotes direct generation without aerodynamic guidance, and SA-MC achieves the lowest normalized drag-to-lift objective.
\textbf{h}, Runtime and GPU-memory comparison for single-sample generation. The grey dashed lines indicate the time and memory reference of unguided generation.
\textbf{i}, Quantitative analysis of the Kyogre-conditioned out-of-distribution generation across guidance scales. The grey dashed line denotes unguided generation from the same OOD view.}
   \label{fig:aircraft}
\end{figure}

We next evaluate the framework on blended-wing aircraft, where the objective is the drag-to-lift ratio $C_d/C_l$. This case is more stringent than the vehicle experiment because improving the objective cannot be separated from maintaining positive lift: a design with negative lift is a failed aircraft, not a valid aerodynamic improvement. The aircraft results therefore test whether the same framework can support visual-feature-preserving optimization and guided generation under a coupled lift-drag objective. Figure~\ref{fig:aircraft}\textbf{a} defines the upper- and lower-surface pressure-coefficient fields, $C_p$-up and $C_p$-down, used to interpret the lift-related pressure changes. As in the vehicle experiment, ``Cost'' denotes cost-only optimization, while ``Density'' denotes the full visual-feature-preserving optimization rule that retains the learned visual-design density term.

For optimization, Figure~\ref{fig:aircraft}\textbf{b,e,f} shows that cost reduction alone is not a sufficient criterion for aircraft design. Starting from the same blended-wing aircraft, the SKL-Density setting changes the flow field and pressure distribution near the wing--body region while maintaining the overall blended-wing appearance (Figure~\ref{fig:aircraft}\textbf{b}). In Figure~\ref{fig:aircraft}\textbf{e}, the grey dashed line at 1.0 is the initial-aircraft objective used for normalization. The Cost variants can drive the lift coefficient below zero, shown as negative values in the plot. These cases are treated as failed aircraft designs and are not valid improvements. By contrast, the Density variants remain valid, and SKL-Density gives the best valid value, $0.712$, corresponding to a 28.8\% reduction relative to the initial aircraft. This comparison directly supports the value of the SKL loss under the full visual-feature-preserving optimization rule.

Figure~\ref{fig:aircraft}\textbf{f} shows why the Density rule is necessary. The grey dashed line marks the visual-deviation boundary. Cost-only optimization can cross this boundary, indicating that the update is no longer visually tied to the reference aircraft. The Density variants keep the aircraft close to the initial design while still improving the aerodynamic objective. Thus, the aircraft optimization result strengthens the vehicle conclusion: the proposed target is not simply a way to lower a predicted cost, but a way to search for lower-cost designs within the visual design distribution.

The generation results in Figure~\ref{fig:aircraft}\textbf{c,g} test whether the same probability target can be used for early-stage view-based design. Here the comparison is with direct generation from the same input view, not with the initial aircraft used in optimization. The grey dashed line in Figure~\ref{fig:aircraft}\textbf{g} denotes this direct-generation baseline. All guided methods improve over the baseline, but SA-MC gives the lowest normalized $C_d/C_l$, reaching $0.314$, a 68.6\% reduction relative to direct generation. The pressure fields in Figure~\ref{fig:aircraft}\textbf{c} are consistent with this quantitative result: different guidance rules produce different upper- and lower-surface pressure patterns, and SA-MC gives the most favourable pressure distribution among the compared methods.

Figure~\ref{fig:aircraft}\textbf{d,i} further tests controllability under an out-of-distribution visual condition. The input is a Kyogre view rather than an aircraft view, so this setting probes whether the guidance rule still provides a usable aerodynamic bias when the input view is far from the aircraft training distribution. Panel \textbf{d} shows that SA-MC reduces the low-velocity wake region relative to unguided generation. In panel \textbf{i}, the grey dashed line denotes unguided generation from the same OOD view. Although the guided methods fluctuate more strongly in this harder setting, SA-MC remains below the unguided baseline across the tested guidance scales and gives the lowest overall $C_d/C_l$.

Finally, Figure~\ref{fig:aircraft}\textbf{h} reports runtime and GPU memory for single-sample generation, with the grey dashed lines marking the unguided-generation reference. SA-MC adds aerodynamic guidance while keeping memory consumption close to the unguided and SIM-MC settings, whereas LGD-MC requires substantially more GPU memory. Together, the aircraft results show that the proposed framework is not limited to vehicle drag reduction. The same SKL-trained cost-biased target, visual-feature-preserving optimization rule and covariance-aware guided sampler extend to a different geometry family and a coupled lift-drag objective, while avoiding invalid negative-lift solutions.

\section*{Controlled analysis of optimization and guidance}

\begin{figure}[t!]
  \centering
   \includegraphics[width=0.90\linewidth]{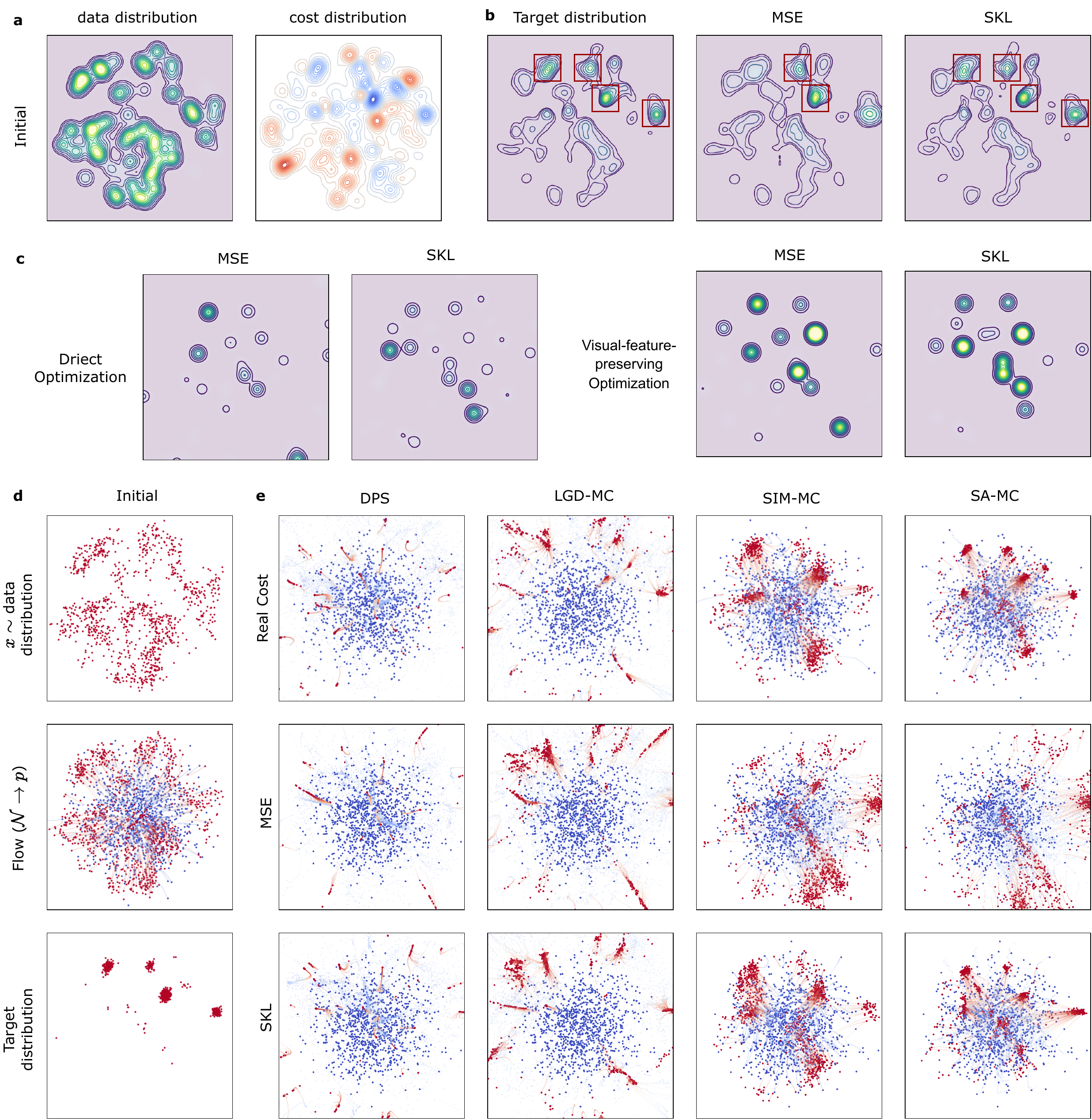}
      \caption{\textbf{Controlled analysis of the mechanisms behind optimization and guidance.}
\textbf{a}, Two-dimensional controlled setting. The data distribution plays the role of the learned visual design distribution, and the cost landscape defines the inverse-design preference.
\textbf{b}, Ground-truth cost-biased target distribution and the target distributions induced by MSE- and SKL-trained predictors. Red boxes highlight the target modes. MSE weakens or misses modes that are important after cost reweighting, whereas SKL better preserves the multimodal target structure.
\textbf{c}, Cost-only optimization and visual-feature-preserving optimization under different predictor losses. Cost-only optimization follows the predicted cost term alone and can converge to isolated low-cost regions away from the data distribution. Visual-feature-preserving optimization adds the data-density term and redirects updates towards supported target regions, especially when combined with SKL.
\textbf{d}, Reference generation process in the controlled setting. Samples from the data distribution define the visual reference, flow matching transports noise samples to this reference distribution, and the desired cost-biased target occupies selected high-probability, low-cost modes.
\textbf{e}, Guided generation using different training-free estimators and predictor losses. SA-MC most closely follows the target modes when paired with the SKL-trained predictor, showing that covariance-aware guidance improves distributional sampling rather than merely reducing a scalar cost.}
   \label{fig:analysis}
\end{figure}

The vehicle and aircraft experiments validate the framework in engineering settings. The controlled two-dimensional experiment in Figure~\ref{fig:analysis} is used for a different purpose: it exposes the distributional mechanism behind those results. In the full problem, visually consistent designs define a learned visual design distribution, and aerodynamic cost reweights this distribution towards better-performing regions. The controlled setting keeps the same structure in two dimensions. The data distribution represents the visual design distribution, and the cost landscape represents the inverse-design objective (Figure~\ref{fig:analysis}\textbf{a}). Because both are known exactly, we can directly compare the true cost-biased target with the targets produced by different losses, optimization rules and guidance estimators. \textcolor{blue}{Supplementary Note~\ref{sec:2d-setting}} gives the construction of the controlled problem, and \textcolor{blue}{Supplementary Note~\ref{sec:si_experimental_analysis}} provides additional analysis of predictor-induced drift.

The first question is whether the learned predictor defines the correct probability target. This matters because the predictor is not used only as a scalar regressor. Its output is exponentiated and combined with the reference distribution, so errors that appear small pointwise can still distort the cost-biased distribution used by optimization and generation. Figure~\ref{fig:analysis}\textbf{b} shows this effect. The MSE-trained predictor weakens or misses several target modes, including modes highlighted by the red boxes, even though it is trained to match costs pointwise. The SKL-trained predictor better preserves the location and relative importance of these modes because its training objective acts on the induced distribution. This explains why the loss affects downstream inverse design: the useful object is not only an accurate cost value, but an accurate cost-biased target distribution.

The second question is why visual-feature-preserving optimization is different from cost-only optimization. In Figure~\ref{fig:analysis}\textbf{c}, cost-only optimization follows the predicted cost direction alone. It can therefore move samples towards isolated low-cost regions that have little support under the data distribution. In the aerodynamic problem, this corresponds to reducing a predicted drag or drag-to-lift objective while losing the visual features of the reference design. Visual-feature-preserving optimization adds the data-density term, which is the two-dimensional analogue of the learned visual-design distribution in the main framework. This term redirects updates towards regions that are both low-cost and supported by the reference distribution. The effect is clearest with SKL, where the learned cost-biased target is already better aligned with the true target. Thus, the optimization result in Figure~\ref{fig:analysis}\textbf{c} explains why the Density rule in the vehicle and aircraft experiments can improve performance without crossing the visual-deviation boundary.

The third question is whether guided generation samples the intended target distribution. Figure~\ref{fig:analysis}\textbf{d} separates the reference distribution from the desired target: unguided flow matching transports noise to the data distribution, whereas guided generation should further bias samples towards the cost-reweighted target modes. Figure~\ref{fig:analysis}\textbf{e} shows that this is not guaranteed by guidance alone. DPS, LGD-MC and SIM-MC can move samples in favourable directions, but they may miss modes, over-concentrate samples or send trajectories into regions that do not match the target distribution. These failure modes become stronger when the predictor is trained with MSE, because the guidance then follows a distorted target. With the SKL-trained predictor, SA-MC gives the closest match to the target modes among the compared training-free guidance rules. This supports the role of covariance-aware guidance: it improves how the sampler follows the local spread of reachable clean samples, rather than simply applying a stronger cost gradient.

Together, the controlled analysis explains the structure observed in the vehicle and aircraft results. SKL improves the probability target used by the framework; visual-feature-preserving optimization follows this target while remaining within the reference distribution; and SA-MC provides a more faithful guided sampler for the same target. The engineering improvements in the preceding sections therefore come from aligning the predictor, optimizer and generator with one cost-biased visual design distribution, rather than from treating loss design, optimization and generation as separate empirical choices.

\section*{Robustness of covariance-aware guidance}

\begin{figure}[t]
  \centering
   \includegraphics[width=0.99\linewidth]{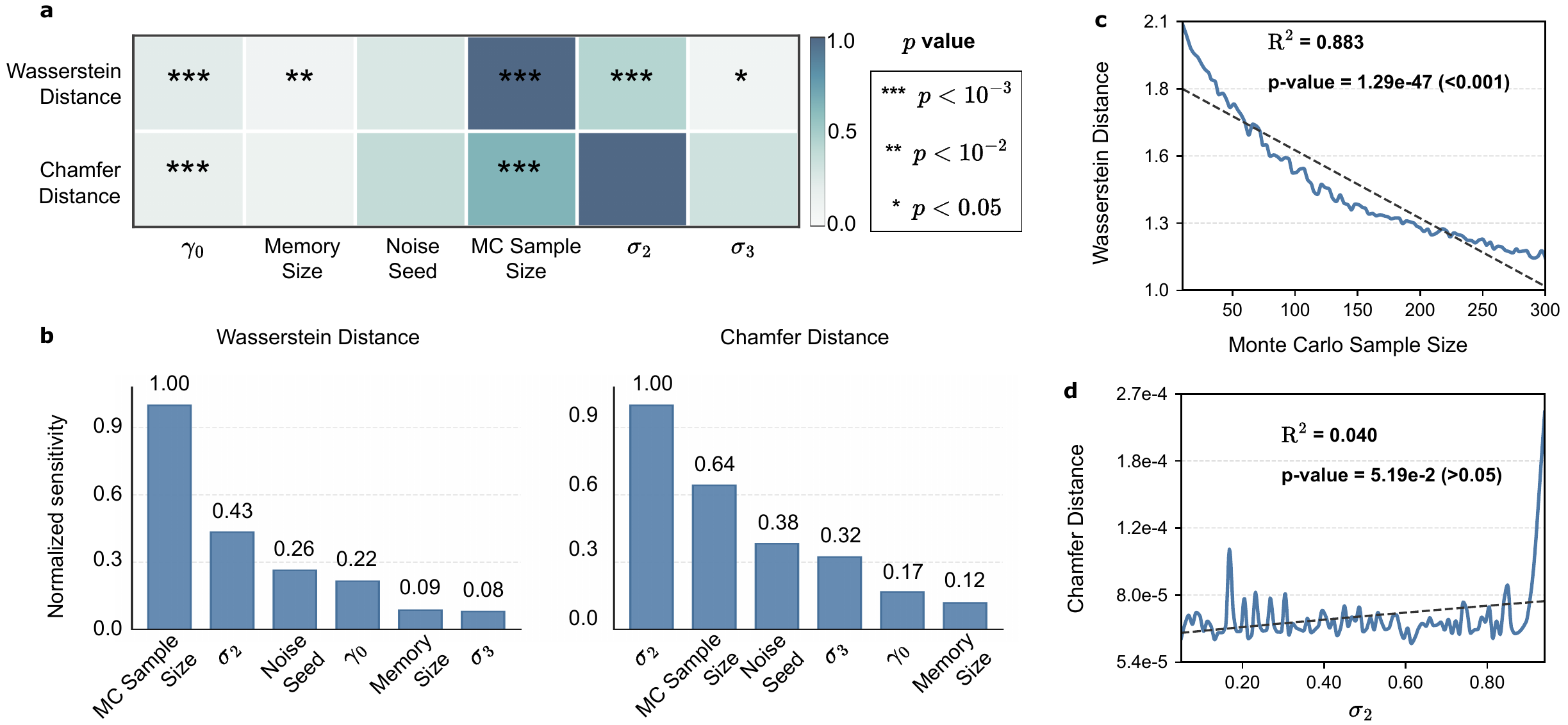}
\caption{\textbf{Two-level parameter sensitivity of SA-MC guidance.}
\textbf{a}, Sensitivity of SA-MC implementation parameters evaluated from two complementary levels. The distribution-level sensitivity is measured by Wasserstein distance, which quantifies how well the generated sample distribution matches the intended cost-biased target. The instance-level sensitivity is measured by Chamfer distance, which quantifies whether the decoded 3D shape changes under different parameter settings. Tested parameters include the initial covariance scale \(\gamma_0\), memory length, perturbation noise seed, Monte Carlo sample size, and damping parameters \(\sigma_2\) and \(\sigma_3\). Colours denote normalized sensitivity scores, and asterisks indicate significance levels.
\textbf{b}, Ranked normalized sensitivities for the two metrics. Monte Carlo sample size dominates distribution-level matching, indicating that the statistical quality of covariance estimation is the main factor controlling the target-distribution approximation. At the instance level, \(\sigma_2\) shows the largest relative sensitivity, followed by Monte Carlo sample size and noise seed.
\textbf{c}, Detailed trend for the dominant distribution-level factor. Increasing the Monte Carlo sample size steadily reduces the Wasserstein distance, confirming that accurate Monte Carlo covariance estimation is critical for distributional alignment.
\textbf{d}, Detailed trend for the dominant instance-level factor. Although \(\sigma_2\) has the largest relative sensitivity in the Chamfer-distance ranking, the curve remains nearly flat within the practical range used in our experiments, \(\sigma_2\in[0.3,0.6]\). Noticeable shape changes mainly occur under extreme damping choices, supporting the stability of SA-MC under moderate \(\sigma_2\) values.}
   \label{fig:sensitivity_all}
\end{figure}

SA-MC estimates local covariance information during sampling, so its behaviour depends on a small set of implementation parameters: the initial covariance scale $\gamma_0$, the memory length used to store recent update directions, the perturbation noise seed, the Monte Carlo sample size, and the damping parameters $\sigma_2$ and $\sigma_3$. We evaluate robustness at two levels. The distribution-level test asks whether the generated sample distribution remains close to the intended cost-biased target, using Wasserstein distance. The instance-level test asks whether the decoded 3D output changes under different parameter settings, using Chamfer distance. Figure~\ref{fig:sensitivity_all} summarizes these sensitivities, and Extended Data Figure~\ref{fig:sensitivity} reports the full one-parameter sweep curves. \textcolor{blue}{Supplementary Note~\ref{subsec:si_samc_algorithm}} defines the role of these parameters in the SA-MC algorithm and \textcolor{blue}{Supplementary Note~\ref{subsec:si_damping_pd}} explains the damping rule used to maintain stable covariance updates.

At the distribution level, the dominant factor is the Monte Carlo sample size (Figure~\ref{fig:sensitivity_all}\textbf{a,b}). This is expected because SA-MC uses Monte Carlo samples to estimate how the reachable clean-design distribution spreads around the current trajectory. With too few samples, the covariance estimate is noisy and the guided sample distribution remains farther from the cost-biased target. Increasing the sample size steadily reduces the Wasserstein distance (Figure~\ref{fig:sensitivity_all}\textbf{c}), and the full sweep in Extended Data Figure~\ref{fig:sensitivity}\textbf{a} shows the same trend. The damping parameter $\sigma_2$ and the initial covariance scale $\gamma_0$ form the main secondary effects because they control the scale and curvature damping of the covariance update. Memory length and $\sigma_3$ show weaker effects over the tested ranges, while the perturbation seed has a smaller and less systematic influence. Thus, distribution matching is governed mainly by the statistical quality and scale control of the covariance estimate, not by a favourable random initialization.

The instance-level analysis gives a complementary message (Figure~\ref{fig:sensitivity_all}\textbf{a,b,d}). In the normalized Chamfer-distance ranking, $\sigma_2$ has the largest relative sensitivity, followed by Monte Carlo sample size and noise seed. However, the detailed $\sigma_2$ curve is nearly flat across the practical range used in our experiments, $\sigma_2\in[0.3,0.6]$, and the fitted trend is not significant at the 0.05 level (Figure~\ref{fig:sensitivity_all}\textbf{d}). Noticeable changes appear mainly at extreme damping choices. The complete instance-level sweeps in Extended Data Figure~\ref{fig:sensitivity}\textbf{b} further show that most parameter changes produce only small Chamfer-distance variation, with larger deviations concentrated at very small Monte Carlo sample sizes or extreme damping. Therefore, parameters that affect distribution-level matching do not necessarily cause large changes in an individual decoded 3D shape.

Taken together, these results indicate that SA-MC is robust once the covariance estimate is statistically reliable and the damping is kept in a moderate regime. In practice, the Monte Carlo sample size should be large enough to stabilize distributional alignment, while $\gamma_0$ and $\sigma_2$ should be chosen within non-extreme ranges. The remaining settings can be selected more flexibly in the tested regime. Importantly, varying these parameters over the tested ranges did not lead to generation collapse or unusable samples, indicating that the sensitivity of SA-MC reflects controllable changes in covariance estimation rather than a fragile failure mode. This robustness is also consistent with \textcolor{blue}{Supplementary Note~\ref{sec:si_offline_rl}}, where the same covariance-aware guidance remains stable in offline reinforcement learning and shows stronger stability than alternative guidance schemes. This supports the use of SA-MC as a practical guidance method: although it introduces more parameters than simpler guidance rules, the main behaviour is controlled by interpretable covariance-estimation choices, and the decoded outputs remain visually stable under moderate settings.

\section*{Discussion}

This work addresses a practical gap in learning-based aerodynamic inverse design: aerodynamic improvement should not be pursued independently of the visual features that define a design. This requirement is especially important in vehicle development, where exterior visual features are tied to design language, brand recognition and user perception. Existing deep-learning approaches can accelerate inverse design, but they are still often organized around aerodynamic objectives alone. In contrast, the present framework makes visual preservation part of the target itself. By representing designs consistent with a reference shape or view as a learned visual design distribution and reweighting this distribution by aerodynamic cost, visual-feature preservation and aerodynamic improvement become two terms of one probability target rather than two separate post-hoc criteria.

This probability view is useful because it gives a common target for tasks that are usually treated separately. In late-stage optimization, the method seeks a low-cost, high-probability design relative to an initial geometry. In early-stage generation, it samples lower-cost three-dimensional candidates relative to direct generation from the same input view. The same formulation also clarifies the roles of the individual algorithmic components. The SKL loss trains the cost predictor to recover the cost-biased distribution used downstream, not only pointwise cost values. The Density optimization rule keeps updates tied to the learned visual design distribution while reducing aerodynamic cost. SA-MC then uses covariance-aware guidance to sample the same target distribution more faithfully during generation. The controlled two-dimensional analysis makes these roles visible: MSE can miss target modes, cost-only optimization can leave supported visual regions, and simplified guidance rules can distort the desired target distribution.

The engineering experiments show that this distributional target is useful beyond a controlled setting. In vehicles, visual-feature-preserving optimization reduces drag relative to the initial vehicle, and guided generation reduces drag relative to direct generation from the same view, while the OpenFOAM fields and miniature wind-tunnel visualizations show consistent wake reduction. In blended-wing aircraft, the same framework improves the best valid drag-to-lift objective and avoids the negative-lift failures produced by cost-only optimization. The out-of-distribution view experiment further probes a harder regime in which the visual input is far from the aircraft training distribution. Although this setting is more variable, it shows that guidance can still provide an aerodynamic bias instead of simply following direct generation. Together with the parameter-sensitivity analysis and the offline reinforcement-learning results in \textcolor{blue}{Supplementary Note~\ref{sec:si_offline_rl}}, these experiments suggest that the proposed guidance principle is not tied to a single aerodynamic benchmark, but reflects a broader way of steering generative models with cost-biased probability targets.

Several limitations remain. The learned visual design distribution is a proxy for visual-feature preservation; it does not capture all aspects of industrial styling, brand-specific design language or human design judgement. In this work, visual preservation is assessed through geometric comparisons such as Chamfer distance and through qualitative inspection, which cannot replace expert styling evaluation. The aerodynamic cost predictor is also trained on a finite set of simulated data, so surrogate errors can affect both optimization and guidance, especially outside the training distribution. The physical validation combines CFD with miniature wind-tunnel wake visualization, but it does not replace full-scale testing across operating conditions. More broadly, the present framework focuses on aerodynamic performance and visual consistency; practical design must also consider structural constraints, packaging, cooling, safety, manufacturability and multi-condition robustness.

Future work should therefore connect probability-guided inverse design more tightly with closed-loop engineering validation. One direction is to combine the present predictor-guided framework with uncertainty estimation and active CFD sampling, so that uncertain or high-value candidates can be re-evaluated and added back into training. Another direction is to extend the cost-biased target to multiple objectives and constraints, allowing aerodynamic improvement to be balanced with structural, manufacturing and operational requirements. With broader design datasets, better calibrated surrogates and closer integration with high-fidelity simulation, visual-feature-preserving generative models could become practical front-end tools for narrowing the design space before expensive physical evaluation, while keeping the visual intent of the original design in the optimization loop.

\section*{Methods}\label{Methods}

\subsection*{Problem definition and latent representation}

We consider aerodynamic inverse design for three-dimensional geometries under specified operating conditions. Let $\mathbf{s}$ denote a physical geometry, represented as a point cloud, SDF or mesh, and let $o$ denote the operating condition, such as Reynolds number, Mach number or angle of attack. The training data are written as
\begin{equation}
    \mathcal{D}
    =
    \{(\mathbf{s}^{[i]},o^{[i]},y^{[i]})\}_{i=1}^{N},
\end{equation}
where $y^{[i]}$ is the aerodynamic cost obtained from high-fidelity simulation. In the vehicle experiments, $y$ is the drag coefficient $C_d$; in the aircraft experiments, $y$ is the drag-to-lift ratio $C_d/C_l$. We use square brackets $[i]$ only for dataset or mini-batch indices.

The main text defines inverse design through the learned cost-biased distribution $\hat{p}_\lambda(x\mid r,o)$, where $x$ is an abstract design representation and $r$ is the visual representation to be preserved or followed. In implementation, we represent $x$ in the latent space of Hunyuan3D 2.1. A geometry $\mathbf{s}$ is encoded by the Shape-VAE into a latent variable
\begin{equation}
    \mathbf{z}=E(\mathbf{s}),
\end{equation}
and a latent variable is decoded back into an SDF and mesh by the Shape-VAE decoder. Throughout the Methods, $\mathbf{z}$ without a time or iteration index denotes a generic latent design.

The notation distinguishes optimization and generation. For visual-feature-preserving optimization, we use superscripts for iterations, such as $\mathbf{z}^{(0)},\mathbf{z}^{(m)}$ and $\mathbf{z}^{(m+1)}$. For guided generation, we use subscripts for flow-matching time, such as $\mathbf{z}_t$ for an intermediate state and $\mathbf{z}_1$ for the final clean latent design.

The DiT-based flow-matching model of Hunyuan3D 2.1 provides the visual-conditional latent distribution $p_\theta(\mathbf{z}\mid r)$. The pretrained 3D generator is kept fixed. Aerodynamic adaptation is performed by biasing this fixed latent distribution with the learned aerodynamic cost. The latent target distribution used in both optimization and generation is
\begin{equation}
    \hat{p}_\lambda(\mathbf{z}\mid r,o)
    ~ \propto ~
    p_\theta(\mathbf{z}\mid r)
    \exp\left[-\lambda \hat{c}_\phi(\mathbf{z},o;\lambda)\right].
    \label{eq:method_latent_target}
\end{equation}
Here, $o$ denotes the physical operating condition, while $\lambda$ controls the strength of aerodynamic preference. The predictor is written as $\hat{c}_\phi(\mathbf{z},o;\lambda)$ because it is conditioned on the guidance strength used to define the cost-biased distribution.

\subsection*{Mini-batch probability loss}

The SKL objective in the main text is defined between continuous cost-biased distributions. During training, we approximate it within a mini-batch. For a mini-batch
$\mathcal{B}=\{(\mathbf{z}^{[i]},o^{[i]},y^{[i]})\}_{i=1}^{B}$,
the simulated and predicted costs define two discrete probability distributions:
\begin{equation}
    w^{[i]}
    =
    \frac{\exp[-\lambda y^{[i]}]}
    {\sum_{j=1}^{B}\exp[-\lambda y^{[j]}]},
    \qquad
    \hat{w}^{[i]}
    =
    \frac{\exp[-\lambda \hat{c}_\phi(\mathbf{z}^{[i]},o^{[i]};\lambda)]}
    {\sum_{j=1}^{B}\exp[-\lambda \hat{c}_\phi(\mathbf{z}^{[j]},o^{[j]};\lambda)]}.
    \label{eq:method_minibatch_weights}
\end{equation}
The training loss is the symmetric KL divergence between these two mini-batch distributions:
\begin{equation}
    \mathcal{L}_{\mathrm{SKL}}^{\mathcal{B}}
    =
    D_{\mathrm{KL}}(w\,\|\,\hat{w})
    +
    D_{\mathrm{KL}}(\hat{w}\,\|\,w)
    =
    \lambda
    \sum_{i=1}^{B}
    \left(
    w^{[i]}-\hat{w}^{[i]}
    \right)
    \left[
    \hat{c}_\phi
    \left(
    \mathbf{z}^{[i]},o^{[i]};\lambda
    \right)
    -
    y^{[i]}
    \right].
    \label{eq:method_minibatch_skl}
\end{equation}
This is the discrete form of the distributional loss used in training. The derivation from the continuous objective is provided in \textcolor{blue}{Supplementary Note~\ref{subsec:si_loss_statement}} and \textcolor{blue}{Supplementary Note~\ref{sec:si_derivations}}.

Different values of $\lambda$ define different cost-biased distributions. Therefore, $\lambda$ is provided as an additional input to the cost predictor. This allows one predictor to support different aerodynamic preference strengths instead of training a separate model for each $\lambda$.

\subsection*{visual-feature-preserving optimization and density gradient from flow matching}

For visual-feature-preserving optimization, we start from an existing 3D design $\mathbf{s}_0$ and encode it into the Shape-VAE latent variable
\begin{equation}
    \mathbf{z}^{(0)}=E(\mathbf{s}_0).
\end{equation}
The goal is to update $\mathbf{z}^{(m)}$ so that the decoded design has lower aerodynamic cost while remaining visually close to the reference representation $r$.

From equation~\eqref{eq:method_latent_target}, the latent update direction is
\begin{equation}
    \nabla_{\mathbf{z}}
    \log
    \hat{p}_\lambda(\mathbf{z}\mid r,o)
    =
    \nabla_{\mathbf{z}}
    \log
    p_\theta(\mathbf{z}\mid r)
    -
    \lambda
    \nabla_{\mathbf{z}}
    \hat{c}_\phi(\mathbf{z},o;\lambda).
    \label{eq:method_shape_opt_direction}
\end{equation}
The first term keeps the design close to designs visually consistent with $r$, and the second term points towards lower aerodynamic cost.

In practice, the aerodynamic term is obtained by back-propagating through the cost predictor. The density term is obtained from the pretrained flow-matching model. A forward evaluation of the flow-matching velocity field gives an estimate of the density direction of the visual prior:
\begin{equation}
    \widehat{
    \nabla_{\mathbf{z}}
    \log p_\theta(\mathbf{z}\mid r)
    }
    =
    \mathrm{ScoreFromFM}(\mathbf{z},r,t),
    \label{eq:method_score_from_fm}
\end{equation}
where $\mathrm{ScoreFromFM}(\cdot)$ denotes the standard conversion from the flow-matching velocity to the density direction under the chosen probability path. The schedule-dependent expression is provided in \textcolor{blue}{Supplementary Note~\ref{subsec:si_latent_opt_statement}}.

At optimization iteration $m$, we update the latent variable as
\begin{equation}
    \mathbf{z}^{(m+1)}
    =
    \mathbf{z}^{(m)}
    +
    \eta^{(m)}
    \left[
    \widehat{
    \nabla_{\mathbf{z}}
    \log p_\theta(\mathbf{z}^{(m)}\mid r)
    }
    -
    \lambda
    \nabla_{\mathbf{z}}
    \hat{c}_\phi(\mathbf{z}^{(m)},o;\lambda)
    \right].
    \label{eq:method_shape_opt_update}
\end{equation}
After the final iteration, the optimized latent variable $\mathbf{z}^{\star}$ is decoded into an SDF and converted to a mesh.

In simple terms, the flow-matching model tells the optimizer how to keep the design visually plausible, and the cost predictor tells it how to improve aerodynamic performance.

\subsection*{Guided generation with SA-MC}

For view-based generation, $r$ is the input design view. Without aerodynamic guidance, the flow-matching model transports a source latent variable $\mathbf{z}_0$ to a clean latent design $\mathbf{z}_1$ sampled from $p_\theta(\mathbf{z}_1\mid r)$. To sample from the cost-biased distribution, the velocity field is modified by a guidance:
\begin{equation}
    \mathbf{g}_t
    ~ \propto ~
    \nabla_{\mathbf{z}_t}
    \log
    \mathbb{E}_{p_\theta(\mathbf{z}_1\mid \mathbf{z}_t,r)}
    \left[
    \exp\left(
    -\lambda
    \hat{c}_\phi(\mathbf{z}_1,o;\lambda)
    \right)
    \right],
    \label{eq:method_generation_guidance}
\end{equation}
where $\mathbf{z}_t$ is the current latent state and $\mathbf{z}_1$ is the final clean latent design.

In practice, SA-MC estimates this guidance by first approximating the conditional clean-design distribution with a general Gaussian, and then using cost-weighted clean-latent samples to construct the correction direction. Specifically, we approximate
\begin{equation}
    p_\theta(\mathbf{z}_1\mid \mathbf{z}_t,r)
    \approx
    q_t(\mathbf{z}_1\mid \mathbf{z}_t,r)
    =
    \mathcal{N}
    \left(
    \boldsymbol{\mu}_t,
    \boldsymbol{\Sigma}_t
    \right),
    \qquad
    \boldsymbol{\Sigma}_t
    =
    \mathbf{L}_t\mathbf{L}_t^\top .
    \label{eq:method_samc_gaussian}
\end{equation}
Here, $\boldsymbol{\mu}_t$ is the clean latent prediction from the flow-matching model, and $\boldsymbol{\Sigma}_t$ is estimated from forward evaluations along the sampling trajectory. The detailed covariance update is provided in \textcolor{blue}{Supplementary Note~\ref{subsec:si_samc_statement}}.

We draw Monte Carlo clean-latent samples from this approximation:
\begin{equation}
    \mathbf{z}_{1}^{[s]}
    =
    \boldsymbol{\mu}_t
    +
    \mathbf{L}_t
    \boldsymbol{\epsilon}^{[s]},
    \qquad
    \boldsymbol{\epsilon}^{[s]}\sim\mathcal{N}(0,I),
    \qquad
    s=1,\ldots,S .
    \label{eq:method_samc_samples}
\end{equation}
The cost predictor scores these samples, and lower-cost samples receive larger weights. With
\begin{equation}\omega^{[s]}
=
\operatorname{softmax}_{s}
\left[
-\lambda
\hat{c}_\phi
\left(
\mathbf{z}_{1}^{[s]},o;\lambda
\right)
\right],\end{equation}
the SA-MC guidance is
\begin{equation}
    \mathbf{g}^{\mathrm{SA\mbox{-}MC}}_t
    =
    b_t
    \sum_{s=1}^{S}
    \omega^{[s]}
    \left(
    \mathbf{z}_{1}^{[s]}
    -
    \boldsymbol{\mu}_t
    \right)
    =
    b_t
    \sum_{s=1}^{S}
    \omega^{[s]}
    \mathbf{L}_t
    \boldsymbol{\epsilon}^{[s]} .
    \label{eq:method_samc_guidance}
\end{equation}
This expression gives the practical guidance direction: the flow-matching model proposes possible clean latent designs, the cost predictor evaluates them, and the weighted displacement points the trajectory towards lower-cost clean designs.

The guided flow is then
\begin{equation}
    \dot{\mathbf{z}}_t
    =
    v_t(\mathbf{z}_t\mid r)
    +
    \mathbf{g}^{\mathrm{SA\mbox{-}MC}}_t .
    \label{eq:method_samc_velocity}
\end{equation}
After the final sampling step, the guided latent variable $\mathbf{z}_1$ is decoded by the Shape-VAE into an SDF and mesh.

\bibliographystyle{unsrt}
\bibliography{references}

@book{hucho2013aerodynamics,
  title={Aerodynamics of Road Vehicles: from fluid mechanics to vehicle engineering},
  author={Hucho, Wolf-Heinrich},
  year={2013},
  publisher={Elsevier}
}

@book{national2015cost,
  title={Cost, Effectiveness, and Deployment of Fuel Economy Technologies for Light-duty Vehicles},
  author={National Research Council and Division on Engineering and Physical Sciences and Board on Energy and Environmental Systems and Committee on the Assessment of Technologies for Improving Fuel Economy of Light-Duty Vehicles and Phase},
  year={2015},
  publisher={National Academies Press}
}

@book{anderson2011ebook,
  title={Fundamentals of Aerodynamics},
  author={Anderson, John},
  year={2011},
  publisher={McGraw Hill}
}

@book{raymer2012aircraft,
  title={Aircraft Design: A Conceptual Approach},
  author={Raymer, Daniel},
  year={2012},
  publisher={American Institute of Aeronautics and Astronautics, Inc.}
}

@article{martins2022aerodynamic,
  title={Aerodynamic Design Optimization: Challenges and Perspectives},
  author={Martins, Joaquim RRA},
  journal={Computers \& Fluids},
  volume={239},
  pages={105391},
  year={2022},
  publisher={Elsevier}
}

@inproceedings{ranscombe2011characterizing,
  title={Characterizing and Evaluating Aesthetic Features in Vehicle Design},
  author={Ranscombe, Charlie and Hicks, Ben and Mullineux, Glen and Singh, Baljinder and others},
  booktitle={ICORD 11: Proceedings of the 3rd International Conference on Research into Design Engineering, Bangalore, India, 10.-12.01. 2011},
  pages={792--799},
  year={2011}
}

@article{li2022machine,
  title={Machine Learning in Aerodynamic Shape Optimization},
  author={Li, Jichao and Du, Xiaosong and Martins, Joaquim RRA},
  journal={Progress in Aerospace Sciences},
  volume={134},
  pages={100849},
  year={2022},
  publisher={Elsevier}
}

@article{he2018aerodynamic,
  title={An Aerodynamic Design Optimization Framework Using a Discrete Adjoint Approach with OpenFOAM},
  author={He, Ping and Mader, Charles A and Martins, Joaquim RRA and Maki, Kevin J},
  journal={Computers \& Fluids},
  volume={168},
  pages={285--303},
  year={2018},
  publisher={Elsevier}
}

@article{papadimitriou2008aerodynamic,
  title={Aerodynamic Shape Optimization Using First and Second Order Adjoint and Direct Approaches},
  author={Papadimitriou, Dimitrios I and Giannakoglou, Kyriakos C},
  journal={Archives of Computational Methods in Engineering},
  volume={15},
  number={4},
  pages={447--488},
  year={2008},
  publisher={Springer}
}

@article{queipo2005surrogate,
  title={Surrogate-Based Analysis and Optimization},
  author={Queipo, Nestor V and Haftka, Raphael T and Shyy, Wei and Goel, Tushar and Vaidyanathan, Rajkumar and Tucker, P Kevin},
  journal={Progress in aerospace sciences},
  volume={41},
  number={1},
  pages={1--28},
  year={2005},
  publisher={Elsevier}
}

@article{tran2024aerodynamics,
  title={Aerodynamics-Guided Machine Learning for Design Optimization of Electric Vehicles},
  author={Tran, Jonathan and Fukami, Kai and Inada, Kenta and Umehara, Daisuke and Ono, Yoshimichi and Ogawa, Kenta and Taira, Kunihiko},
  journal={Communications Engineering},
  volume={3},
  number={1},
  pages={174},
  year={2024},
  publisher={Nature Publishing Group UK London}
}

@article{vatani2025tripoptimizer,
  title={Tripoptimizer: Generative 3d Shape Optimization and Drag Prediction Using Triplane VAE Networks},
  author={Vatani, Parsa and Elrefaie, Mohamed and Nazarpour, Farhad and Ahmed, Faez},
  journal={arXiv preprint arXiv:2509.12224},
  year={2025}
}

@article{you2025physgen,
  title={PhysGen: Physically Grounded 3D Shape Generation for Industrial Design},
  author={You, Yingxuan and Zhao, Chen and Zhang, Hantao and Xu, Mingda and Fua, Pascal},
  journal={arXiv preprint arXiv:2512.00422},
  year={2025}
}

@inproceedings{hao3did,
  title={3DID: Direct 3D Inverse Design for Aerodynamics with Physics-Aware Optimization},
  author={Hao, Yuze and Zhu, Linchao and Yang, Yi},
  booktitle={The Thirty-ninth Annual Conference on Neural Information Processing Systems},
year={2025}
}

@article{chung2022diffusion,
  title={Diffusion Posterior Sampling for General Noisy Inverse Problems},
  author={Chung, Hyungjin and Kim, Jeongsol and Mccann, Michael T and Klasky, Marc L and Ye, Jong Chul},
  journal={arXiv preprint arXiv:2209.14687},
  year={2022}
}

@inproceedings{song2023loss,
  title={Loss-guided Diffusion Models for Plug-and-play Controllable Generation},
  author={Song, Jiaming and Zhang, Qinsheng and Yin, Hongxu and Mardani, Morteza and Liu, Ming-Yu and Kautz, Jan and Chen, Yongxin and Vahdat, Arash},
  booktitle={International Conference on Machine Learning},
  pages={32483--32498},
  year={2023},
  organization={PMLR}
}

@article{feng2025guidance,
  title={On the Guidance of Flow Matching},
  author={Feng, Ruiqi and Yu, Chenglei and Deng, Wenhao and Hu, Peiyan and Wu, Tailin},
  journal={arXiv preprint arXiv:2502.02150},
  year={2025}
}

@article{boystweedie,
  title={Tweedie Moment Projected Diffusions for Inverse Problems},
  author={Boys, Benjamin and Girolami, Mark and Pidstrigach, Jakiw and Reich, Sebastian and Mosca, Alan and Akyildiz, Omer Deniz},
  journal={Transactions on Machine Learning Research},
  year={2024}
}

@article{hunyuan3d2025hunyuan3d,
  title={Hunyuan3D 2.1: From Images to High-Fidelity 3D Assets with Production-Ready PBR Material},
  author={Hunyuan3D, Team and Yang, Shuhui and Yang, Mingxin and Feng, Yifei and Huang, Xin and Zhang, Sheng and He, Zebin and Luo, Di and Liu, Haolin and Zhao, Yunfei and others},
  journal={arXiv preprint arXiv:2506.15442},
  year={2025}
}

@inproceedings{jasak2007openfoam,
  title={OpenFOAM: A C++ library for complex physics simulations},
  author={Jasak, Hrvoje and Jemcov, Aleksandar and Tukovic, Zeljko and others},
  booktitle={International workshop on coupled methods in numerical dynamics},
  volume={1000},
  pages={1--20},
  year={2007},
  organization={Dubrovnik, Croatia)}
}

@inproceedings{elrefaie2024drivaernet,
  title={Drivaernet: A Parametric Car Dataset for Data-driven Aerodynamic Design and Graph-based Drag Prediction},
  author={Elrefaie, Mohamed and Dai, Angela and Ahmed, Faez},
  booktitle={International Design Engineering Technical Conferences and Computers and Information in Engineering Conference},
  volume={88360},
  pages={V03AT03A019},
  year={2024},
  organization={American Society of Mechanical Engineers}
}

@article{aultman2021effects,
  title={Effects of Wheel Rotation on Long-period Wake Dynamics of the DrivAer Fastback Model},
  author={Aultman, Matthew and Auza-Gutierrez, Rodrigo and Disotell, Kevin and Duan, Lian},
  journal={Fluids},
  volume={7},
  number={1},
  pages={19},
  year={2021},
  publisher={MDPI}
}

@article{byrd1994representations,
  title={Representations of Quasi-Newton Matrices and Their Use in Limited Memory Methods},
  author={Byrd, Richard H and Nocedal, Jorge and Schnabel, Robert B},
  journal={Mathematical Programming},
  volume={63},
  number={1},
  pages={129--156},
  year={1994},
  publisher={Springer}
}

@article{elrefaie2024drivaernet++,
  title={Drivaernet++: A Large-scale Multimodal Car Dataset with Computational Fluid Dynamics Simulations and Deep Learning Benchmarks},
  author={Elrefaie, Mohamed and Morar, Florin and Dai, Angela and Ahmed, Faez},
  journal={Advances in Neural Information Processing Systems},
  volume={37},
  pages={499--536},
  year={2024}
}

@article{fu2020d4rl,
  title={D4rl: Datasets for Deep Data-driven Reinforcement Learning},
  author={Fu, Justin and Kumar, Aviral and Nachum, Ofir and Tucker, George and Levine, Sergey},
  journal={arXiv preprint arXiv:2004.07219},
  year={2020}
}

@article{janner2022planning,
  title={Planning with Diffusion for Flexible Behavior Synthesis},
  author={Janner, Michael and Du, Yilun and Tenenbaum, Joshua B and Levine, Sergey},
  journal={arXiv preprint arXiv:2205.09991},
  year={2022}
}

@article{levine2018reinforcement,
  title={Reinforcement Learning and Control as Probabilistic Inference: Tutorial and Review},
  author={Levine, Sergey},
  journal={arXiv preprint arXiv:1805.00909},
  year={2018}
}

@book{nocedal2006numerical,
  title={Numerical Optimization},
  author={Nocedal, Jorge and Wright, Stephen J},
  year={2006},
  publisher={Springer}
}

@article{powell1978algorithms,
  title={Algorithms for Nonlinear Constraints that Use Lagrangian Functions},
  author={Powell, Michael JD},
  journal={Mathematical programming},
  volume={14},
  number={1},
  pages={224--248},
  year={1978},
  publisher={Springer}
}

@inproceedings{sung2025blendednet,
  title={Blendednet: A Blended Wing Body Aircraft Dataset and Surrogate Model for Aerodynamic Predictions},
  author={Sung, Nicholas and Spreizer, Steven and Elrefaie, Mohamed and Samuel, Kaira and Jones, Matthew C and Ahmed, Faez},
  booktitle={International Design Engineering Technical Conferences and Computers and Information in Engineering Conference},
  volume={89237},
  pages={V03BT03A049},
  year={2025},
  organization={American Society of Mechanical Engineers}
}

\section*{Acknowledgement}
The work is supported by the National Natural Science Foundation of China (No. 62276269 and No. 92270118) and the Beijing Natural Science Foundation (No. 1232009).

\section*{Author contributions} 
H.S. led the project. H.C., N.L., and L.C. jointly organized this project. H.C. and N.L. led the work on the experimental studies. N.L. was responsible for the theoretical proofs of the corresponding parts. L.C. provided support on the generative modeling components. H.S. supervised all aspects of the project. All authors participated in writing and revising the manuscript.

\section*{Corresponding authors} 
Hao Sun (\url{haosun@ruc.edu.cn}).

\section*{Competing interests}
The authors declare no competing interests.

\clearpage
\setcounter{figure}{0}
\renewcommand{\figurename}{Extended Data Figure}
\setcounter{table}{0}
\renewcommand{\tablename}{Extended Data Table}


\begin{figure}[t!]
 \centering
\includegraphics[width=0.90\linewidth]{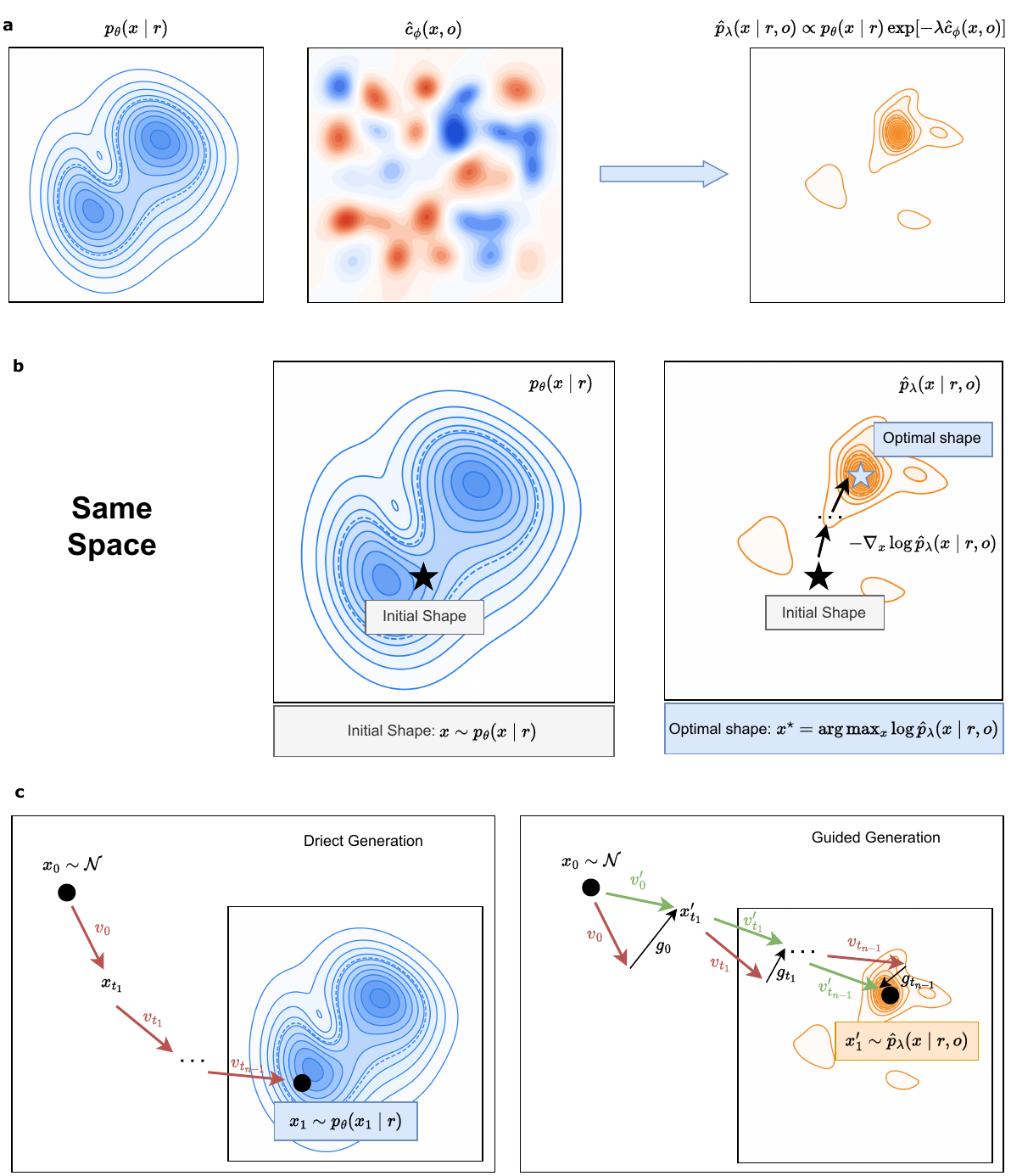}
\caption{
\textbf{Unified probabilistic view of aerodynamic optimization and generation.}
\textbf{a}, A cost-biased target distribution is constructed by reweighting the visual-conditional design distribution $p_{\mathrm{data}}(x\mid r)$ with the aerodynamic cost $c(x,o)$, yielding $p_\lambda(x\mid r,o)\propto p_{\mathrm{data}}(x\mid r)\exp[-\lambda c(x,o)]$. High-density regions correspond to designs that are visually plausible and aerodynamically favourable.
\textbf{b}, visual-feature-preserving optimization searches for a high-density point of this target distribution. The density term keeps the design close to the visual design manifold, while the cost term drives aerodynamic improvement.
\textbf{c}, Guided generation samples from the same target distribution. The original flow-matching velocity samples from the visual design prior, and aerodynamic guidance adds a correction $g_t$ that steers the trajectory towards lower-cost clean designs. Optimization and generation are therefore two complementary operations on the same cost-biased design distribution.
}
\label{fig:probability}
\end{figure}

\begin{figure}[t!]
  \centering
   \includegraphics[width=0.99\linewidth]{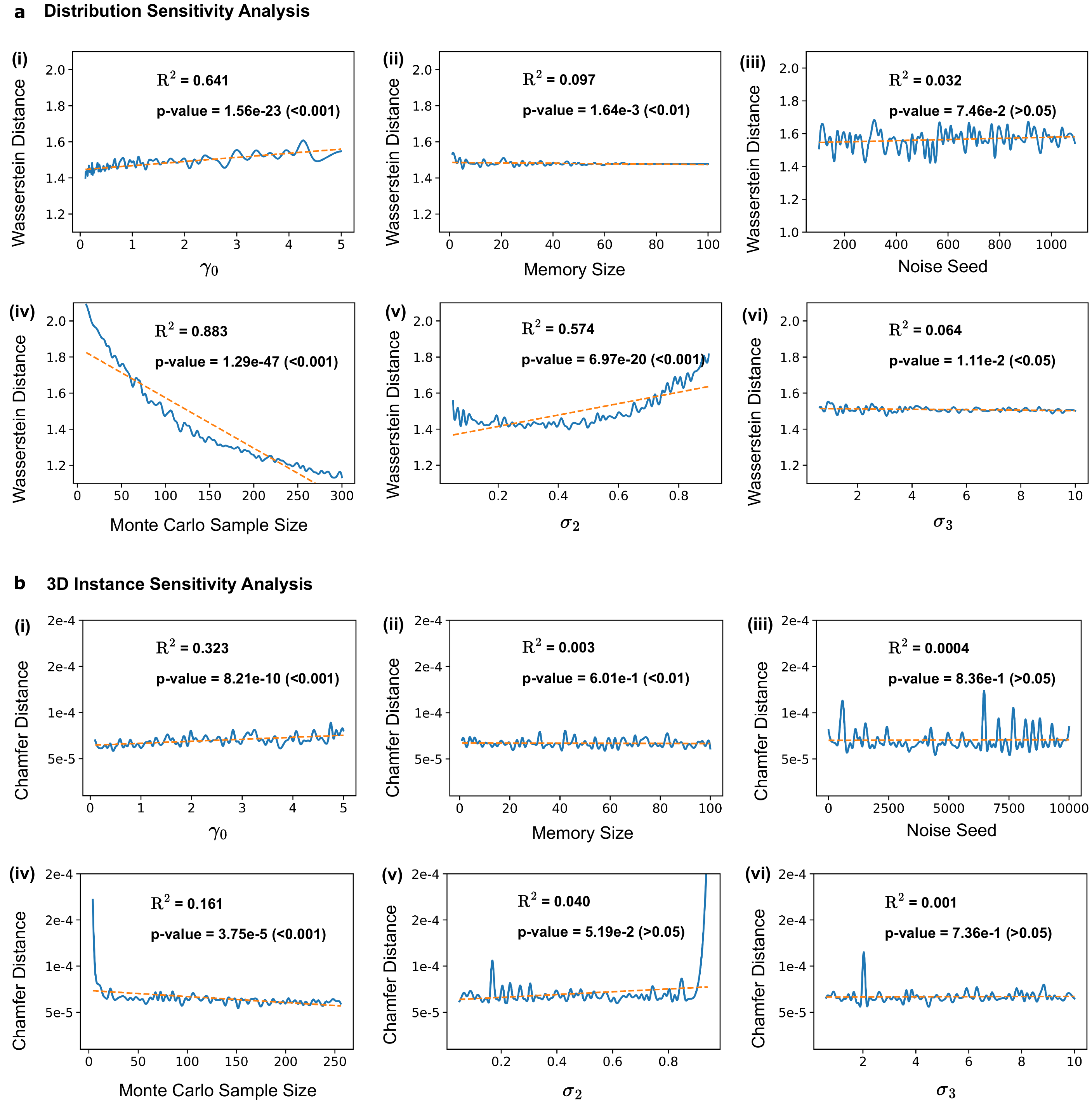}
\caption{\textbf{Parameter sensitivity of SA-MC guidance.}
\textbf{a}, Distribution-level sensitivity measured by Wasserstein distance while varying the initial covariance scale $\gamma_0$, memory size, noise seed, Monte Carlo sample size, and damping parameters $\sigma_2$ and $\sigma_3$. The Monte Carlo sample size has the clearest effect on distribution matching, while $\gamma_0$ and $\sigma_2$ also affect the target approximation. Memory size, noise seed and $\sigma_3$ have weaker effects over the tested ranges.
\textbf{b}, Instance-level sensitivity measured by Chamfer distance for generated 3D shapes under the same parameter sweeps. The individual 3D outputs remain comparatively stable for most settings, with noticeable changes mainly appearing under very small Monte Carlo sample sizes or extreme damping choices.}
   \label{fig:sensitivity}
\end{figure}

\clearpage
\setcounter{figure}{0}
\renewcommand{\figurename}{Supplementary Figure}
\setcounter{table}{0}
\renewcommand{\tablename}{Supplementary Table}

\begin{center}
{\Large\bfseries Supplementary Information\par}
\vspace{0.6em}
{\large for\par}
\vspace{0.6em}
{\Large\itshape Optimization and Generation in Aerodynamic Inverse Design\par}
\end{center}

\vspace{24pt}

\noindent This Supplementary Information provides additional details on the theoretical derivations, model implementation, datasets, CFD evaluation settings, miniature wind-tunnel experiments, sensitivity analyses, and additional experimental results for aerodynamic inverse design.

\clearpage

{\footnotesize
\tableofcontents
}

\clearpage

\section{Complete implementation framework}
\label{sec:si_theory}

This section gives the complete implementation logic used by the method. The aim is to make clear what is solved, why it is solved, and how each non-obvious term is computed. Supplementary Section~\ref{sec:si_derivations} is reserved for the theoretical results and proofs that justify the nontrivial steps in this section. Algorithmic pseudocode and time complexity are given in Supplementary Section~\ref{sec:si_algorithms}.

The notation follows the main manuscript. The visual representation or design condition is denoted by $r$, the operating condition by $o$, the Shape-VAE latent design variable by $\mathbf{z}$, and the guidance strength by $\lambda$. The learned visual design prior is $p_\theta(\mathbf{z}\mid r)$ and the aerodynamic predictor is $\hat{c}_\phi(\mathbf{z},o;\lambda)$. For compact derivations, later sections sometimes use a generic variable $\mathbf{x}$ and energy $J$; the latent implementation is obtained by substituting $\mathbf{x}=\mathbf{z}$ and $J(\mathbf{z})=\lambda\hat{c}_\phi(\mathbf{z},o;\lambda)$.

\subsection{Latent cost-biased design distribution}
\label{subsec:si_cost_biased_statement}

The common target of the whole framework is to improve aerodynamic performance while remaining in the high-density region of visually consistent designs represented by the visual prior. A cost-only objective is not sufficient for this purpose, because it can drive the latent design away from the visually consistent designs represented by $r$. We therefore bias the learned latent prior by the aerodynamic cost and use the following cost-biased design distribution:

\begin{equation}
    \hat{p}_\lambda(\mathbf{z}\mid r,o)
    =
    \frac{1}{Z_\lambda(r,o)}
    p_\theta(\mathbf{z}\mid r)
    \exp\left[-\lambda\hat{c}_\phi(\mathbf{z},o;\lambda)\right].
    \label{eq:si_cost_biased_latent}
\end{equation}
Here, $p_\theta(\mathbf{z}\mid r)$ preserves the visual design prior, while the exponential term increases the probability of latent designs with lower predicted aerodynamic cost. The guidance scale $\lambda$ controls the strength of this bias. The fact that Eq.~\eqref{eq:si_cost_biased_latent} is the variational optimum of a free-energy inverse-design problem is stated in Theorem~\ref{thm:si_distributional_target}.

This distribution is used in two ways. Visual-feature-preserving optimization returns a point estimate by searching for the maximum-density latent design,
\begin{equation}
    \mathbf{z}^{\star}
    =
    \arg\max_{\mathbf{z}}
    \hat{p}_\lambda(\mathbf{z}\mid r,o)
    =
    \arg\min_{\mathbf{z}}
    \left[
    \lambda\hat{c}_\phi(\mathbf{z},o;\lambda)
    -
    \log p_\theta(\mathbf{z}\mid r)
    \right].
    \label{eq:si_map_objective}
\end{equation}
Guided generation returns diverse candidates by sampling from the same distribution,
\begin{equation}
    \mathbf{z}^{[s]}
    \sim
    \hat{p}_\lambda(\mathbf{z}\mid r,o),
    \qquad
    s=1,\ldots,S.
    \label{eq:si_generation_sampling_target}
\end{equation}
Thus, the optimizer and the generator solve different computational forms of the same inverse-design target: optimization uses the MAP consequence of Theorem~\ref{thm:si_distributional_target}, while generation samples from the same theorem's distributional solution.

\subsection{Distributional training of the aerodynamic predictor}
\label{subsec:si_loss_statement}

The predictor is trained to support Eq.~\eqref{eq:si_cost_biased_latent}. Downstream optimization and generation do not use costs only as scalar labels; they use exponentiated costs to redistribute probability mass. The predictor should therefore reproduce the cost-induced probability redistribution, not only minimize pointwise error.

For a mini-batch $\mathcal{B}=\{(\mathbf{z}^{[i]},o^{[i]},y^{[i]})\}_{i=1}^{B}$, where $y^{[i]}$ is the simulated aerodynamic cost, the implementation first converts simulated and predicted costs into two discrete cost-biased distributions:
\begin{equation}
    w^{[i]}
    =
    \frac{\exp[-\lambda y^{[i]}]}{\sum_{j=1}^{B}\exp[-\lambda y^{[j]}]},
    \qquad
    \hat{w}^{[i]}
    =
    \frac{\exp[-\lambda\hat{c}_\phi(\mathbf{z}^{[i]},o^{[i]};\lambda)]}
    {\sum_{j=1}^{B}\exp[-\lambda\hat{c}_\phi(\mathbf{z}^{[j]},o^{[j]};\lambda)]}.
    \label{eq:si_minibatch_statement_weights}
\end{equation}
The predictor is then trained with the mini-batch symmetric-KL loss
\begin{equation}
    \mathcal{L}_{\mathrm{SKL}}^{\mathcal{B}}
    =
    D_{\mathrm{KL}}(w\,\|\,\hat{w})
    +
    D_{\mathrm{KL}}(\hat{w}\,\|\,w)
    =
    \lambda\sum_{i=1}^{B}
    \left(w^{[i]}-\hat{w}^{[i]}\right)
    \left[
    \hat{c}_\phi(\mathbf{z}^{[i]},o^{[i]};\lambda)-y^{[i]}
    \right].
    \label{eq:si_minibatch_statement_skl}
\end{equation}
The normalization in Eq.~\eqref{eq:si_minibatch_statement_weights} is computed within each mini-batch using a log-sum-exp implementation for numerical stability. The guidance scale $\lambda$ is provided as an input to the predictor, so a single predictor can represent different strengths of aerodynamic preference rather than requiring a separate model for each value of $\lambda$.

The compact mini-batch form in Eq.~\eqref{eq:si_minibatch_statement_skl} follows from the symmetric-KL identity in Theorem~\ref{thm:si_skl_identity}. Additional analysis of why this differs from MSE training is given in Supplementary Section~\ref{sec:si_supp_analysis}.

\subsection{visual-feature-preserving latent optimization}
\label{subsec:si_latent_opt_statement}

Visual-feature-preserving optimization uses the cost-biased distribution in Eq.~\eqref{eq:si_cost_biased_latent} as a MAP problem. The purpose is to improve the predicted aerodynamic cost while keeping the latent design in a high-density region of the learned design prior. Starting from an existing geometry $\mathbf{s}_{\mathrm{ref}}$, we encode it as
\begin{equation}
    \mathbf{z}^{(0)}=E(\mathbf{s}_{\mathrm{ref}}).
    \label{eq:si_initial_latent_workflow}
\end{equation}
The update direction is the gradient of the log target density:
\begin{equation}
    \nabla_{\mathbf{z}}\log\hat{p}_\lambda(\mathbf{z}\mid r,o)
    =
    \nabla_{\mathbf{z}}\log p_\theta(\mathbf{z}\mid r)
    -
    \lambda\nabla_{\mathbf{z}}\hat{c}_\phi(\mathbf{z},o;\lambda).
    \label{eq:si_unified_opt_gradient}
\end{equation}
The second term in Eq.~\eqref{eq:si_unified_opt_gradient} is obtained by back-propagating through the aerodynamic predictor. The non-obvious term is the density gradient $\nabla_{\mathbf{z}}\log p_\theta(\mathbf{z}\mid r)$, because the prior is represented by the flow-matching model rather than by an explicit density.

For the affine probability path used by the flow model, define
\begin{equation}
    a_t=\frac{\dot{\sigma}_t}{\sigma_t},
    \qquad
    b_t=\frac{\dot{\alpha}_t\sigma_t-\dot{\sigma}_t\alpha_t}{\sigma_t},
    \qquad
    s_t=\frac{b_t\sigma_t^2}{\alpha_t}.
    \label{eq:si_schedule_coefficients}
\end{equation}
By Theorem~\ref{thm:si_flow_score_clean_mean}, the flow velocity and the intermediate density score satisfy
\begin{equation}
    v_t(\mathbf{z}_t\mid r)
    =
    \frac{\dot{\alpha}_t}{\alpha_t}\mathbf{z}_t
    -
    s_t\nabla_{\mathbf{z}_t}\log p_t(\mathbf{z}_t\mid r),
    \label{eq:si_fm_score_relation}
\end{equation}
so a forward query of the pretrained flow model gives the density direction
\begin{equation}
    \nabla_{\mathbf{z}_t}\log p_t(\mathbf{z}_t\mid r)
    =
    \frac{
    \frac{\dot{\alpha}_t}{\alpha_t}\mathbf{z}_t
    -
    v_t(\mathbf{z}_t\mid r)
    }{s_t}.
    \label{eq:si_fm_score}
\end{equation}
During optimization, we evaluate this score at an intermediate noisy state for numerical stability, combine it with the cost-predictor gradient, and update the latent variable:
\begin{equation}
    \mathbf{z}^{(m+1)}
    =
    \mathbf{z}^{(m)}
    +
    \eta_m
    \left[
    \nabla_{\mathbf{z}}\log p_\theta(\mathbf{z}^{(m)}\mid r)
    -
    \lambda
    \nabla_{\mathbf{z}}\hat{c}_\phi(\mathbf{z}^{(m)},o;\lambda)
    \right].
    \label{eq:si_practical_latent_update}
\end{equation}
The optimized latent variable is decoded by the Shape-VAE into an SDF and mesh. The time-annealed implementation used in the experiments is given in Algorithm~\ref{alg:time_annealed_opt}. The score conversion in Eqs.~\eqref{eq:si_fm_score_relation}--\eqref{eq:si_fm_score} is the optimization part of Theorem~\ref{thm:si_flow_score_clean_mean}.

\subsection{Guided generation and the guidance correction}
\label{subsec:si_guidance_statement}

Guided generation uses Eq.~\eqref{eq:si_cost_biased_latent} as a sampling target. Without aerodynamic guidance, the flow-matching velocity field $v_t(\mathbf{z}_t\mid r)$ transports a source latent sample to the visual prior $p_\theta(\mathbf{z}_1\mid r)$. To sample instead from the cost-biased target, we modify the velocity field:
\begin{equation}
    v'_t(\mathbf{z}_t\mid r,o)
    =
    v_t(\mathbf{z}_t\mid r)
    +
    \mathbf{g}_t(\mathbf{z}_t,r,o).
    \label{eq:si_guided_velocity}
\end{equation}
The quantity that must be solved is the correction $\mathbf{g}_t$. This correction asks the following question at every intermediate state: among the clean latent designs that can be reached from $\mathbf{z}_t$, which directions lead to lower aerodynamic cost?

The clean-prediction mean from the current state, derived in Theorem~\ref{thm:si_flow_score_clean_mean}, is
\begin{equation}
    \boldsymbol{\mu}(\mathbf{z}_t)
    =
    \mathbb{E}_{p_\theta(\mathbf{z}_1\mid\mathbf{z}_t,r)}[\mathbf{z}_1]
    =
    -\frac{a_t}{b_t}\mathbf{z}_t
    +
    \frac{1}{b_t}v_t(\mathbf{z}_t\mid r).
    \label{eq:si_clean_mean_statement}
\end{equation}
Writing $J(\mathbf{z}_1)=\lambda\hat{c}_\phi(\mathbf{z}_1,o;\lambda)$, the cost-biased correction is
\begin{equation}
    \mathbf{g}_t
    =
    \frac{b_t\sigma_t^2}{\alpha_t}
    \nabla_{\mathbf{z}_t}
    \log
    \mathbb{E}_{p_\theta(\mathbf{z}_1\mid\mathbf{z}_t,r)}
    \left[\exp[-J(\mathbf{z}_1)]\right].
    \label{eq:si_exact_guidance_statement}
\end{equation}
Equation~\eqref{eq:si_exact_guidance_statement} is the coefficient-explicit form of the guidance equation in the main text. The only non-obvious object is the conditional clean-design distribution $p_\theta(\mathbf{z}_1\mid\mathbf{z}_t,r)$. Existing training-free guidance methods can be read as different approximations to this conditional distribution: DPS uses a point approximation at $\boldsymbol{\mu}(\mathbf{z}_t)$, while LGD-MC and SIM-MC use simplified Gaussian sampling. The score-form equivalence and the three estimator forms are stated and proved in Theorem~\ref{thm:loss_guided_estimators}.

\subsection{SA-MC solver for covariance-aware guided generation}
\label{subsec:si_samc_statement}

SA-MC is the solver used to evaluate Eq.~\eqref{eq:si_exact_guidance_statement} in high-dimensional 3D latent space. It replaces the unknown conditional clean-design distribution with a covariance-aware Gaussian approximation:
\begin{equation}
    p_\theta(\mathbf{z}_1\mid\mathbf{z}_t,r)
    \approx
    q_t(\mathbf{z}_1\mid\mathbf{z}_t,r)
    =
    \mathcal{N}(\boldsymbol{\mu}_t,\boldsymbol{\Sigma}_t),
    \qquad
    \boldsymbol{\Sigma}_t=\mathbf{L}_t\mathbf{L}_t^\top .
    \label{eq:si_samc_gaussian_workflow}
\end{equation}
The mean $\boldsymbol{\mu}_t$ is computed directly from Eq.~\eqref{eq:si_clean_mean_statement}. The difficult term is the covariance, or equivalently the standard-deviation map $\mathbf{L}_t$. The covariance-mismatch bound in Theorem~\ref{thm:lgd_sim_gap} shows why the covariance matters. A covariance-matched Gaussian should use
\begin{equation}
    \boldsymbol{\Sigma}^{\star}_t
    =
    \frac{\sigma_t^2}{\alpha_t}
    \nabla_{\mathbf{z}_t}\boldsymbol{\mu}_t(\mathbf{z}_t),
    \label{eq:si_covariance_matching_target}
\end{equation}
but directly forming the Jacobian $\nabla_{\mathbf{z}_t}\boldsymbol{\mu}_t$ is prohibitively expensive.

SA-MC estimates this covariance from the sampler trajectory. Let $\mathbf{J}_k\approx\nabla_{\mathbf{z}_k}v_{t_k}(\mathbf{z}_k)$ and define the local clean-prediction Jacobian proxy
\begin{equation}
    \mathbf{B}_{k}
    =
    -\frac{a_{t_k}}{b_{t_k}}I
    +
    \frac{1}{b_{t_k}}\mathbf{J}_{k}
    \approx
    \nabla_{\mathbf{z}_t}
    \mathbb{E}_{p_\theta(\mathbf{z}_1\mid\mathbf{z}_t,r)}[\mathbf{z}_1].
    \label{eq:si_B_definition_statement}
\end{equation}
The observed change between consecutive sampler states gives a secant pair:
\begin{equation}
    \mathbf{s}_k=\mathbf{z}_{k+1}-\mathbf{z}_k,
    \qquad
    \mathbf{r}_k=v_{t_{k+1}}(\mathbf{z}_{k+1})-v_{t_k}(\mathbf{z}_k),
    \qquad
    \mathbf{y}_k=-\frac{a_{t_k}}{b_{t_k}}\mathbf{s}_k+\frac{1}{b_{t_k}}\mathbf{r}_k.
    \label{eq:si_secant_statement}
\end{equation}
With $\rho_k=(\mathbf{y}_k^{\top}\mathbf{s}_k)^{-1}$, $\mathbf{V}_k=I-\rho_k\mathbf{s}_k\mathbf{y}_k^{\top}$, $u_k=b_{t_k}/b_{t_{k+1}}$ and $w_k=(a_{t_k}-a_{t_{k+1}})/b_{t_{k+1}}$, the covariance proxy is updated by
\begin{equation}
    \mathbf{B}_{k+1}
    =
    u_k\left(
    \mathbf{V}_k^{\top}\mathbf{B}_k\mathbf{V}_k
    +
    \rho_k\mathbf{y}_k\mathbf{y}_k^{\top}
    \right)
    +
    w_kI.
    \label{eq:si_updateB_statement}
\end{equation}
By Theorem~\ref{thm:updataB}, this update is the symmetric secant recursion for the clean-mean Jacobian proxy, so it extracts local covariance information without computing a full Jacobian.

Because directly storing or factorizing a dense $d\times d$ matrix is also expensive, the covariance proxy is maintained as a scalar plus a low-rank correction,
\begin{equation}
    \mathbf{B}_{k}
    =
    \gamma_k I
    +
    \mathbf{U}_k\boldsymbol{\Gamma}_k\mathbf{U}_k^{\top},
    \label{eq:si_compact_B_statement}
\end{equation}
where $\mathbf{U}_k$ stores a short memory of secant directions and $\boldsymbol{\Gamma}_k$ is a small dense matrix. The existence and recursion of this compact form are stated in Theorem~\ref{thm:Bk_compact_form}. In practice, the schedule factor in Eq.~\eqref{eq:si_covariance_matching_target} is absorbed into the compact matrix before sampling.

To draw samples, we need the standard-deviation map $\mathbf{L}_t$ satisfying $\boldsymbol{\Sigma}_t=\mathbf{L}_t\mathbf{L}_t^\top$. For a compact positive-definite matrix $\mathbf{B}=\gamma I+\mathbf{U}\boldsymbol{\Gamma}\mathbf{U}^{\top}$, SA-MC computes this map through a reduced QR factorization $\mathbf{U}=\mathbf{Q}\mathbf{R}$ and a small Cholesky factorization
\begin{equation}
    \mathbf{C}
    =
    \gamma I
    +
    \mathbf{R}\boldsymbol{\Gamma}\mathbf{R}^{\top}
    =
    \mathbf{L}_{\mathbf{C}}\mathbf{L}_{\mathbf{C}}^{\top}.
    \label{eq:si_small_cholesky_statement}
\end{equation}
By Theorem~\ref{thm:semi_numerical_sqrt}, the resulting matrix square root is applied implicitly as
\begin{equation}
    \mathbf{L}
    =
    \sqrt{\gamma}I
    +
    \mathbf{Q}
    \left(\mathbf{L}_{\mathbf{C}}-\sqrt{\gamma}I\right)
    \mathbf{Q}^{\top}.
    \label{eq:si_semi_sqrt_statement}
\end{equation}
This avoids materializing a dense covariance or square-root matrix.

After $\boldsymbol{\mu}_t$ and $\mathbf{L}_t$ have been obtained, SA-MC draws clean-latent proposals and weights them by aerodynamic preference,
\begin{equation}
    \mathbf{z}_1^{[s]}
    =
    \boldsymbol{\mu}_t+\mathbf{L}_t\boldsymbol{\epsilon}^{[s]},
    \qquad
    \boldsymbol{\epsilon}^{[s]}\sim\mathcal{N}(0,I),
    \qquad
    \omega^{[s]}
    =
    \frac{\exp[-J(\mathbf{z}_1^{[s]})]}
    {\sum_j\exp[-J(\mathbf{z}_1^{[j]})]}.
    \label{eq:si_samc_weighted_samples}
\end{equation}
The final guidance correction used in the guided flow is the tilted mean displacement,
\begin{equation}
    \mathbf{g}^{\mathrm{SA\mbox{-}MC}}_t
    =
    b_t
    \sum_{s=1}^{S}
    \omega^{[s]}
    \left(\mathbf{z}_1^{[s]}-\boldsymbol{\mu}_t\right)
    =
    b_t
    \sum_{s=1}^{S}
    \omega^{[s]}\mathbf{L}_t\boldsymbol{\epsilon}^{[s]}.
    \label{eq:si_samc_guidance_workflow}
\end{equation}
Equations~\eqref{eq:si_exact_guidance_statement}--\eqref{eq:si_samc_guidance_workflow} define the complete SA-MC solution flow: solve the guidance correction required by the cost-biased target, approximate the reachable clean-design distribution, estimate its covariance and standard-deviation map, and use cost-weighted Monte Carlo samples to construct the final correction. Theorem~\ref{thm:loss_guided_estimators}, Theorem~\ref{thm:lgd_sim_gap}, Theorem~\ref{thm:updataB}, Theorem~\ref{thm:Bk_compact_form} and Theorem~\ref{thm:semi_numerical_sqrt} provide the corresponding theoretical support. Algorithm~\ref{alg:si_samc_main} gives the same procedure in executable form.

\clearpage
\section{Supplementary theory and proofs}
\label{sec:si_derivations}

Supplementary Section~\ref{sec:si_theory} states the implementation flow. The purpose of the present section is different: it records the theoretical statements used by that flow and then gives the complete proofs for the nontrivial steps. The organization is therefore deliberately two-stage. Supplementary Section~\ref{subsec:si_theory_statements} first states the results invoked by the implementation; Supplementary Section~\ref{subsec:si_complete_proofs} then proves them in full.

\subsection{Theoretical statements used by the implementation}
\label{subsec:si_theory_statements}

\begin{theorem}[Cost-biased inverse-design distribution]
\label{thm:si_distributional_target}
Let $p_\theta(\mathbf{z}\mid r)$ be the learned latent prior and let $\hat{c}_\phi(\mathbf{z},o;\lambda)$ be the predicted aerodynamic cost. Among all distributions $q(\mathbf{z}\mid r,o)$ that are absolutely continuous with respect to $p_\theta$, the free-energy problem
\begin{equation}
    \min_q
    \left\{
    D_{\mathrm{KL}}
    \left(q(\mathbf{z}\mid r,o)\,\|\,p_\theta(\mathbf{z}\mid r)\right)
    +
    \lambda\,\mathbb{E}_{q}
    \left[\hat{c}_\phi(\mathbf{z},o;\lambda)\right]
    \right\}
    \label{eq:si_free_energy_problem}
\end{equation}
has the unique minimizer in Eq.~\eqref{eq:si_cost_biased_latent}, provided the normalizing constant is finite. The corresponding MAP estimator is Eq.~\eqref{eq:si_map_objective}, and its ascent direction is Eq.~\eqref{eq:si_unified_opt_gradient}.
\end{theorem}

\begin{theorem}[Symmetric-KL identity for distributional cost training]
\label{thm:si_skl_identity}
Let $J(\mathbf{z})$ be the simulated aerodynamic cost, $J_\phi(\mathbf{z})$ the predicted cost, and $p(\mathbf{z})$ a reference design distribution. Define the two cost-biased distributions
\begin{equation}
    p^{(J,\lambda)}(\mathbf{z})
    =
    \frac{p(\mathbf{z})\exp[-\lambda J(\mathbf{z})]}{Z_J},
    \qquad
    p^{(J_\phi,\lambda)}(\mathbf{z})
    =
    \frac{p(\mathbf{z})\exp[-\lambda J_\phi(\mathbf{z})]}{Z_\phi}.
    \label{eq:si_skl_population_distributions}
\end{equation}
Then the symmetric KL divergence between these two distributions admits the identity
\begin{equation}
    D_{\mathrm{KL}}(p^{(J,\lambda)}\|p^{(J_\phi,\lambda)})
    +
    D_{\mathrm{KL}}(p^{(J_\phi,\lambda)}\|p^{(J,\lambda)})
    =
    \lambda\,\mathbb{E}_{p}
    \left[
        \left(w_J(\mathbf{z})-w_\phi(\mathbf{z})\right)
        \left(J_\phi(\mathbf{z})-J(\mathbf{z})\right)
    \right],
    \label{eq:si_skl_population_identity}
\end{equation}
where $w_J(\mathbf{z})=\exp[-\lambda J(\mathbf{z})]/Z_J$ and $w_\phi(\mathbf{z})=\exp[-\lambda J_\phi(\mathbf{z})]/Z_\phi$. Replacing the population expectation and normalizers by their mini-batch estimates gives Eq.~\eqref{eq:si_minibatch_statement_skl}.
\end{theorem}

\begin{theorem}[Flow-score and clean-mean identities]
\label{thm:si_flow_score_clean_mean}
For the affine probability path used by the flow-matching prior, with coefficients defined in Eq.~\eqref{eq:si_schedule_coefficients}, the intermediate velocity and density score satisfy Eq.~\eqref{eq:si_fm_score_relation}. Hence a forward velocity evaluation gives the score estimator in Eq.~\eqref{eq:si_fm_score}. The same probability path gives the one-step clean prediction mean in Eq.~\eqref{eq:si_clean_mean_statement}.
\end{theorem}

For guided generation, define the tilted conditional
\begin{equation}
    p_{1|t}^{(J)}(\mathbf{x}_1\mid\mathbf{x}_t)
    =
    \frac{p_{1|t}(\mathbf{x}_1\mid\mathbf{x}_t)\exp[-J(\mathbf{x}_1)]}
    {\mathbb{E}_{p_{1|t}(\mathbf{x}_1\mid\mathbf{x}_t)}[\exp[-J(\mathbf{x}_1)]]}.
    \label{eq:si_tilted_conditional_statement}
\end{equation}
The pathwise guidance in Eq.~\eqref{eq:si_exact_guidance_statement} is equivalently written in score form as
\begin{equation}
    \mathbf{g}_t
    =
    \frac{b_t\sigma_t^2}{\alpha_t}
    \mathbb{E}_{p_{1|t}^{(J)}(\mathbf{x}_1\mid\mathbf{x}_t)}
    \left[
        \nabla_{\mathbf{x}_t}\log p_{1|t}(\mathbf{x}_1\mid\mathbf{x}_t)
    \right].
    \label{eq:gt_score_based}
\end{equation}

\begin{theorem}[Loss-guided estimators]
\label{thm:loss_guided_estimators}
Let $\boldsymbol{\epsilon}^{(i)}\sim\mathcal{N}(\mathbf{0},\mathbf{I})$ for $i=1,\dots,n$, and let
\begin{equation*}
    \boldsymbol{\mu}(\mathbf{x}_t)
    \triangleq
    \mathbb{E}_{p_{1|t}(\mathbf{x}_1\mid\mathbf{x}_t)}[\mathbf{x}_1]
    =
    -\frac{a_t}{b_t}\mathbf{x}_t+\frac{1}{b_t}\mathbf{v}_t(\mathbf{x}_t).
\end{equation*}
Different choices of the proposal $q_{1|t}(\mathbf{x}_1\mid\mathbf{x}_t)$ yield the following stochastic estimators of $\mathbf{g}_t$.

\textbf{(i) DPS.}
For a Dirac proposal $q_{1|t}(\mathbf{x}_1\mid\mathbf{x}_t)=\delta(\mathbf{x}_1-\boldsymbol{\mu}(\mathbf{x}_t))$,
\begin{equation}
    \mathbf{g}_t^{\mathrm{LGD\text{-}DPS}}
    =
    -\frac{b_t\sigma_t^2}{\alpha_t}
    \nabla_{\mathbf{x}_t}J(\boldsymbol{\mu}(\mathbf{x}_t)).
    \label{eq:dps}
\end{equation}

\textbf{(ii) LGD-MC.}
For a Gaussian proposal $q_{1|t}(\mathbf{x}_1\mid\mathbf{x}_t)=\mathcal{N}(\boldsymbol{\mu}(\mathbf{x}_t),\boldsymbol{\Sigma}_t)$ with $\boldsymbol{\Sigma}_t=\mathbf{L}_t\mathbf{L}_t^\top$ and samples $\mathbf{x}_1^{(i)}=\boldsymbol{\mu}(\mathbf{x}_t)+\mathbf{L}_t\boldsymbol{\epsilon}^{(i)}$,
\begin{equation}
    \mathbf{g}_t^{\mathrm{LGD\text{-}MC}}
    =
    \frac{b_t\sigma_t^2}{\alpha_t}
    \nabla_{\mathbf{x}_t}
    \log\left[\frac{1}{n}\sum_{i=1}^{n}\exp[-J(\mathbf{x}_1^{(i)})]\right].
    \label{eq:lgd_mc}
\end{equation}

\textbf{(iii) SIM-MC.}
For a Gaussian proposal $q_{1|t}(\mathbf{x}_1\mid\mathbf{x}_t)=\mathcal{N}(\boldsymbol{\mu}(\mathbf{x}_t),\boldsymbol{\Sigma}_t(\mathbf{x}_t))$, define $\mathbf{z}_t^{(i)}=\mathbf{L}_t(\mathbf{x}_t)\boldsymbol{\epsilon}^{(i)}$ and $\mathbf{x}_1^{(i)}=\boldsymbol{\mu}(\mathbf{x}_t)+\mathbf{z}_t^{(i)}$. Then
\begin{equation}
    \mathbf{g}_t^{\mathrm{SIM\text{-}MC}}
    =
    b_t\sum_{i=1}^{n}
    w_{\{\mathbf{x}_1^{(i)}\}}^{(J)}(\mathbf{x}_1^{(i)})\,\mathbf{z}_t^{(i)},
    \label{sim-mc}
\end{equation}
where $w_{\{\mathbf{x}_1^{(i)}\}}^{(J)}$ denotes normalized exponential weights over the Monte Carlo samples.
\end{theorem}

\begin{theorem}[Covariance mismatch in LGD and SIM]
\label{thm:lgd_sim_gap}
Let $q_{1|t}(\mathbf{x}_1)=\mathcal{N}(\boldsymbol{\mu}(\mathbf{x}_t),\boldsymbol{\Sigma}_t)$ and $q_t(\mathbf{z}_t)=\mathcal{N}(\mathbf{0},\boldsymbol{\Sigma}_t(\mathbf{x}_t))$. Define $J_t(\mathbf{z}_t)=J(\boldsymbol{\mu}(\mathbf{x}_t)+\mathbf{z}_t)$ and
\begin{equation}
    e_t(\mathbf{x}_t)
    =
    \left\|
    \boldsymbol{\Sigma}_t
    -
    \frac{\sigma_t^2}{\alpha_t}
    \nabla_{\mathbf{x}_t}
    \mathbb{E}_{p_{1|t}(\mathbf{x}_1\mid\mathbf{x}_t)}[\mathbf{x}_1]
    \right\|_2^2.
    \label{eq:si_covariance_gap_statement}
\end{equation}
Then
\begin{equation}
    \left\|\mathbf{g}_t^{\mathrm{LGD}}-\mathbf{g}_t^{\mathrm{SIM}}\right\|_2^2
    \leq
    b_t^2\left\|\boldsymbol{\Sigma}_t^{-1}\right\|_2^2
    \mathbb{E}_{q_t^{(J_t)}(\mathbf{z}_t)}[\|\mathbf{z}_t\|_2^2]
    e_t(\mathbf{x}_t).
    \label{eq:si_lgd_sim_gap_bound}
\end{equation}
\end{theorem}

\begin{theorem}[Secant covariance update]
\label{thm:updataB}
Let $\mathbf{J}_k\approx\nabla_{\mathbf{x}_k}\mathbf{v}_{t_k}(\mathbf{x}_k)$ and define
\begin{equation*}
    \mathbf{B}_k
    =
    -\frac{a_{t_k}}{b_{t_k}}\mathbf{I}
    +
    \frac{1}{b_{t_k}}\mathbf{J}_k
    \approx
    \nabla_{\mathbf{x}_t}\mathbb{E}_{p_{1|t}(\mathbf{x}_1\mid\mathbf{x}_t)}[\mathbf{x}_1].
\end{equation*}
Let $\mathbf{s}_k=\mathbf{x}_{k+1}-\mathbf{x}_k$, $\mathbf{r}_k=\mathbf{v}_{k+1}-\mathbf{v}_k$, $\mathbf{y}_k=-a_{t_k}\mathbf{s}_k/b_{t_k}+\mathbf{r}_k/b_{t_k}$, $\rho_k=(\mathbf{y}_k^\top\mathbf{s}_k)^{-1}$, $\mathbf{V}_k=\mathbf{I}-\rho_k\mathbf{s}_k\mathbf{y}_k^\top$, $u_k=b_{t_k}/b_{t_{k+1}}$, and $w_k=(a_{t_k}-a_{t_{k+1}})/b_{t_{k+1}}$. Then the symmetric secant recursion is Eq.~\eqref{eq:si_updateB_statement}.
\end{theorem}

\begin{theorem}[Compact representation of the accumulated update]
\label{thm:Bk_compact_form}
If $\mathbf{B}_0=\gamma_0\mathbf{I}$ and the accepted secant pairs are accumulated through Theorem~\ref{thm:updataB}, then $\mathbf{B}_k$ admits the compact representation
\begin{equation}
    \mathbf{B}_k=\gamma_k\mathbf{I}+\mathbf{U}_k\boldsymbol{\Gamma}_k\mathbf{U}_k^\top,
    \label{eq:si_compact_full_statement}
\end{equation}
where $\mathbf{U}_{k+1}=[\mathbf{U}_k,\mathbf{s}_k,\mathbf{y}_k]$, $\gamma_{k+1}=u_k\gamma_k+w_k$, and, with
\begin{equation*}
    \mathbf{p}_k=\boldsymbol{\Gamma}_k\mathbf{U}_k^\top\mathbf{s}_k,
    \qquad
    \tau_k=\mathbf{s}_k^\top\mathbf{U}_k\boldsymbol{\Gamma}_k\mathbf{U}_k^\top\mathbf{s}_k,
    \qquad
    \delta_k=\tau_k+\gamma_k\mathbf{s}_k^\top\mathbf{s}_k,
\end{equation*}
$\boldsymbol{\Gamma}_k$ obeys
\begin{equation}
    \boldsymbol{\Gamma}_{k+1}
    =
    u_k
    \begin{bmatrix}
    \boldsymbol{\Gamma}_k & \mathbf{0} & -\rho_k\mathbf{p}_k\\
    \mathbf{0}^\top & 0 & -\gamma_k\rho_k\\
    -\rho_k\mathbf{p}_k^\top & -\gamma_k\rho_k & \rho_k(1+\rho_k\delta_k)
    \end{bmatrix}.
    \label{eq:si_gamma_recursion_statement}
\end{equation}
\end{theorem}

\begin{theorem}[Semi-numerical square root for SA-MC]
\label{thm:semi_numerical_sqrt}
Let $\mathbf{B}=\gamma\mathbf{I}+\mathbf{U}\boldsymbol{\Gamma}\mathbf{U}^\top$ with $\gamma>0$, $\mathbf{U}\in\mathbb{R}^{d\times m}$, and $\boldsymbol{\Gamma}\in\mathbb{R}^{m\times m}$. Let $\mathbf{U}=\mathbf{Q}\mathbf{R}$ be a reduced QR factorization and let
\begin{equation*}
    \mathbf{C}=\gamma\mathbf{I}+\mathbf{R}\boldsymbol{\Gamma}\mathbf{R}^\top
    =\mathbf{L}_{\mathbf{C}}\mathbf{L}_{\mathbf{C}}^\top.
\end{equation*}
Then
\begin{equation}
    \mathbf{L}
    =
    \mathbf{Q}(\mathbf{L}_{\mathbf{C}}-\sqrt{\gamma}\mathbf{I})\mathbf{Q}^\top
    +
    \sqrt{\gamma}\mathbf{I}
    \label{eq:si_sqrt_theorem_statement}
\end{equation}
satisfies $\mathbf{B}=\mathbf{L}\mathbf{L}^\top$.
\end{theorem}

\subsection{Complete proofs}
\label{subsec:si_complete_proofs}

\subsubsection{Proof of Theorem~\ref{thm:si_distributional_target}}
\begin{proof}
Write $q$ and $p_\theta$ for $q(\mathbf{z}\mid r,o)$ and $p_\theta(\mathbf{z}\mid r)$, respectively. Enforcing $\int q(\mathbf{z})\,\mathrm{d}\mathbf{z}=1$ with a Lagrange multiplier $\eta$, the Lagrangian of Eq.~\eqref{eq:si_free_energy_problem} is
\begin{equation}
    \mathcal{L}(q)
    =
    \int q(\mathbf{z})
    \left[
        \log\frac{q(\mathbf{z})}{p_\theta(\mathbf{z}\mid r)}
        +
        \lambda\hat{c}_\phi(\mathbf{z},o;\lambda)
    \right]\mathrm{d}\mathbf{z}
    +
    \eta\left(\int q(\mathbf{z})\,\mathrm{d}\mathbf{z}-1\right).
    \label{eq:si_free_energy_lagrangian}
\end{equation}
The first variation with respect to $q$ gives
\begin{equation}
    \log\frac{q(\mathbf{z})}{p_\theta(\mathbf{z}\mid r)}
    +1+
    \lambda\hat{c}_\phi(\mathbf{z},o;\lambda)
    +\eta=0.
    \label{eq:si_free_energy_stationary}
\end{equation}
Solving Eq.~\eqref{eq:si_free_energy_stationary} yields
\begin{equation}
    q(\mathbf{z}\mid r,o)
    =
    C\,p_\theta(\mathbf{z}\mid r)
    \exp[-\lambda\hat{c}_\phi(\mathbf{z},o;\lambda)].
\end{equation}
The constant $C$ is determined by normalization, giving $C=1/Z_\lambda(r,o)$ and therefore Eq.~\eqref{eq:si_cost_biased_latent}. The KL term is strictly convex in $q$ on the support of $p_\theta$, while the expectation term is linear, so the stationary point is the unique minimizer when $Z_\lambda(r,o)<\infty$. Taking $-\log$ of Eq.~\eqref{eq:si_cost_biased_latent} and dropping the constant $\log Z_\lambda(r,o)$ gives Eq.~\eqref{eq:si_map_objective}. Differentiating $\log\hat{p}_\lambda(\mathbf{z}\mid r,o)$ with respect to $\mathbf{z}$ gives Eq.~\eqref{eq:si_unified_opt_gradient}.
\end{proof}

\subsubsection{Proof of Theorem~\ref{thm:si_skl_identity}}
\begin{proof}
For compactness write $p_J=p^{(J,\lambda)}$ and $p_\phi=p^{(J_\phi,\lambda)}$. The first KL term is
\begin{align}
    D_{\mathrm{KL}}(p_J\|p_\phi)
    &=
    \mathbb{E}_{p_J}
    \left[
        \log\frac{p_J(\mathbf{z})}{p_\phi(\mathbf{z})}
    \right] \\
    &=
    \mathbb{E}_{p_J}
    \left[
        -\lambda J(\mathbf{z})-\log Z_J
        +\lambda J_\phi(\mathbf{z})+\log Z_\phi
    \right] \\
    &=
    \lambda\mathbb{E}_{p_J}[J_\phi(\mathbf{z})-J(\mathbf{z})]
    +
    \log Z_\phi-
    \log Z_J.
    \label{eq:si_forward_skl_derivation}
\end{align}
Similarly,
\begin{equation}
    D_{\mathrm{KL}}(p_\phi\|p_J)
    =
    -\lambda\mathbb{E}_{p_\phi}[J_\phi(\mathbf{z})-J(\mathbf{z})]
    +
    \log Z_J-
    \log Z_\phi.
    \label{eq:si_reverse_skl_derivation}
\end{equation}
Adding Eqs.~\eqref{eq:si_forward_skl_derivation} and~\eqref{eq:si_reverse_skl_derivation} cancels the log-normalizers and gives
\begin{equation}
    \lambda
    \left(
    \mathbb{E}_{p_J}[J_\phi-J]
    -
    \mathbb{E}_{p_\phi}[J_\phi-J]
    \right).
\end{equation}
Since $p_J(\mathbf{z})=p(\mathbf{z})w_J(\mathbf{z})$ and $p_\phi(\mathbf{z})=p(\mathbf{z})w_\phi(\mathbf{z})$, this is exactly Eq.~\eqref{eq:si_skl_population_identity}. Replacing $p$ by the empirical measure on a mini-batch and replacing $J,J_\phi$ by $y^{[i]},\hat{c}_\phi(\mathbf{z}^{[i]},o^{[i]};\lambda)$ gives the normalized weights in Eq.~\eqref{eq:si_minibatch_statement_weights} and the mini-batch loss in Eq.~\eqref{eq:si_minibatch_statement_skl}.
\end{proof}

\subsubsection{Proof of Theorem~\ref{thm:si_flow_score_clean_mean}}
\begin{proof}
The score conversion in Eq.~\eqref{eq:si_fm_score_relation} is the standard affine-path relation between the flow-matching velocity and the score of the intermediate density. With $s_t=b_t\sigma_t^2/\alpha_t$, rearranging Eq.~\eqref{eq:si_fm_score_relation} immediately gives Eq.~\eqref{eq:si_fm_score}.

It remains to derive the clean-mean identity. For the affine path, write
\begin{equation}
    \mathbf{x}_t=\alpha_t\mathbf{x}_1+\sigma_t\boldsymbol{\epsilon},
    \qquad
    \boldsymbol{\epsilon}\sim\mathcal{N}(\mathbf{0},\mathbf{I}).
    \label{eq:si_affine_path_for_mean}
\end{equation}
The conditional velocity is the conditional expectation of the path derivative:
\begin{equation}
    \mathbf{v}_t(\mathbf{x}_t)
    =
    \mathbb{E}
    \left[
        \dot{\alpha}_t\mathbf{x}_1+\dot{\sigma}_t\boldsymbol{\epsilon}
        \mid \mathbf{x}_t
    \right].
    \label{eq:si_conditional_velocity_mean_proof}
\end{equation}
Substituting $\boldsymbol{\epsilon}=(\mathbf{x}_t-\alpha_t\mathbf{x}_1)/\sigma_t$ into Eq.~\eqref{eq:si_conditional_velocity_mean_proof} gives
\begin{align}
    \mathbf{v}_t(\mathbf{x}_t)
    &=
    \frac{\dot{\sigma}_t}{\sigma_t}\mathbf{x}_t
    +
    \left(\dot{\alpha}_t-\frac{\dot{\sigma}_t\alpha_t}{\sigma_t}\right)
    \mathbb{E}[\mathbf{x}_1\mid\mathbf{x}_t] \\
    &=
    a_t\mathbf{x}_t+b_t\boldsymbol{\mu}(\mathbf{x}_t),
    \label{eq:si_velocity_clean_mean_linear}
\end{align}
where $a_t=\dot{\sigma}_t/\sigma_t$, $b_t=(\dot{\alpha}_t\sigma_t-\dot{\sigma}_t\alpha_t)/\sigma_t$, and $\boldsymbol{\mu}(\mathbf{x}_t)=\mathbb{E}[\mathbf{x}_1\mid\mathbf{x}_t]$. Solving Eq.~\eqref{eq:si_velocity_clean_mean_linear} for $\boldsymbol{\mu}(\mathbf{x}_t)$ gives Eq.~\eqref{eq:si_clean_mean_statement}.
\end{proof}

\subsubsection{Proof of Theorem~\ref{thm:loss_guided_estimators}}
\begin{proof}
We derive each estimator by specifying the proposal $q_{1|t}(\mathbf{x}_1\mid \mathbf{x}_t)$ in the LGD and SIM formulations. 
Let $\boldsymbol{\epsilon}^{(i)}\sim\mathcal{N}(\mathbf{0},\mathbf{I})$, and recall the conditional mean
\begin{equation}
\boldsymbol{\mu}(\mathbf{x}_{t})
\triangleq \mathbb{E}_{p_{1|t}(\mathbf{x}_{1}\mid \mathbf{x}_t)}[\mathbf{x}_{1}]
= -\frac{a_{t}}{b_{t}}\mathbf{x}_{t} + \frac{1}{b_{t}}\mathbf{v}_{t}(\mathbf{x}_{t}).
\end{equation}

\paragraph{LGD (DPS and LGD-MC).}
Define the log-partition function
\begin{equation}
Z_t(\mathbf{x}_t)
\;\triangleq\;
\mathbb{E}_{q_{1|t}(\mathbf{x}_1\mid \mathbf{x}_t)}\!\left[\exp\!\big(-J(\mathbf{x}_1)\big)\right]
=
\int q_{1|t}(\mathbf{x}_1\mid\mathbf{x}_t)\,\exp\!\big(-J(\mathbf{x}_1)\big)\,d\mathbf{x}_1 .
\end{equation}
Then
\begin{equation}
\nabla_{\mathbf{x}_t}\log Z_t(\mathbf{x}_t)
=
\frac{1}{Z_t(\mathbf{x}_t)}\nabla_{\mathbf{x}_t} Z_t(\mathbf{x}_t).
\end{equation}
Under mild regularity conditions that allow exchanging differentiation and integration, differentiating under the integral sign gives
\begin{align}
\nabla_{\mathbf{x}_t} Z_t(\mathbf{x}_t)
&=
\int \exp\!\big(-J(\mathbf{x}_1)\big)\,\nabla_{\mathbf{x}_t} q_{1|t}(\mathbf{x}_1\mid\mathbf{x}_t)\,d\mathbf{x}_1 \nonumber\\
&=
\int q_{1|t}(\mathbf{x}_1\mid\mathbf{x}_t)\,\exp\!\big(-J(\mathbf{x}_1)\big)\,
\nabla_{\mathbf{x}_t}\log q_{1|t}(\mathbf{x}_1\mid\mathbf{x}_t)\,d\mathbf{x}_1 .
\end{align}
Introduce the exponentially tilted distribution
\begin{equation}
q_{1|t}^{(J)}(\mathbf{x}_1\mid\mathbf{x}_t)
\;\triangleq\;
\frac{q_{1|t}(\mathbf{x}_1\mid\mathbf{x}_t)\exp\!\big(-J(\mathbf{x}_1)\big)}
{Z_t(\mathbf{x}_t)}.
\end{equation}
Substituting into the gradient of $\log Z_t$ yields the identity
\begin{equation}
\nabla_{\mathbf{x}_t}\log Z_t(\mathbf{x}_t)
=
\mathbb{E}_{q_{1|t}^{(J)}(\mathbf{x}_1\mid\mathbf{x}_t)}
\!\left[\nabla_{\mathbf{x}_t}\log q_{1|t}(\mathbf{x}_1\mid\mathbf{x}_t)\right].
\end{equation}
Therefore the LGD guidance direction can be written as
\begin{equation}
\mathbf{g}_{t}^{\mathrm{LGD}}
=\frac{b_{t}\sigma_{t}^{2}}{\alpha_{t}}\,
\nabla_{\mathbf{x}_{t}}
\log \mathbb{E}_{q_{1|t}(\mathbf{x}_{1}\mid \mathbf{x}_t)}\!\left[\exp\!\big(-J(\mathbf{x}_{1})\big)\right].
\end{equation}

When $q_{1|t}(\cdot\mid\mathbf{x}_t)$ is reparameterizable, i.e., $\mathbf{x}_1=\mathcal{T}(\mathbf{x}_t,\boldsymbol{\epsilon})$ with
$\boldsymbol{\epsilon}\sim\mathcal{N}(\mathbf{0},\mathbf{I})$, we have
\begin{equation}
Z_t(\mathbf{x}_t)=
\mathbb{E}_{\boldsymbol{\epsilon}}\!\left[\exp\!\big(-J(\mathcal{T}(\mathbf{x}_t,\boldsymbol{\epsilon}))\big)\right].
\end{equation}
A practical Monte Carlo estimator is obtained by drawing $\boldsymbol{\epsilon}^{(i)}$ and setting
$\mathbf{x}_1^{(i)}=\mathcal{T}(\mathbf{x}_t,\boldsymbol{\epsilon}^{(i)})$, which gives
\begin{equation}
\mathbf{g}_{t}^{\mathrm{LGD}}
\approx
\frac{b_{t}\sigma_{t}^{2}}{\alpha_{t}}\,
\nabla_{\mathbf{x}_{t}}
\log\!\left(\frac{1}{n}\sum_{i=1}^{n}\exp\!\big(-J(\mathbf{x}_{1}^{(i)})\big)\right).
\end{equation}
In particular, if $q_{1|t}(\mathbf{x}_{1}\mid \mathbf{x}_t)=\delta(\mathbf{x}_{1}-\boldsymbol{\mu}(\mathbf{x}_{t}))$, then
\begin{equation}
\mathbf{g}_{t}^{\mathrm{LGD\text{-}DPS}}
=-\frac{b_{t}\sigma_{t}^{2}}{\alpha_{t}}\,
\nabla_{\mathbf{x}_{t}} J\!\left(\boldsymbol{\mu}(\mathbf{x}_{t})\right),
\end{equation}
and if $q_{1|t}(\mathbf{x}_{1}\mid \mathbf{x}_t)=\mathcal{N}(\boldsymbol{\mu}(\mathbf{x}_{t}),\boldsymbol{\Sigma}_{t})$ with
$\boldsymbol{\Sigma}_t=\mathbf{L}_t\mathbf{L}_t^{\top}$, then using
$\mathbf{x}_{1}^{(i)}=\boldsymbol{\mu}(\mathbf{x}_{t})+\mathbf{L}_{t}\boldsymbol{\epsilon}^{(i)}$ yields
\begin{equation}
\mathbf{g}_{t}^{\mathrm{LGD\text{-}MC}}
=\frac{b_{t}\sigma_{t}^{2}}{\alpha_{t}}\,
\nabla_{\mathbf{x}_{t}}
\log\!\left(\frac{1}{n}\sum_{i=1}^{n}\exp\!\big(-J(\mathbf{x}_{1}^{(i)})\big)\right).
\end{equation}

\paragraph{SIM (SIM-MC).}
Start from the SIM estimator in score form, Eq.~\eqref{eq:gt_score_based}:
\begin{equation}
\mathbf{g}_{t}^{\mathrm{SIM}}
=\frac{b_{t}\sigma_{t}^{2}}{\alpha_{t}}\,
\mathbb{E}_{q_{1|t}^{(J)}(\mathbf{x}_{1}\mid \mathbf{x}_t)}
\!\left[\nabla_{\mathbf{x}_{t}}\log p_{1|t}(\mathbf{x}_{1}\mid \mathbf{x}_t)\right],
\label{eq:sim_start_app}
\end{equation}
where $q_{1|t}^{(J)}(\mathbf{x}_{1}\mid \mathbf{x}_t)\propto q_{1|t}(\mathbf{x}_{1}\mid \mathbf{x}_t)\exp(-J(\mathbf{x}_{1}))$.
We will use the identity
\begin{equation}
\mathbb{E}_{q_{1|t}(\mathbf{x}_{1}\mid \mathbf{x}_t)}
\!\left[\nabla_{\mathbf{x}_{t}}\log p_{1|t}(\mathbf{x}_{1}\mid \mathbf{x}_t)\right]=\mathbf{0}.
\label{eq:score_zero_mean_app}
\end{equation}
To verify \eqref{eq:score_zero_mean_app}, apply Bayes' rule:
\begin{equation}
\nabla_{\mathbf{x}_{t}}\log p_{1|t}(\mathbf{x}_{1}\mid \mathbf{x}_t)
=
\nabla_{\mathbf{x}_{t}}\log p_{t|1}(\mathbf{x}_{t}\mid \mathbf{x}_1)
-\nabla_{\mathbf{x}_{t}}\log p_{t}(\mathbf{x}_{t}).
\end{equation}
For the forward noising kernel $p_{t|1}(\mathbf{x}_{t}\mid \mathbf{x}_1)=\mathcal{N}(\alpha_t\mathbf{x}_1,\sigma_t^{2}\mathbf{I})$, we have
\begin{equation}
\nabla_{\mathbf{x}_{t}}\log p_{t|1}(\mathbf{x}_{t}\mid \mathbf{x}_1)
=\frac{1}{\sigma_t^2}\big(\alpha_t\mathbf{x}_1-\mathbf{x}_t\big).
\end{equation}
Taking expectation under $q_{1|t}(\mathbf{x}_1\mid \mathbf{x}_t)$ yields
\begin{equation}
\mathbb{E}_{q_{1|t}}\!\left[\nabla_{\mathbf{x}_{t}}\log p_{1|t}(\mathbf{x}_{1}\mid \mathbf{x}_t)\right]
=
\frac{1}{\sigma_t^2}\big(\alpha_t\,\mathbb{E}_{q_{1|t}}[\mathbf{x}_1]-\mathbf{x}_t\big)
-\nabla_{\mathbf{x}_{t}}\log p_t(\mathbf{x}_t).
\label{eq:score_mean_app}
\end{equation}
We choose the SIM proposal to match the conditional mean,
$q_{1|t}(\mathbf{x}_{1}\mid \mathbf{x}_t)=\mathcal{N}\!\big(\boldsymbol{\mu}(\mathbf{x}_t),\boldsymbol{\Sigma}_t(\mathbf{x}_t)\big)$, so that
$\mathbb{E}_{q_{1|t}}[\mathbf{x}_1]=\boldsymbol{\mu}(\mathbf{x}_t)$.
By the (Tweedie) identity implied by the definition of $\boldsymbol{\mu}(\mathbf{x}_t)$,
\begin{equation}
\nabla_{\mathbf{x}_t}\log p_t(\mathbf{x}_t)=\frac{1}{\sigma_t^2}\big(\alpha_t\boldsymbol{\mu}(\mathbf{x}_t)-\mathbf{x}_t\big),
\end{equation}
and substituting into \eqref{eq:score_mean_app} gives \eqref{eq:score_zero_mean_app}.

Using \eqref{eq:score_zero_mean_app}, we can subtract zero from \eqref{eq:sim_start_app}:
\begin{equation}
\mathbf{g}_{t}^{\mathrm{SIM}}
=\frac{b_{t}\sigma_{t}^{2}}{\alpha_{t}}
\left(
\mathbb{E}_{q_{1|t}^{(J)}}\!\left[\nabla_{\mathbf{x}_{t}}\log p_{1|t}(\mathbf{x}_{1}\mid \mathbf{x}_t)\right]
-
\mathbb{E}_{q_{1|t}}\!\left[\nabla_{\mathbf{x}_{t}}\log p_{1|t}(\mathbf{x}_{1}\mid \mathbf{x}_t)\right]
\right).
\label{eq:sim_diff_app}
\end{equation}
Plugging in the Bayes decomposition and canceling the $\nabla_{\mathbf{x}_t}\log p_t(\mathbf{x}_t)$ term yields
\begin{align}
\mathbf{g}_{t}^{\mathrm{SIM}}
&=\frac{b_t\sigma_t^2}{\alpha_t}
\left(
\mathbb{E}_{q_{1|t}^{(J)}}\!\left[\nabla_{\mathbf{x}_t}\log p_{t|1}(\mathbf{x}_t\mid \mathbf{x}_1)\right]
-
\mathbb{E}_{q_{1|t}}\!\left[\nabla_{\mathbf{x}_t}\log p_{t|1}(\mathbf{x}_t\mid \mathbf{x}_1)\right]
\right)\nonumber\\
&=\frac{b_t\sigma_t^2}{\alpha_t}\cdot\frac{\alpha_t}{\sigma_t^2}
\left(\mathbb{E}_{q_{1|t}^{(J)}}[\mathbf{x}_1]-\mathbb{E}_{q_{1|t}}[\mathbf{x}_1]\right)
= b_t\left(\mathbb{E}_{q_{1|t}^{(J)}}[\mathbf{x}_1]-\boldsymbol{\mu}(\mathbf{x}_t)\right).
\label{eq:sim_mean_shift_app}
\end{align}

Finally, write $\mathbf{x}_1=\boldsymbol{\mu}(\mathbf{x}_t)+\mathbf{z}_t$ with
$q_t(\mathbf{z}_t)=\mathcal{N}(\mathbf{0},\boldsymbol{\Sigma}_t(\mathbf{x}_t))$ and define
$J_t(\mathbf{z}_t)\triangleq J(\boldsymbol{\mu}(\mathbf{x}_t)+\mathbf{z}_t)$.
Then \eqref{eq:sim_mean_shift_app} becomes
\begin{equation}
\mathbf{g}_t^{\mathrm{SIM}}
= b_t\,\mathbb{E}_{q_t^{(J_t)}}[\mathbf{z}_t]
= b_t\,\mathbb{E}_{q_t}\!\left[w^{(J)}\!\left(\boldsymbol{\mu}(\mathbf{x}_t)+\mathbf{z}_t\right)\mathbf{z}_t\right],
\end{equation}
where $q_t^{(J_t)}(\mathbf{z}_t)\propto q_t(\mathbf{z}_t)\exp(-J_t(\mathbf{z}_t))$ and $w^{(J)}$ is the corresponding normalized weight.
Sampling $\mathbf{z}_t^{(i)}=\mathbf{L}_t(\mathbf{x}_t)\boldsymbol{\epsilon}^{(i)}$ and setting
$\mathbf{x}_1^{(i)}=\boldsymbol{\mu}(\mathbf{x}_t)+\mathbf{z}_t^{(i)}$, the Monte Carlo approximation gives
\begin{equation}
\mathbf{g}_{t}^{\mathrm{SIM\text{-}MC}}
= b_t\sum_{i=1}^{n} w_{\{\mathbf{x}_1^{(i)}\}}^{(J)}\!\left(\mathbf{x}_1^{(i)}\right)\,\mathbf{z}_t^{(i)},
\end{equation}
where $w_{\{\mathbf{x}_1^{(i)}\}}^{(J)}$ is the normalized exponential weight over the $n$ samples (using the normalized exponential weights in Eq.~\eqref{eq:si_minibatch_statement_weights}).
\end{proof}

\subsubsection{Proof of Theorem~\ref{thm:lgd_sim_gap}}
\begin{proof}
Fix $\mathbf{x}_t$ throughout and write $\boldsymbol{\mu}=\boldsymbol{\mu}(\mathbf{x}_t)$.
Starting from the score-based definitions under a proposal $q_{1|t}$,
\begin{equation}
\mathbf{g}_t^{\mathrm{LGD}}
=\frac{b_t\sigma_t^2}{\alpha_t}\,
\mathbb{E}_{q_{1|t}^{(J)}(\mathbf{x}_1\mid \mathbf{x}_t)}
\!\left[\nabla_{\mathbf{x}_t}\log q_{1|t}(\mathbf{x}_1\mid \mathbf{x}_t)\right],
\qquad
\mathbf{g}_t^{\mathrm{SIM}}
=\frac{b_t\sigma_t^2}{\alpha_t}\,
\mathbb{E}_{q_{1|t}^{(J)}(\mathbf{x}_1\mid \mathbf{x}_t)}
\!\left[\nabla_{\mathbf{x}_t}\log p_{1|t}(\mathbf{x}_1\mid \mathbf{x}_t)\right],
\end{equation}
where $q_{1|t}^{(J)}(\mathbf{x}_1\mid \mathbf{x}_t)\propto q_{1|t}(\mathbf{x}_1\mid \mathbf{x}_t)\exp(-J(\mathbf{x}_1))$.
Assume $q_{1|t}(\mathbf{x}_1\mid \mathbf{x}_t)=\mathcal{N}(\boldsymbol{\mu},\boldsymbol{\Sigma}_t)$ and let
$\mathbf{z}_t\triangleq \mathbf{x}_1-\boldsymbol{\mu}$, so that $\mathbf{z}_t\sim q_t(\mathbf{z}_t)=\mathcal{N}(\mathbf{0},\boldsymbol{\Sigma}_t)$.
Let $\mathbf{B}_t\triangleq \nabla_{\mathbf{x}_t}\boldsymbol{\mu}(\mathbf{x}_t)$.

Since $q_{1|t}$ depends on $\mathbf{x}_t$ only through $\boldsymbol{\mu}(\mathbf{x}_t)$, we have
\[
\log q_{1|t}(\mathbf{x}_1\mid \mathbf{x}_t)
= -\tfrac12(\mathbf{x}_1-\boldsymbol{\mu})^{\top}\boldsymbol{\Sigma}_t^{-1}(\mathbf{x}_1-\boldsymbol{\mu})+\mathrm{const},
\]
and differentiating w.r.t.\ $\mathbf{x}_t$ yields
\begin{equation}
\nabla_{\mathbf{x}_t}\log q_{1|t}(\mathbf{x}_1\mid \mathbf{x}_t)
=
\mathbf{B}_t^{\top}\boldsymbol{\Sigma}_t^{-1}(\mathbf{x}_1-\boldsymbol{\mu})
=
\mathbf{B}_t^{\top}\boldsymbol{\Sigma}_t^{-1}\mathbf{z}_t.
\label{eq:score_q_app_clean}
\end{equation}
For the true conditional, Bayes' rule gives
$\log p_{1|t}(\mathbf{x}_1\mid \mathbf{x}_t)=\log p_{t|1}(\mathbf{x}_t\mid \mathbf{x}_1)-\log p_t(\mathbf{x}_t)$.
With $p_{t|1}(\mathbf{x}_t\mid \mathbf{x}_1)=\mathcal{N}(\alpha_t\mathbf{x}_1,\sigma_t^2\mathbf{I})$,
\[
\nabla_{\mathbf{x}_t}\log p_{t|1}(\mathbf{x}_t\mid \mathbf{x}_1)
=\sigma_t^{-2}(\alpha_t\mathbf{x}_1-\mathbf{x}_t),
\]
and the Tweedie identity implies
$\nabla_{\mathbf{x}_t}\log p_t(\mathbf{x}_t)=\sigma_t^{-2}(\alpha_t\boldsymbol{\mu}-\mathbf{x}_t)$.
Therefore,
\begin{equation}
\nabla_{\mathbf{x}_t}\log p_{1|t}(\mathbf{x}_1\mid \mathbf{x}_t)
=\frac{\alpha_t}{\sigma_t^2}\,(\mathbf{x}_1-\boldsymbol{\mu})
=\frac{\alpha_t}{\sigma_t^2}\,\mathbf{z}_t.
\label{eq:score_p_app_clean}
\end{equation}

Subtracting the two guidance terms and using \eqref{eq:score_q_app_clean}--\eqref{eq:score_p_app_clean} gives
\begin{equation}
\mathbf{g}_t^{\mathrm{LGD}}-\mathbf{g}_t^{\mathrm{SIM}}
=
\frac{b_t\sigma_t^2}{\alpha_t}\,
\mathbb{E}_{q_{1|t}^{(J)}}
\!\left[
\Big(\mathbf{B}_t^{\top}\boldsymbol{\Sigma}_t^{-1}-\frac{\alpha_t}{\sigma_t^2}\mathbf{I}\Big)\mathbf{z}_t
\right].
\label{eq:gap_raw_app_clean}
\end{equation}
Noting that $\mathbf{x}_1=\boldsymbol{\mu}+\mathbf{z}_t$, the tilted law $q_{1|t}^{(J)}$ induces a tilted law on $\mathbf{z}_t$:
$q_t^{(J_t)}(\mathbf{z}_t)\propto q_t(\mathbf{z}_t)\exp(-J_t(\mathbf{z}_t))$ with $J_t(\mathbf{z}_t)=J(\boldsymbol{\mu}+\mathbf{z}_t)$.
Hence the expectation in \eqref{eq:gap_raw_app_clean} can be written under $q_t^{(J_t)}$.

Define the covariance mismatch
\[
\boldsymbol{\Delta}_t(\mathbf{x}_t)\triangleq \boldsymbol{\Sigma}_t-\frac{\sigma_t^2}{\alpha_t}\mathbf{B}_t,
\qquad
e_t(\mathbf{x}_t)\triangleq \|\boldsymbol{\Delta}_t(\mathbf{x}_t)\|_2^2.
\]
Then $\mathbf{B}_t-\tfrac{\alpha_t}{\sigma_t^2}\boldsymbol{\Sigma}_t= -\tfrac{\alpha_t}{\sigma_t^2}\boldsymbol{\Delta}_t(\mathbf{x}_t)$, and since $\boldsymbol{\Sigma}_t$ is symmetric,
\[
\mathbf{B}_t^{\top}\boldsymbol{\Sigma}_t^{-1}-\frac{\alpha_t}{\sigma_t^2}\mathbf{I}
=
\big(\boldsymbol{\Sigma}_t^{-1}(\mathbf{B}_t-\tfrac{\alpha_t}{\sigma_t^2}\boldsymbol{\Sigma}_t)\big)^{\top}
=
-\frac{\alpha_t}{\sigma_t^2}\big(\boldsymbol{\Sigma}_t^{-1}\boldsymbol{\Delta}_t(\mathbf{x}_t)\big)^{\top}.
\]
Substituting into \eqref{eq:gap_raw_app_clean} yields
\begin{equation}
\mathbf{g}_t^{\mathrm{LGD}}-\mathbf{g}_t^{\mathrm{SIM}}
=
-b_t\,
\mathbb{E}_{q_t^{(J_t)}}
\!\left[
\big(\boldsymbol{\Sigma}_t^{-1}\boldsymbol{\Delta}_t(\mathbf{x}_t)\big)^{\top}\mathbf{z}_t
\right].
\label{eq:gap_final_app_clean}
\end{equation}

Finally, by Jensen's inequality,
\[
\|\mathbf{g}_t^{\mathrm{LGD}}-\mathbf{g}_t^{\mathrm{SIM}}\|_2^2
\le
b_t^2\,
\mathbb{E}_{q_t^{(J_t)}}\!\left[
\left\|
\big(\boldsymbol{\Sigma}_t^{-1}\boldsymbol{\Delta}_t(\mathbf{x}_t)\big)^{\top}\mathbf{z}_t
\right\|_2^2
\right].
\]
Using $\|A\mathbf{z}\|_2\le \|A\|_2\|\mathbf{z}\|_2$ and submultiplicativity,
\[
\left\|
\big(\boldsymbol{\Sigma}_t^{-1}\boldsymbol{\Delta}_t(\mathbf{x}_t)\big)^{\top}\mathbf{z}_t
\right\|_2^2
\le
\|\boldsymbol{\Sigma}_t^{-1}\|_2^2\,\|\boldsymbol{\Delta}_t(\mathbf{x}_t)\|_2^2\,\|\mathbf{z}_t\|_2^2.
\]
Taking expectation under $q_t^{(J_t)}$ and substituting $e_t(\mathbf{x}_t)=\|\boldsymbol{\Delta}_t(\mathbf{x}_t)\|_2^2$ gives
\[
\|\mathbf{g}_t^{\mathrm{LGD}}-\mathbf{g}_t^{\mathrm{SIM}}\|_2^2
\le
b_t^2\,\|\boldsymbol{\Sigma}_t^{-1}\|_2^2\,
\mathbb{E}_{q_t^{(J_t)}}\!\left[\|\mathbf{z}_t\|_2^2\right]\,
e_t(\mathbf{x}_t),
\]
which concludes the proof.
\end{proof}

\subsubsection{Proof of Theorem~\ref{thm:updataB}}
\begin{proof}
Recall that
\begin{equation}
\nabla_{\mathbf{x}_t}\,\mathbb{E}_{p_{1|t}(\mathbf{x}_1\mid \mathbf{x}_t)}[\mathbf{x}_1]
=
-\frac{a_t}{b_t}\mathbf{I}
+\frac{1}{b_t}\nabla_{\mathbf{x}_t}\mathbf{v}_t(\mathbf{x}_t).
\label{eq:mu_jac_app}
\end{equation}
At iteration $k$, let $\mathbf{J}_k\approx \nabla_{\mathbf{x}_k}\mathbf{v}_{t_k}(\mathbf{x}_k)$ and define
\begin{equation}
\mathbf{B}_k \triangleq -\frac{a_{t_k}}{b_{t_k}}\mathbf{I}+\frac{1}{b_{t_k}}\mathbf{J}_k,
\qquad\Longleftrightarrow\qquad
\mathbf{J}_k=b_{t_k}\mathbf{B}_k+a_{t_k}\mathbf{I}.
\label{eq:BJ_relation_app}
\end{equation}

We approximate the Jacobian $\mathbf{J}_{k+1}$ using the secant condition induced by the observed pair
$(\mathbf{x}_k,\mathbf{v}_k)$, $(\mathbf{x}_{k+1},\mathbf{v}_{k+1})$, namely
\begin{equation}
\mathbf{J}_{k+1}\mathbf{s}_k=\mathbf{r}_k,
\qquad
\mathbf{s}_k=\mathbf{x}_{k+1}-\mathbf{x}_k,\quad
\mathbf{r}_k=\mathbf{v}_{k+1}-\mathbf{v}_k.
\label{eq:secant_J_app}
\end{equation}
Using \eqref{eq:BJ_relation_app} at time $t_k$, introduce the scaled matrix
\begin{equation}
\widetilde{\mathbf{B}}_{k+1}
\triangleq
-\frac{a_{t_k}}{b_{t_k}}\mathbf{I}+\frac{1}{b_{t_k}}\mathbf{J}_{k+1},
\label{eq:Btilde_def_app}
\end{equation}
so that $\mathbf{J}_{k+1}=b_{t_k}\widetilde{\mathbf{B}}_{k+1}+a_{t_k}\mathbf{I}$.
Substituting this into \eqref{eq:secant_J_app} yields the equivalent secant equation
\begin{equation}
\widetilde{\mathbf{B}}_{k+1}\mathbf{s}_k
=
-\frac{a_{t_k}}{b_{t_k}}\mathbf{s}_k+\frac{1}{b_{t_k}}\mathbf{r}_k
\triangleq \mathbf{y}_k,
\label{eq:secant_Btilde_app}
\end{equation}
which matches the definition of $\mathbf{y}_k$ in the theorem.

The secant constraint \eqref{eq:secant_Btilde_app} does not identify a unique matrix. To obtain a stable and covariance-compatible estimate, we impose (i) symmetry, and (ii) a least-change principle:
\begin{equation}
\widetilde{\mathbf{B}}_{k+1}
=
\arg\min_{\mathbf{B}=\mathbf{B}^{\top}}
\ \|\mathbf{B}-\mathbf{B}_k\|
\quad\text{s.t.}\quad
\mathbf{B}\mathbf{s}_k=\mathbf{y}_k.
\label{eq:least_change_app}
\end{equation}
With the standard quasi-Newton choice of norm that yields a symmetric positive-definite update (the same setting used to derive the DFP/BFGS family; see~\cite{nocedal2006numerical}), the unique minimizer of \eqref{eq:least_change_app} has the closed form
\begin{equation}
\widetilde{\mathbf{B}}_{k+1}
=
\mathbf{V}_k^{\top}\mathbf{B}_k\mathbf{V}_k
+\rho_k\,\mathbf{y}_k\mathbf{y}_k^{\top},
\qquad
\rho_k=\frac{1}{\mathbf{y}_k^{\top}\mathbf{s}_k},\quad
\mathbf{V}_k=\mathbf{I}-\rho_k\mathbf{s}_k\mathbf{y}_k^{\top}.
\label{eq:Btilde_update_app}
\end{equation}
We briefly verify feasibility: symmetry is immediate from $\mathbf{B}_k=\mathbf{B}_k^{\top}$ and the rank-one term.
Moreover, since $\mathbf{V}_k\mathbf{s}_k=\mathbf{s}_k-\rho_k\mathbf{s}_k(\mathbf{y}_k^{\top}\mathbf{s}_k)=\mathbf{0}$,
\begin{equation}
\widetilde{\mathbf{B}}_{k+1}\mathbf{s}_k
=
\mathbf{V}_k^{\top}\mathbf{B}_k(\mathbf{V}_k\mathbf{s}_k)
+\rho_k\mathbf{y}_k(\mathbf{y}_k^{\top}\mathbf{s}_k)
=
\mathbf{y}_k,
\end{equation}
so \eqref{eq:secant_Btilde_app} holds.

Finally, the matrix $\mathbf{B}_{k+1}$ in the theorem is defined at time $t_{k+1}$:
\begin{equation}
\mathbf{B}_{k+1}
=
-\frac{a_{t_{k+1}}}{b_{t_{k+1}}}\mathbf{I}
+\frac{1}{b_{t_{k+1}}}\mathbf{J}_{k+1}.
\label{eq:Bkp1_def_app}
\end{equation}
Combining \eqref{eq:Bkp1_def_app} with \eqref{eq:Btilde_def_app} gives the affine relation
\begin{equation}
\mathbf{B}_{k+1}
=
\frac{b_{t_k}}{b_{t_{k+1}}}\,\widetilde{\mathbf{B}}_{k+1}
+\frac{a_{t_k}-a_{t_{k+1}}}{b_{t_{k+1}}}\mathbf{I}
\triangleq
u_k\,\widetilde{\mathbf{B}}_{k+1}+w_k\mathbf{I},
\end{equation}
where $u_k=b_{t_k}/b_{t_{k+1}}$ and $w_k=(a_{t_k}-a_{t_{k+1}})/b_{t_{k+1}}$.
Substituting \eqref{eq:Btilde_update_app} completes the claimed recursion:
\begin{equation}
\mathbf{B}_{k+1}
=
u_k\left(\mathbf{V}_k^{\top}\mathbf{B}_k\mathbf{V}_k+\rho_k\,\mathbf{y}_k\mathbf{y}_k^{\top}\right)
+w_k\mathbf{I}.
\end{equation}
\end{proof}

\subsubsection{Proof of Theorem~\ref{thm:Bk_compact_form}}
\begin{proof}
We prove by induction that the iterates produced by
\begin{equation}
\widetilde{\mathbf{B}}_{k+1}
=\mathbf{V}_k^{\top}\mathbf{B}_k\mathbf{V}_k+\rho_k\,\mathbf{y}_k\mathbf{y}_k^{\top},
\qquad
\mathbf{B}_{k+1}=u_k\widetilde{\mathbf{B}}_{k+1}+w_k\mathbf{I},
\label{eq:B_update_app}
\end{equation}
admit the claimed compact representation. Throughout, $\rho_k=(\mathbf{y}_k^{\top}\mathbf{s}_k)^{-1}$ and
$\mathbf{V}_k=\mathbf{I}-\rho_k\mathbf{s}_k\mathbf{y}_k^{\top}$.

\paragraph{Base case.}
Initialize $\mathbf{B}_0=\gamma_0\mathbf{I}$, i.e., $\mathbf{U}_0=[\,]$ and $\boldsymbol{\Gamma}_0=[\,]$.
For $k=0$, using $\mathbf{V}_0=\mathbf{I}-\rho_0\mathbf{s}_0\mathbf{y}_0^{\top}$ and expanding
\(
\mathbf{V}_0^{\top}\mathbf{B}_0\mathbf{V}_0
=\gamma_0\mathbf{V}_0^{\top}\mathbf{V}_0
\),
we obtain
\begin{align}
\widetilde{\mathbf{B}}_1
&=\gamma_0\mathbf{I}
-\gamma_0\rho_0(\mathbf{y}_0\mathbf{s}_0^{\top}+\mathbf{s}_0\mathbf{y}_0^{\top})
+\gamma_0\rho_0^2(\mathbf{s}_0^{\top}\mathbf{s}_0)\mathbf{y}_0\mathbf{y}_0^{\top}
+\rho_0\mathbf{y}_0\mathbf{y}_0^{\top} \nonumber\\
&=\gamma_0\mathbf{I}
+\begin{bmatrix}\mathbf{s}_0 & \mathbf{y}_0\end{bmatrix}
\begin{bmatrix}
0 & -\gamma_0\rho_0\\
-\gamma_0\rho_0 & \rho_0+\rho_0^2\gamma_0(\mathbf{s}_0^{\top}\mathbf{s}_0)
\end{bmatrix}
\begin{bmatrix}\mathbf{s}_0^{\top}\\ \mathbf{y}_0^{\top}\end{bmatrix}.
\label{eq:base_case_Btilde}
\end{align}
Thus $\widetilde{\mathbf{B}}_1=\gamma_0\mathbf{I}+\mathbf{U}_1\widetilde{\boldsymbol{\Gamma}}_1\mathbf{U}_1^{\top}$ holds with
$\mathbf{U}_1=[\mathbf{s}_0,\mathbf{y}_0]$ and $\widetilde{\boldsymbol{\Gamma}}_1$ given by the $2\times 2$ matrix in
\eqref{eq:base_case_Btilde}. Applying \eqref{eq:B_update_app} yields
\(
\mathbf{B}_1=(u_0\gamma_0+w_0)\mathbf{I}+\mathbf{U}_1(u_0\widetilde{\boldsymbol{\Gamma}}_1)\mathbf{U}_1^{\top}
\),
so the claim holds for $k=1$ with $\gamma_1=u_0\gamma_0+w_0$ and $\boldsymbol{\Gamma}_1=u_0\widetilde{\boldsymbol{\Gamma}}_1$.
(If one uses a different column ordering for $\mathbf{U}_1$, the corresponding $\boldsymbol{\Gamma}_1$ is obtained by the same permutation.)

\paragraph{Inductive step.}
Assume that for some $k\ge 1$,
\begin{equation}
\mathbf{B}_k=\gamma_k\mathbf{I}+\mathbf{U}_k\boldsymbol{\Gamma}_k\mathbf{U}_k^{\top},
\label{eq:ind_hyp}
\end{equation}
where $\mathbf{U}_k\in\mathbb{R}^{d\times m}$ collects past secant vectors and $\boldsymbol{\Gamma}_k\in\mathbb{R}^{m\times m}$.
Define
\begin{equation}
\mathbf{p}_k \triangleq \boldsymbol{\Gamma}_k\mathbf{U}_k^{\top}\mathbf{s}_k,
\qquad
\tau_k \triangleq \mathbf{s}_k^{\top}\mathbf{U}_k\boldsymbol{\Gamma}_k\mathbf{U}_k^{\top}\mathbf{s}_k,
\qquad
\delta_k \triangleq \tau_k+\gamma_k\,\mathbf{s}_k^{\top}\mathbf{s}_k.
\label{eq:pk_tauk_deltak}
\end{equation}
Then
\begin{equation}
\mathbf{B}_k\mathbf{s}_k=\gamma_k\mathbf{s}_k+\mathbf{U}_k\mathbf{p}_k,
\qquad
\mathbf{s}_k^{\top}\mathbf{B}_k\mathbf{s}_k=\delta_k.
\label{eq:Bs_and_quad}
\end{equation}

Using $\mathbf{V}_k=\mathbf{I}-\rho_k\mathbf{s}_k\mathbf{y}_k^{\top}$, we expand
\begin{align}
\mathbf{V}_k^{\top}\mathbf{B}_k\mathbf{V}_k
&=(\mathbf{I}-\rho_k\mathbf{y}_k\mathbf{s}_k^{\top})\mathbf{B}_k(\mathbf{I}-\rho_k\mathbf{s}_k\mathbf{y}_k^{\top}) \nonumber\\
&=\mathbf{B}_k
-\rho_k\,\mathbf{B}_k\mathbf{s}_k\mathbf{y}_k^{\top}
-\rho_k\,\mathbf{y}_k\mathbf{s}_k^{\top}\mathbf{B}_k
+\rho_k^2\,\mathbf{y}_k(\mathbf{s}_k^{\top}\mathbf{B}_k\mathbf{s}_k)\mathbf{y}_k^{\top}.
\label{eq:expand_VBV}
\end{align}
Substituting \eqref{eq:Bs_and_quad} into \eqref{eq:expand_VBV} gives
\begin{align}
\mathbf{V}_k^{\top}\mathbf{B}_k\mathbf{V}_k
&=\mathbf{B}_k
-\rho_k(\gamma_k\mathbf{s}_k+\mathbf{U}_k\mathbf{p}_k)\mathbf{y}_k^{\top}
-\rho_k\mathbf{y}_k(\gamma_k\mathbf{s}_k^{\top}+\mathbf{p}_k^{\top}\mathbf{U}_k^{\top})
+\rho_k^2\delta_k\,\mathbf{y}_k\mathbf{y}_k^{\top}.
\label{eq:VBV_sub}
\end{align}
Adding the rank-one term $\rho_k\mathbf{y}_k\mathbf{y}_k^{\top}$ yields
\begin{align}
\widetilde{\mathbf{B}}_{k+1}
&=\mathbf{V}_k^{\top}\mathbf{B}_k\mathbf{V}_k+\rho_k\mathbf{y}_k\mathbf{y}_k^{\top} \nonumber\\
&=\mathbf{B}_k
-\rho_k\mathbf{U}_k\mathbf{p}_k\mathbf{y}_k^{\top}
-\rho_k\mathbf{y}_k\mathbf{p}_k^{\top}\mathbf{U}_k^{\top}
-\gamma_k\rho_k(\mathbf{s}_k\mathbf{y}_k^{\top}+\mathbf{y}_k\mathbf{s}_k^{\top})
+\rho_k\big(1+\rho_k\delta_k\big)\mathbf{y}_k\mathbf{y}_k^{\top}.
\label{eq:Btilde_collect}
\end{align}
Now substitute the induction hypothesis \eqref{eq:ind_hyp} into \eqref{eq:Btilde_collect} and group terms.
Let
\begin{equation}
\mathbf{U}_{k+1}\triangleq
\begin{bmatrix}\mathbf{U}_k & \mathbf{s}_k & \mathbf{y}_k\end{bmatrix}.
\end{equation}
A direct multiplication shows that
\begin{equation}
\widetilde{\mathbf{B}}_{k+1}
=
\gamma_k\mathbf{I}
+\mathbf{U}_{k+1}\,\widetilde{\boldsymbol{\Gamma}}_{k+1}\,\mathbf{U}_{k+1}^{\top},
\label{eq:Btilde_compact}
\end{equation}
where
\begin{equation}
\widetilde{\boldsymbol{\Gamma}}_{k+1}
=
\begin{bmatrix}
\boldsymbol{\Gamma}_{k} & \mathbf{0} & -\rho_{k}\mathbf{p}_{k}\\
\mathbf{0}^{\top} & 0 & -\gamma_{k}\rho_{k} \\
-\rho_{k}\mathbf{p}_{k}^{\top} & -\gamma_{k}\rho_{k}
& \rho_k\big(1+\rho_k\delta_k\big)
\end{bmatrix},
\qquad
\delta_k=\tau_k+\gamma_k\mathbf{s}_k^{\top}\mathbf{s}_k.
\label{eq:Gammatilde_formula}
\end{equation}
Indeed, the top-left block reproduces $\mathbf{U}_k\boldsymbol{\Gamma}_k\mathbf{U}_k^{\top}$, the off-diagonal blocks
$-\rho_k\mathbf{p}_k$ reproduce the cross terms
$-\rho_k\mathbf{U}_k\mathbf{p}_k\mathbf{y}_k^{\top}-\rho_k\mathbf{y}_k\mathbf{p}_k^{\top}\mathbf{U}_k^{\top}$,
the entries $-\gamma_k\rho_k$ reproduce $-\gamma_k\rho_k(\mathbf{s}_k\mathbf{y}_k^{\top}+\mathbf{y}_k\mathbf{s}_k^{\top})$,
and the bottom-right entry gives the coefficient of $\mathbf{y}_k\mathbf{y}_k^{\top}$ in \eqref{eq:Btilde_collect}.

Finally, applying the affine update in \eqref{eq:B_update_app} to \eqref{eq:Btilde_compact} gives
\begin{align}
\mathbf{B}_{k+1}
&=u_k\widetilde{\mathbf{B}}_{k+1}+w_k\mathbf{I}
=(u_k\gamma_k+w_k)\mathbf{I}+\mathbf{U}_{k+1}\,(u_k\widetilde{\boldsymbol{\Gamma}}_{k+1})\,\mathbf{U}_{k+1}^{\top}.
\end{align}
Therefore the representation holds at $k+1$ with
\begin{equation}
\gamma_{k+1}=u_k\gamma_k+w_k,
\qquad
\boldsymbol{\Gamma}_{k+1}=u_k\,\widetilde{\boldsymbol{\Gamma}}_{k+1},
\end{equation}
and substituting \eqref{eq:Gammatilde_formula} yields exactly the stated recursion.
\end{proof}

\subsubsection{Proof of Theorem~\ref{thm:semi_numerical_sqrt}}
\begin{proof}
Let $\mathbf{B}=\gamma\mathbf{I}+\mathbf{U}\boldsymbol{\Gamma}\mathbf{U}^{\top}$ with $\gamma>0$.
Take the reduced QR factorization of $\mathbf{U}\in\mathbb{R}^{d\times m}$:
\begin{equation}
\mathbf{U}=\mathbf{Q}\mathbf{R},\qquad 
\mathbf{Q}\in\mathbb{R}^{d\times m},\ \mathbf{Q}^{\top}\mathbf{Q}=\mathbf{I}_m,\qquad 
\mathbf{R}\in\mathbb{R}^{m\times m}\ \text{upper triangular}.
\end{equation}
Extend $\mathbf{Q}$ to an orthonormal basis of $\mathbb{R}^d$ by choosing
$\mathbf{Q}_{\perp}\in\mathbb{R}^{d\times(d-m)}$ such that
\begin{equation}
\mathbf{P}\triangleq[\mathbf{Q},\mathbf{Q}_{\perp}]\ \text{is orthogonal},\qquad
\mathbf{P}^{\top}\mathbf{P}=\mathbf{I}_d,
\qquad
\mathbf{Q}\mathbf{Q}^{\top}+\mathbf{Q}_{\perp}\mathbf{Q}_{\perp}^{\top}=\mathbf{I}_d.
\label{eq:P_orth_app}
\end{equation}
Using $\mathbf{U}=\mathbf{Q}\mathbf{R}$, we rewrite
\begin{equation}
\mathbf{B}
=\gamma\mathbf{I}+\mathbf{Q}\mathbf{R}\boldsymbol{\Gamma}\mathbf{R}^{\top}\mathbf{Q}^{\top}.
\end{equation}
Conjugating by the orthogonal matrix $\mathbf{P}$ yields a block diagonal form:
\begin{align}
\mathbf{P}^{\top}\mathbf{B}\mathbf{P}
&=
\begin{bmatrix}
\mathbf{Q}^{\top}\\ \mathbf{Q}_{\perp}^{\top}
\end{bmatrix}
\left(\gamma\mathbf{I}+\mathbf{Q}\mathbf{R}\boldsymbol{\Gamma}\mathbf{R}^{\top}\mathbf{Q}^{\top}\right)
\begin{bmatrix}
\mathbf{Q} & \mathbf{Q}_{\perp}
\end{bmatrix} \nonumber\\
&=
\begin{bmatrix}
\gamma\mathbf{I}_m + \mathbf{R}\boldsymbol{\Gamma}\mathbf{R}^{\top} & \mathbf{0}\\
\mathbf{0} & \gamma\mathbf{I}_{d-m}
\end{bmatrix}.
\label{eq:block_diag_app}
\end{align}
Define the $m\times m$ matrix (consistent with Theorem~\ref{thm:semi_numerical_sqrt})
\begin{equation}
\mathbf{C}\triangleq \gamma\mathbf{I}_m+\mathbf{R}\boldsymbol{\Gamma}\mathbf{R}^{\top}.
\label{eq:C_def_app}
\end{equation}
Assuming $\mathbf{B}\succ \mathbf{0}$, we have $\mathbf{P}^{\top}\mathbf{B}\mathbf{P}\succ \mathbf{0}$ since $\mathbf{P}$ is orthogonal.
By \eqref{eq:block_diag_app}, this implies $\mathbf{C}\succ \mathbf{0}$, so $\mathbf{C}$ admits a (lower) Cholesky factorization
\begin{equation}
\mathbf{C}=\mathbf{L}_{\mathbf{C}}\mathbf{L}_{\mathbf{C}}^{\top}.
\end{equation}

We now construct $\mathbf{L}$ such that $\mathbf{B}=\mathbf{L}\mathbf{L}^{\top}$.
Consider the matrix
\begin{equation}
\widetilde{\mathbf{L}}
\triangleq
\begin{bmatrix}
\mathbf{L}_{\mathbf{C}} & \mathbf{0}\\
\mathbf{0} & \sqrt{\gamma}\mathbf{I}_{d-m}
\end{bmatrix}
\in\mathbb{R}^{d\times d}.
\label{eq:Ltilde_app}
\end{equation}
Then by direct multiplication,
\begin{equation}
\widetilde{\mathbf{L}}\widetilde{\mathbf{L}}^{\top}
=
\begin{bmatrix}
\mathbf{L}_{\mathbf{C}}\mathbf{L}_{\mathbf{C}}^{\top} & \mathbf{0}\\
\mathbf{0} & \gamma\mathbf{I}_{d-m}
\end{bmatrix}
=
\begin{bmatrix}
\mathbf{C} & \mathbf{0}\\
\mathbf{0} & \gamma\mathbf{I}_{d-m}
\end{bmatrix}.
\end{equation}
Combining with \eqref{eq:block_diag_app} and the fact that $\mathbf{P}$ is orthogonal, we obtain
\begin{equation}
\mathbf{B}
=\mathbf{P}
\begin{bmatrix}
\mathbf{C} & \mathbf{0}\\
\mathbf{0} & \gamma\mathbf{I}_{d-m}
\end{bmatrix}
\mathbf{P}^{\top}
=\big(\mathbf{P}\widetilde{\mathbf{L}}\mathbf{P}^{\top}\big)\big(\mathbf{P}\widetilde{\mathbf{L}}\mathbf{P}^{\top}\big)^{\top}.
\end{equation}
Therefore, one valid factor is
\begin{equation}
\mathbf{L}_\star=\mathbf{P}\widetilde{\mathbf{L}}\mathbf{P}^{\top}
=\mathbf{Q}\mathbf{L}_{\mathbf{C}}\mathbf{Q}^{\top}+\sqrt{\gamma}\,\mathbf{Q}_{\perp}\mathbf{Q}_{\perp}^{\top}.
\label{eq:Lstar_app}
\end{equation}

To eliminate the explicit dependence on $\mathbf{Q}_{\perp}$, we use the projector identity from \eqref{eq:P_orth_app}:
\begin{equation}
\mathbf{Q}_{\perp}\mathbf{Q}_{\perp}^{\top}=\mathbf{I}-\mathbf{Q}\mathbf{Q}^{\top}.
\end{equation}
Substituting into \eqref{eq:Lstar_app} gives
\begin{equation}
\mathbf{L}_\star
=
\mathbf{Q}\mathbf{L}_{\mathbf{C}}\mathbf{Q}^{\top}
+\sqrt{\gamma}\,(\mathbf{I}-\mathbf{Q}\mathbf{Q}^{\top})
=
\mathbf{Q}\big(\mathbf{L}_{\mathbf{C}}-\sqrt{\gamma}\mathbf{I}_m\big)\mathbf{Q}^{\top}
+\sqrt{\gamma}\mathbf{I}.
\end{equation}
Finally, since $\mathbf{B}=\mathbf{L}_\star\mathbf{L}_\star^{\top}$ by construction, taking $\mathbf{L}\equiv \mathbf{L}_\star$ proves the claim.
\end{proof}

\clearpage
\section{Supplementary analysis}
\label{sec:si_supp_analysis}

\subsection{Relationship between $\mathcal{J}_{\mathrm{SKL}}^{(\lambda)}$ and the MSE Loss}
\label{subsec:sKL-mse}

This section clarifies why the symmetric-KL objective used in Eq.~\eqref{eq:si_minibatch_statement_skl} is related to, but not equivalent to, ordinary mean-squared regression. The key distinction is what the trained predictor is asked to control. MSE controls pointwise cost values. In this work, the cost predictor is used downstream after exponentiation, where it reshapes a reference design distribution into a cost-biased distribution. The relevant question is therefore not only whether $J_\theta$ is numerically close to the simulated cost $J$, but whether it induces the same probability mass over low-cost designs.

For a reference distribution $p(\mathbf{x})$ and a fixed guidance scale $\lambda>0$, let
\begin{equation}
    p^{(J,\lambda)}(\mathbf{x})
    =
    p(\mathbf{x})w_J^\lambda(\mathbf{x}),
    \qquad
    w_J^\lambda(\mathbf{x})
    =
    \frac{\exp[-\lambda J(\mathbf{x})]}
    {\mathbb{E}_{p}[\exp[-\lambda J(\mathbf{x})]]},
    \label{eq:skl_true_weight_appendix}
\end{equation}
and define $p^{(J_\theta,\lambda)}$ and $w_\theta^\lambda$ analogously using the predicted cost $J_\theta$. The MSE objective is
\begin{equation}
    \mathcal{J}_{\mathrm{MSE}}^{(\lambda)}(\theta)
    =
    \mathbb{E}_{p}
    \left[
    \left(J_\theta(\mathbf{x})-J(\mathbf{x})\right)^2
    \right].
    \label{eq:mse_pop_appendix}
\end{equation}
This objective is appropriate when the downstream task is cost-value regression. It treats all samples according to their probability under $p$ and penalizes absolute offsets in the cost scale.

The SKL objective instead compares the two induced distributions:
\begin{equation}
    \mathcal{J}_{\mathrm{SKL}}^{(\lambda)}(\theta)
    =
    D_{\mathrm{KL}}\left(p^{(J,\lambda)}\,\|\,p^{(J_\theta,\lambda)}\right)
    +
    D_{\mathrm{KL}}\left(p^{(J_\theta,\lambda)}\,\|\,p^{(J,\lambda)}\right).
    \label{eq:skl_distribution_appendix}
\end{equation}
By Theorem~\ref{thm:si_skl_identity}, this can be written as
\begin{equation}
    \mathcal{J}_{\mathrm{SKL}}^{(\lambda)}(\theta)
    =
    \lambda\,\mathbb{E}_{p}
    \left[
    \left(w_J^\lambda(\mathbf{x})-w_\theta^\lambda(\mathbf{x})\right)
    \left(J_\theta(\mathbf{x})-J(\mathbf{x})\right)
    \right].
    \label{eq:skl_pop_appendix}
\end{equation}
Equation~\eqref{eq:skl_pop_appendix} gives the main intuition. The cost error matters most where it changes the normalized exponential weights that define the cost-biased distribution. Consequently, SKL is invariant to an additive cost offset: if $J_\theta(\mathbf{x})=J(\mathbf{x})+c$, then $p^{(J_\theta,\lambda)}=p^{(J,\lambda)}$ and $\mathcal{J}_{\mathrm{SKL}}^{(\lambda)}=0$, whereas MSE still penalizes the constant $c$. This shift invariance is natural for inverse design because optimization and sampling depend on relative cost-induced probabilities, not on the arbitrary zero of the cost scale.

\paragraph{Two useful theoretical relations.}
The following theorem formalizes the two regimes in which the relationship between SKL and MSE is most transparent. Locally, SKL behaves like a weighted and centred squared error. At large guidance scale, it is governed by whether the predicted and true costs identify the same low-cost region.

\begin{theorem}[Local and large-guidance behaviour of SKL]
\label{thm:skl_local_large_behavior}
Let $\delta_\theta(\mathbf{x})=J_\theta(\mathbf{x})-J(\mathbf{x})$ and
$\bar{\delta}_\theta=\mathbb{E}_{p^{(J,\lambda)}}[\delta_\theta(\mathbf{x})]$.
Write $\eta_\theta(\mathbf{x})=\delta_\theta(\mathbf{x})-\bar{\delta}_\theta$.

\textbf{(i) Local prediction-error regime.}
If $\lambda\eta_\theta$ is uniformly small and the third and fourth central moments of $\delta_\theta$ under $p^{(J,\lambda)}$ are finite, then
\begin{equation}
    \mathcal{J}_{\mathrm{SKL}}^{(\lambda)}(\theta)
    =
    \lambda^2\,
    \mathbb{V}_{p^{(J,\lambda)}}
    \left[
    \delta_\theta(\mathbf{x})
    \right]
    +
    O\!\left(
    \lambda^3\,
    \mathbb{E}_{p^{(J,\lambda)}}
    \left[
    \left|\delta_\theta(\mathbf{x})-\bar{\delta}_\theta\right|^3
    \right]
    \right).
    \label{eq:skl_local_mse_relation}
\end{equation}

\textbf{(ii) Large-guidance regime.}
Suppose $J$ has a unique global minimizer $\mathbf{x}^{*}=\arg\min_{\mathbf{x}}J(\mathbf{x})$ and $J_\theta$ has a unique global minimizer
$\mathbf{x}^{*}_{\theta}=\arg\min_{\mathbf{x}}J_\theta(\mathbf{x})$. Under standard regularity conditions for Laplace asymptotics,
\begin{equation}
    \mathcal{J}_{\mathrm{SKL}}^{(\lambda)}(\theta)
    =
    \lambda\Big(
    J_\theta(\mathbf{x}^{*})-J_\theta(\mathbf{x}^{*}_{\theta})
    +
    J(\mathbf{x}^{*}_{\theta})-J(\mathbf{x}^{*})
    \Big)
    +
    o(\lambda).
    \label{eq:skl_large_lambda}
\end{equation}
\end{theorem}

\begin{proof}
For part (i), write
\[
    P_\lambda=p^{(J,\lambda)},\qquad
    Q_{\theta,\lambda}=p^{(J_\theta,\lambda)},\qquad
    \eta_\theta(\mathbf{x})=\delta_\theta(\mathbf{x})-\bar{\delta}_\theta .
\]
Because $J_\theta=J+\delta_\theta$, the predicted cost-biased distribution can be written as a tilted version of the true one:
\begin{equation}
    \frac{dQ_{\theta,\lambda}}{dP_\lambda}(\mathbf{x})
    =
    \frac{\exp[-\lambda\delta_\theta(\mathbf{x})]}
    {\mathbb{E}_{p^{(J,\lambda)}}[\exp[-\lambda\delta_\theta(\mathbf{x})]]}.
    \label{eq:predicted_distribution_error_tilt}
\end{equation}
The additive constant $\bar{\delta}_\theta$ cancels between the numerator and denominator, so
\begin{equation}
    \frac{dQ_{\theta,\lambda}}{dP_\lambda}(\mathbf{x})
    =
    \frac{\exp[-\lambda\eta_\theta(\mathbf{x})]}
    {\mathbb{E}_{P_\lambda}[\exp[-\lambda\eta_\theta(\mathbf{x})]]},
    \qquad
    \mathbb{E}_{P_\lambda}[\eta_\theta(\mathbf{x})]=0 .
    \label{eq:tilt_centered_error}
\end{equation}
Expanding the denominator gives
\begin{equation}
    \mathbb{E}_{P_\lambda}[\exp[-\lambda\eta_\theta]]
    =
    1
    +
    \frac{\lambda^2}{2}\mathbb{E}_{P_\lambda}[\eta_\theta^2]
    -
    \frac{\lambda^3}{6}\mathbb{E}_{P_\lambda}[\eta_\theta^3]
    +
    O\!\left(\lambda^4\mathbb{E}_{P_\lambda}[\eta_\theta^4]\right),
    \label{eq:tilt_denominator_expansion}
\end{equation}
where the linear term vanishes because $\eta_\theta$ is centred under $P_\lambda$. Hence the likelihood ratio in Eq.~\eqref{eq:tilt_centered_error} satisfies
\begin{equation}
    \frac{dQ_{\theta,\lambda}}{dP_\lambda}(\mathbf{x})
    =
    1-\lambda\left(\delta_\theta(\mathbf{x})-\bar{\delta}_\theta\right)
    +
    \frac{\lambda^2}{2}
    \left[
    \eta_\theta^2(\mathbf{x})
    -
    \mathbb{E}_{P_\lambda}[\eta_\theta^2]
    \right]
    +
    O\!\left(\lambda^3|\eta_\theta(\mathbf{x})|^3\right).
    \label{eq:tilt_ratio_expansion}
\end{equation}
Using Eq.~\eqref{eq:skl_pop_appendix} and the identities
\[
    \mathbb{E}_{p}\!\left[w_J^\lambda \delta_\theta\right]
    =
    \mathbb{E}_{P_\lambda}[\delta_\theta],
    \qquad
    \mathbb{E}_{p}\!\left[w_\theta^\lambda \delta_\theta\right]
    =
    \mathbb{E}_{Q_{\theta,\lambda}}[\delta_\theta],
\]
we obtain
\begin{equation}
    \mathcal{J}_{\mathrm{SKL}}^{(\lambda)}(\theta)
    =
    \lambda
    \left(
    \mathbb{E}_{P_\lambda}[\delta_\theta]
    -
    \mathbb{E}_{Q_{\theta,\lambda}}[\delta_\theta]
    \right).
    \label{eq:skl_expectation_difference}
\end{equation}
Since $\delta_\theta=\bar{\delta}_\theta+\eta_\theta$ and
$\mathbb{E}_{P_\lambda}[\eta_\theta]=0$, Eq.~\eqref{eq:tilt_ratio_expansion} gives
\begin{equation}
    \mathbb{E}_{Q_{\theta,\lambda}}[\delta_\theta]
    =
    \bar{\delta}_\theta
    -
    \lambda\,\mathbb{E}_{P_\lambda}[\eta_\theta^2]
    +
    O\!\left(
    \lambda^2
    \mathbb{E}_{P_\lambda}[|\eta_\theta|^3]
    \right).
\end{equation}
Substituting this expression into Eq.~\eqref{eq:skl_expectation_difference} gives
\[
    \mathcal{J}_{\mathrm{SKL}}^{(\lambda)}(\theta)
    =
    \lambda^2\mathbb{E}_{P_\lambda}[\eta_\theta^2]
    +
    O\!\left(
    \lambda^3
    \mathbb{E}_{P_\lambda}[|\eta_\theta|^3]
    \right),
\]
which is Eq.~\eqref{eq:skl_local_mse_relation} because
$\mathbb{E}_{P_\lambda}[\eta_\theta^2]=
\mathbb{V}_{p^{(J,\lambda)}}[\delta_\theta(\mathbf{x})]$.

For part (ii), start again from Eq.~\eqref{eq:skl_expectation_difference}. Under the stated uniqueness and regularity assumptions, Laplace's method gives
\begin{equation}
    \mathbb{E}_{p^{(J,\lambda)}}[\delta_\theta(\mathbf{x})]
    =
    \delta_\theta(\mathbf{x}^{*})
    +
    o(1),
    \qquad
    \mathbb{E}_{p^{(J_\theta,\lambda)}}[\delta_\theta(\mathbf{x})]
    =
    \delta_\theta(\mathbf{x}^{*}_{\theta})
    +
    o(1),
    \label{eq:laplace_delta_limits}
\end{equation}
because $p^{(J,\lambda)}$ concentrates at the minimizer of $J$ and
$p^{(J_\theta,\lambda)}$ concentrates at the minimizer of $J_\theta$ as $\lambda\to\infty$. Therefore
\begin{equation}
    \mathcal{J}_{\mathrm{SKL}}^{(\lambda)}(\theta)
    =
    \lambda
    \left[
    \delta_\theta(\mathbf{x}^{*})
    -
    \delta_\theta(\mathbf{x}^{*}_{\theta})
    \right]
    +
    o(\lambda).
\end{equation}
Expanding $\delta_\theta=J_\theta-J$ yields
\[
    \delta_\theta(\mathbf{x}^{*})
    -
    \delta_\theta(\mathbf{x}^{*}_{\theta})
    =
    J_\theta(\mathbf{x}^{*})-J_\theta(\mathbf{x}^{*}_{\theta})
    +
    J(\mathbf{x}^{*}_{\theta})-J(\mathbf{x}^{*}),
\]
which proves Eq.~\eqref{eq:skl_large_lambda}.
\end{proof}

Theorem~\ref{thm:skl_local_large_behavior} shows why SKL can look like MSE near a well-calibrated predictor but behave differently in inverse-design use. In the local regime, SKL removes additive cost offsets and weights errors by the true cost-biased distribution. In the large-guidance regime, it mainly tests whether the predictor preserves the relative ordering of low-cost regions.

\paragraph{Why MSE does not globally control SKL.}
The local relationship above does not imply that MSE is sufficient for distributional training. Applying Cauchy--Schwarz to Eq.~\eqref{eq:skl_pop_appendix} gives
\begin{equation}
    \mathcal{J}_{\mathrm{SKL}}^{(\lambda)}(\theta)
    \le
    \lambda
    \sqrt{\mathcal{J}_{\mathrm{MSE}}^{(\lambda)}(\theta)}
    \sqrt{
    \mathbb{E}_{p}
    \left[
    \left(w_J^\lambda(\mathbf{x})-w_\theta^\lambda(\mathbf{x})\right)^2
    \right]
    }.
    \label{eq:skl_mse_exp_bound}
\end{equation}
The second factor is exactly the quantity that MSE does not control: the error in the normalized exponential weights. A predictor may have small average squared error while still creating a spurious low-cost basin that attracts most of the cost-biased probability mass.

\begin{theorem}[Small MSE can give large distributional error]
\label{thm:skl_mse_bound_counterexample}
For any fixed $\lambda>0$, there exists a sequence of predictors with $\mathcal{J}_{\mathrm{MSE}}^{(\lambda)}\to 0$ but $\mathcal{J}_{\mathrm{SKL}}^{(\lambda)}\to\infty$.
\end{theorem}

\begin{proof}
Consider $p$ uniform on $[0,1]$ and the true cost $J(x)=0$. Then
$p^{(J,\lambda)}=p$ for every $\lambda$. For $\varepsilon\in(0,1/2)$, let
$A_\varepsilon=[0,\varepsilon]$ and define
\begin{equation}
    J_\varepsilon(x)
    =
    -a_\varepsilon\,\mathbf{1}_{A_\varepsilon}(x),
    \qquad
    a_\varepsilon
    =
    \frac{2}{\lambda}\log\frac{1}{\varepsilon}.
\end{equation}
The MSE of this predictor is small because the error is supported only on a set of length $\varepsilon$:
\begin{equation}
    \mathcal{J}_{\mathrm{MSE}}^{(\lambda)}(\varepsilon)
    =
    \mathbb{E}_{p}\left[(J_\varepsilon(x)-J(x))^2\right]
    =
    a_\varepsilon^2\varepsilon
    =
    \frac{4}{\lambda^2}\varepsilon\log^2\frac{1}{\varepsilon}
    \to 0.
    \label{eq:counterexample_mse}
\end{equation}
The same localized error has a very different effect after exponentiation. The predicted cost-biased normalizer is
\begin{equation}
    Z_\varepsilon
    =
    \varepsilon e^{\lambda a_\varepsilon}+1-\varepsilon
    =
    \varepsilon^{-1}+1-\varepsilon.
    \label{eq:counterexample_normalizer}
\end{equation}
Thus the predicted normalized weight is piecewise constant:
\begin{equation}
    w_\varepsilon^\lambda
    =
    \frac{\exp[-\lambda J_\varepsilon(x)]}{Z_\varepsilon}
    =
    \begin{cases}
        \dfrac{\varepsilon^{-2}}{\varepsilon^{-1}+1-\varepsilon},
        & x\in A_\varepsilon,\\[8pt]
        \dfrac{1}{\varepsilon^{-1}+1-\varepsilon},
        & x\notin A_\varepsilon .
    \end{cases}
    \label{eq:counterexample_weights_piecewise}
\end{equation}
In particular, on $A_\varepsilon$,
\begin{equation}
    w_\varepsilon^\lambda
    =
    \frac{\varepsilon^{-2}}{\varepsilon^{-1}+1-\varepsilon}
    \sim
    \varepsilon^{-1},
\end{equation}
and the predicted cost-biased probability of $A_\varepsilon$ satisfies
\begin{equation}
    p^{(J_\varepsilon,\lambda)}(A_\varepsilon)
    =
    \varepsilon w_\varepsilon^\lambda
    =
    \frac{\varepsilon^{-1}}{\varepsilon^{-1}+1-\varepsilon}
    \to 1.
    \label{eq:counterexample_mass}
\end{equation}
Thus the predicted cost-biased distribution places asymptotically all of its probability on a set whose true probability is only $\varepsilon$.
To lower-bound the KL divergence, apply the data-processing inequality for KL to the binary partition
$\{A_\varepsilon,A_\varepsilon^{c}\}$. This gives
\begin{equation}
    D_{\mathrm{KL}}
    \left(
    p^{(J_\varepsilon,\lambda)}
    \,\|\,p^{(J,\lambda)}
    \right)
    \ge
    q_\varepsilon\log\frac{q_\varepsilon}{\varepsilon}
    +
    (1-q_\varepsilon)
    \log
    \frac{1-q_\varepsilon}{1-\varepsilon},
    \label{eq:counterexample_binary_kl}
\end{equation}
where
$q_\varepsilon=p^{(J_\varepsilon,\lambda)}(A_\varepsilon)$. By Eq.~\eqref{eq:counterexample_mass}, $q_\varepsilon\to 1$. Hence the first term in Eq.~\eqref{eq:counterexample_binary_kl} satisfies
\[
    q_\varepsilon\log\frac{q_\varepsilon}{\varepsilon}
    =
    q_\varepsilon\log q_\varepsilon
    +
    q_\varepsilon\log\frac{1}{\varepsilon}
    \to
    \infty,
\]
while the second term is bounded below and tends to zero because $1-q_\varepsilon\to0$ and $1-\varepsilon\to1$. Therefore
\[
    D_{\mathrm{KL}}
    \left(
    p^{(J_\varepsilon,\lambda)}
    \,\|\,p^{(J,\lambda)}
    \right)
    \to\infty .
\]
Since the symmetric KL is the sum of this KL term and its reverse direction, it also diverges. Taking any sequence $\varepsilon_i\downarrow0$ gives a sequence of predictors for which Eq.~\eqref{eq:counterexample_mse} tends to zero while
$\mathcal{J}_{\mathrm{SKL}}^{(\lambda)}\to\infty$. Smooth bump approximations of $\mathbf{1}_{A_\varepsilon}$ give the same conclusion with continuous predictors.
\end{proof}

Theorem~\ref{thm:skl_mse_bound_counterexample} shows the limitation of pointwise regression for inverse design. A narrow but strongly underestimated region can contribute little to MSE because it occupies little probability mass under $p$, but after exponentiation it can dominate $p^{(J_\theta,\lambda)}$ and therefore dominate guided optimization or generation.

Together, these results give a practical interpretation of the mini-batch loss in Eq.~\eqref{eq:si_minibatch_statement_skl}. The normalized exponential weights in the mini-batch are softmax probabilities over the simulated and predicted costs. MSE asks the predictor to match all costs as scalar values, whereas SKL asks it to match the probability redistribution that will actually be used by optimization and generation. This is why MSE is a natural baseline for regression accuracy, but SKL is the appropriate objective when the predictor is used to define a cost-biased inverse-design distribution.

\subsection{The choice of $\varphi_k$ and the rationale for preserving positive definiteness}
\label{subsec:si_damping_pd}

Our covariance construction relies on the factorization in Theorem~\ref{thm:semi_numerical_sqrt}, which requires $\mathbf{B}_k \succ 0$ and $\gamma_k>0$. This condition is not merely technical. In our method, $\mathbf{B}_k$ serves as a local proxy for the Jacobian-induced matrix that determines the Gaussian covariance; if $\mathbf{B}_k$ loses positive definiteness, then (i) the implied covariance may become indefinite, and (ii) the matrix square-root used in sampling becomes ill-conditioned or undefined, which can destabilize the entire guidance procedure. We therefore enforce $\mathbf{B}_k \succ 0$ throughout sampling.

The update in Theorem~\ref{thm:updataB} inherits the structure of quasi-Newton secant updates. As in DFP/BFGS-type methods, positive definiteness can be preserved if the curvature condition holds:
\begin{equation}
\mathbf{s}_k^{\top}\mathbf{y}_k>0.
\label{eq:curvature_cond}
\end{equation}
In classical optimization, one can enforce \eqref{eq:curvature_cond} by choosing step sizes that satisfy Wolfe conditions~\cite{nocedal2006numerical}. In our setting, however, $\mathbf{s}_k=\mathbf{x}_{k+1}-\mathbf{x}_k$ is produced by a fixed-time ODE integrator along a prescribed time grid, and the ``step size'' cannot be freely adjusted. Consequently, \eqref{eq:curvature_cond} may fail, especially when the velocity field varies rapidly over time or when numerical error accumulates.

To guarantee \eqref{eq:curvature_cond} and maintain $\mathbf{B}_k\succ 0$, we use a damped secant vector and replace $\mathbf{y}_k$ by
\begin{equation}
\widehat{\mathbf{y}}_{k}
=
\varphi_{k}\mathbf{y}_{k}+(1-\varphi_{k})\mathbf{B}_{k}\mathbf{s}_{k},
\qquad
\varphi_k\in[0,1].
\label{eq:yhat_def}
\end{equation}
This interpolation has two effects. First, it enforces a controlled curvature:
\begin{equation}
\mathbf{s}_k^{\top}\widehat{\mathbf{y}}_k
=
\varphi_k\,\mathbf{s}_k^{\top}\mathbf{y}_k+(1-\varphi_k)\,\mathbf{s}_k^{\top}\mathbf{B}_k\mathbf{s}_k.
\label{eq:curvature_yhat}
\end{equation}
Since $\mathbf{B}_k\succ 0$ implies $\mathbf{s}_k^{\top}\mathbf{B}_k\mathbf{s}_k>0$ for any nonzero $\mathbf{s}_k$, choosing $\varphi_k$ sufficiently small guarantees $\mathbf{s}_k^{\top}\widehat{\mathbf{y}}_k>0$ even when $\mathbf{s}_k^{\top}\mathbf{y}_k\le 0$. Second, it prevents overly aggressive updates: when $\varphi_k\to 0$, we have $\widehat{\mathbf{y}}_k=\mathbf{B}_k\mathbf{s}_k$, and the secant update becomes an identity update. Indeed, with $\widehat{\mathbf{y}}_k=\mathbf{B}_k\mathbf{s}_k$,
\begin{equation}
\mathbf{V}_k=\mathbf{I}-\rho_k\mathbf{s}_k\widehat{\mathbf{y}}_k^{\top}
=\mathbf{I}-\rho_k\mathbf{s}_k\mathbf{s}_k^{\top}\mathbf{B}_k,
\qquad
\rho_k=\frac{1}{\widehat{\mathbf{y}}_k^{\top}\mathbf{s}_k}=\frac{1}{\mathbf{s}_k^{\top}\mathbf{B}_k\mathbf{s}_k},
\end{equation}
and the standard symmetric secant update yields $\widetilde{\mathbf{B}}_{k+1}=\mathbf{B}_k$, so the DFP/BFGS correction is effectively skipped and only the affine time-rescaling remains:
\begin{equation}
\mathbf{B}_{k+1}=u_k\mathbf{B}_k+w_k\mathbf{I}.
\end{equation}
Thus, damping provides a continuous mechanism that interpolates between a full secant correction ($\varphi_k=1$) and a conservative no-correction step ($\varphi_k=0$).

We follow the standard damped-update design used in quasi-Newton methods~\cite{powell1978algorithms,byrd1994representations} and choose $\varphi_k$ to keep the effective curvature within a controlled band relative to $\mathbf{s}_k^{\top}\mathbf{B}_k\mathbf{s}_k$. Define
\begin{equation}
\tau_k \triangleq
\frac{\mathbf{s}_k^{\top}\mathbf{y}_k}{\mathbf{s}_k^{\top}\mathbf{B}_k\mathbf{s}_k}.
\label{eq:tau_def}
\end{equation}
We then set
\begin{equation}
\varphi_{k} =
\begin{cases}
\dfrac{\sigma_{2}}{1-\tau_{k}}, & \tau_{k}<1-\sigma_{2},\\[6pt]
\dfrac{\sigma_{3}}{\tau_{k}-1}, & \tau_{k}>1+\sigma_{3},\\[6pt]
1, & \text{otherwise},
\end{cases}
\qquad
\sigma_2,\sigma_3\in(0,1).
\label{eq:phi_piecewise}
\end{equation}
This choice enforces
\begin{equation}
(1-\sigma_{2})\,\mathbf{s}_{k}^{\top}\mathbf{B}_{k}\mathbf{s}_{k}
\le
\mathbf{s}_{k}^{\top}\widehat{\mathbf{y}}_{k}
\le
(1+\sigma_{3})\,\mathbf{s}_{k}^{\top}\mathbf{B}_{k}\mathbf{s}_{k},
\label{eq:curvature_band}
\end{equation}
which follows by substituting \eqref{eq:yhat_def} into \eqref{eq:curvature_yhat} and using \eqref{eq:tau_def}. In particular, the left inequality in \eqref{eq:curvature_band} guarantees $\mathbf{s}_k^{\top}\widehat{\mathbf{y}}_k>0$ whenever $\mathbf{s}_k\neq 0$, so the curvature condition holds for the damped pair $(\mathbf{s}_k,\widehat{\mathbf{y}}_k)$ even when the raw pair violates \eqref{eq:curvature_cond}. As a result, the symmetric secant update in Theorem~\ref{thm:updataB} preserves $\mathbf{B}_k\succ 0$ in the same sense as standard quasi-Newton updates.

\clearpage
\section{Algorithms and time-complexity analysis}
\label{sec:si_algorithms}

\subsection{visual-feature-preserving latent optimization}

\begin{algorithm}[H]
\caption{Time-annealed visual-feature-preserving latent optimization}
\label{alg:time_annealed_opt}
\begin{algorithmic}[1]
\Require Initial latent $\mathbf{x}_0$; cost $J(\mathbf{x},\lambda)$; guidance scale $\lambda>0$; step size $\eta$; iterations $K$; velocity field $\mathbf{v}_t(\mathbf{x})$; noise schedule $(\alpha_t,\sigma_t)$; time bounds $(t_{\min},t_{\max})$; annealing lower bound $t_{\min}^{(k)}$
\Ensure Optimized latent $\mathbf{x}_K$

\For{$k=0$ \textbf{to} $K-1$}
    \State Sample $t_k \sim \mathcal{U}\!\left(t_{\min}^{(k)},\,t_{\max}\right)$
    \State $\mathbf{x}_{t_k}^{k} \Leftarrow \textsc{AddNoise}(\mathbf{x}_k,t_k)$
    \State $\mathbf{g} \Leftarrow \nabla_{\mathbf{x}}J(\mathbf{x}_k,\lambda)
    -
    \dfrac{
    \mathbf{v}_{t_k}\!\left(\mathbf{x}_{t_k}^{k}\right)
    -
    \left(\dot{\alpha}_{t_k}/\alpha_{t_k}\right)\mathbf{x}_{t_k}^{k}
    }{
    s_{t_k}
    }$
    \State $\mathbf{x}_{k+1} \Leftarrow \mathbf{x}_{k}-\eta\,\mathbf{g}$
\EndFor

\State \Return $\mathbf{x}_K$
\end{algorithmic}
\end{algorithm}

\textbf{visual-feature-preserving latent optimization.}
Alg.~\ref{alg:time_annealed_opt} performs gradient descent in the Shape-VAE latent space while explicitly counteracting distributional drift. In the aerodynamic inverse-design setting, $\mathbf{x}$ is instantiated as the latent design variable $\mathbf{z}$ and $J(\mathbf{x},\lambda)$ is instantiated as the guided aerodynamic objective induced by $\lambda\hat{c}_\phi$.
At iteration $k$, we sample a noise level $t_k$ from an annealed interval $\big[t_{\min}^{(k)},t_{\max}\big]$ and perturb the current iterate via
$\mathbf{x}_{t_k}\leftarrow \textsc{AddNoise}(\mathbf{x}_k,t_k)$ using the same schedule as the generative model.
We then form an update direction
$\mathbf{g}=\nabla_{\mathbf{x}}J(\mathbf{x}_k,\lambda)\;-\;({\mathbf{v}_{t_k}(\mathbf{x}_{t_k}^{k})-(\dot{\alpha_{t_k}}/\alpha_{t_k})\mathbf{x}_{t_k}^{k}})/{s_{t_k}}$,
where the first term improves the objective and the second term acts as a prior correction induced by the flow field,
encouraging $\mathbf{x}_k$ to remain in a plausible region of the learned visual design distribution.
Finally, we update $\mathbf{x}_{k+1}\leftarrow \mathbf{x}_k-\eta\,\mathbf{g}$ and repeat for $K$ iterations.

\textbf{Time complexity.}
Let $d$ be the latent dimension.
Denote by $C_{\text{DiT}}^{\text{fwd}}$ the cost of one forward evaluation of the velocity field $\mathbf{v}_t(\cdot)$,
and by $C_{J}^{\text{bwd}}$ the cost of computing $\nabla_{\mathbf{x}}J(\mathbf{x},\lambda)$ via automatic differentiation
(including forward and backward passes of the cost predictor).
Each iteration requires: (i) one \textsc{AddNoise} and several vector operations, costing $\mathcal{O}(d)$;
(ii) one DiT forward, costing $\mathcal{O}(C_{\text{DiT}}^{\text{fwd}})$;
(iii) one cost-gradient evaluation, costing $\mathcal{O}(C_{J}^{\text{bwd}})$.
Therefore, the total complexity over $K$ iterations is
\[
\mathcal{O}\!\Big(K\big[C_{\text{DiT}}^{\text{fwd}}+C_{J}^{\text{bwd}}+d\big]\Big).
\]
The memory overhead (excluding model parameters) is $\mathcal{O}(d)$ for storing the current iterate and intermediate activations; in practice it is dominated by the activation memory of the cost predictor during backpropagation.

\subsection{SA-MC guided generation}
\label{subsec:si_samc_algorithm}

SA-MC is the sampling counterpart of the visual-feature-preserving optimizer above. Instead of returning one maximum-density latent point, it approximately draws guided samples from the cost-biased distribution in Eq.~\eqref{eq:si_cost_biased_latent}. The main sampler is given in Algorithm~\ref{alg:si_samc_main}. The four routines that follow are its subroutines: damping stabilizes secant pairs, \textsc{UpdateB} maintains the compact covariance proxy, \textsc{SemiNumericalSqrt} constructs the corresponding sampling map, and \textsc{SampleTiltedMean} estimates the Monte Carlo guidance correction.

\begin{algorithm}[H]
\caption{SA-MC guided generation}
\label{alg:si_samc_main}
\begin{algorithmic}[1]
\Require Visual input $r$; operating condition $o$; guidance scale $\lambda$; velocity field $v_t$; cost predictor $\hat{c}_\phi$; time grid $\{t_k\}_{k=0}^{K}$; Monte Carlo sample size $S$; memory length $M$; initial covariance scale $\gamma_0$.
\Ensure Guided latent design $\mathbf{z}_1$ and decoded geometry.
\State Initialize $\mathbf{z}_{t_0}$ from the source distribution and set the memory queue $\mathcal{M}_0\leftarrow[\,]$.
\For{$k=0$ \textbf{to} $K-1$}
    \State Evaluate $\mathbf{v}_k\leftarrow v_{t_k}(\mathbf{z}_{t_k}\mid r)$ and compute the clean prediction $\boldsymbol{\mu}_{t_k}$.
    \If{$k\ge 1$}
        \State Form the secant pair $(\mathbf{s}_{k-1},\mathbf{y}_{k-1})$ from consecutive sampler states and velocities.
        \State Apply damping to obtain $\widehat{\mathbf{y}}_{k-1}$ and update the memory queue $\mathcal{M}_k$.
        \State Reconstruct the compact matrix $(\gamma_k,\mathbf{U}_k,\boldsymbol{\Gamma}_k)$ with \textsc{UpdateB}.
    \EndIf
    \State Compute the implicit square-root map $\mathbf{L}_k$ with \textsc{SemiNumericalSqrt}.
    \State Draw $S$ clean-latent proposals $\mathbf{z}_{1}^{[s]}=\boldsymbol{\mu}_{t_k}+\mathbf{L}_k\boldsymbol{\epsilon}^{[s]}$.
    \State Evaluate $\hat{c}_\phi(\mathbf{z}_{1}^{[s]},o;\lambda)$ and compute normalized weights $\omega^{[s]}\propto\exp[-\lambda\hat{c}_\phi(\mathbf{z}_{1}^{[s]},o;\lambda)]$.
    \State Set $\mathbf{g}_{t_k}^{\mathrm{SA\mbox{-}MC}}=b_{t_k}\sum_{s=1}^{S}\omega^{[s]}(\mathbf{z}_{1}^{[s]}-\boldsymbol{\mu}_{t_k})$.
    \State Advance the guided flow with velocity $v_{t_k}(\mathbf{z}_{t_k}\mid r)+\mathbf{g}_{t_k}^{\mathrm{SA\mbox{-}MC}}$.
\EndFor
\State Decode the final latent variable to an SDF and mesh.
\State \Return Guided 3D design.
\end{algorithmic}
\end{algorithm}

\noindent\textbf{Time complexity.}
Let $d$ be the latent dimension, $M$ the memory length, $r\le 2M$ the compact rank, $S$ the Monte Carlo sample size, $C_{\mathrm{DiT}}^{\mathrm{fwd}}$ the cost of one velocity-field forward pass, and $C_J^{\mathrm{fwd}}$ the cost of one predictor forward pass. At each sampling step, SA-MC requires one velocity-field evaluation, compact covariance reconstruction, one compact square-root construction, and $S$ guided cost evaluations. The per-step cost is
\begin{equation}
    \mathcal{O}\!\left(
    C_{\mathrm{DiT}}^{\mathrm{fwd}}
    +dM^2+M^3
    +dr^2+r^3
    +S\left(dr+r^2+C_J^{\mathrm{fwd}}\right)
    \right).
    \label{eq:si_samc_complexity}
\end{equation}
The memory overhead is $\mathcal{O}(dr+r^2)$ for the compact covariance factors plus $\mathcal{O}(Sd)$ if all Monte Carlo perturbations are stored. A streaming implementation reduces the sample workspace to $\mathcal{O}(d)$.

\subsubsection{Damping (Damp.) subroutine}

\begin{algorithm}[!h]
\caption{\textsc{Damp}$(\mathbf{y},\mathbf{s},\gamma,\mathbf{U},\boldsymbol{\Gamma})$}
\label{alg:damp}
\begin{algorithmic}[1]
\Require Secant vectors $\mathbf{s},\mathbf{y}$; compact form $(\gamma,\mathbf{U},\boldsymbol{\Gamma})$ defining 
$\mathbf{B}=\gamma\mathbf{I}+\mathbf{U}\boldsymbol{\Gamma}\mathbf{U}^\top$; parameters $\sigma_2\in(0,1)$ and $\sigma_3>0$
\Ensure Damped vector $\widehat{\mathbf{y}}$ and coefficient $\varphi$

\State $\mathbf{B}\mathbf{s} \Leftarrow \textsc{ApplyB}(\mathbf{s};\gamma,\mathbf{U},\boldsymbol{\Gamma})$({$=\gamma\mathbf{s}+\mathbf{U}\boldsymbol{\Gamma}(\mathbf{U}^\top\mathbf{s})$})

\State $s^\top y \Leftarrow \mathbf{s}^\top\mathbf{y}$
\State $s^\top Bs \Leftarrow \mathbf{s}^\top(\mathbf{B}\mathbf{s})$
\State $\tau \Leftarrow (s^\top y)/(s^\top Bs)$
\State $\varphi \Leftarrow 1$

\If{$\tau < 1-\sigma_2$}
    \State $\varphi \Leftarrow \sigma_2/(1-\tau)$
\ElsIf{$\tau > 1+\sigma_3$}
    \State $\varphi \Leftarrow \sigma_3/(\tau-1)$
\EndIf

\State $\widehat{\mathbf{y}} \Leftarrow \varphi\,\mathbf{y}+(1-\varphi)\,\mathbf{B}\mathbf{s}$

\State \Return $(\widehat{\mathbf{y}},\varphi)$
\end{algorithmic}
\end{algorithm}

\noindent\textbf{Damping for positive definiteness.}
Alg.~\ref{alg:damp} enforces a curvature condition for the DFP-style secant update used in our covariance estimator.
Given $\mathbf{B}=\gamma\mathbf{I}+\mathbf{U}\boldsymbol{\Gamma}\mathbf{U}^\top$ and a new secant pair $(\mathbf{s},\mathbf{y})$, we form the ratio
$\tau=(\mathbf{s}^\top\mathbf{y})/(\mathbf{s}^\top\mathbf{B}\mathbf{s})$ and replace $\mathbf{y}$ by a convex combination
$\widehat{\mathbf{y}}=\varphi\,\mathbf{y}+(1-\varphi)\mathbf{B}\mathbf{s}$.
The piecewise choice of $\varphi$ guarantees
$(1-\sigma_2)\,\mathbf{s}^\top\mathbf{B}\mathbf{s}\le \mathbf{s}^\top\widehat{\mathbf{y}}\le (1+\sigma_3)\,\mathbf{s}^\top\mathbf{B}\mathbf{s}$,
so in particular $\mathbf{s}^\top\widehat{\mathbf{y}}>0$ whenever $\mathbf{B}\succ 0$.
This preserves the positive definiteness of the updated $\mathbf{B}$ and stabilizes the downstream matrix square-root used for Monte Carlo sampling.

\noindent\textbf{Time complexity.}
Let $d$ be the latent dimension and let $r$ denote the rank of the compact representation (in our memory-queue implementation, $r=2m$).
The dominant computation in Alg.~\ref{alg:damp} is applying the compact secant operator to a vector,
\[
\mathbf{B}\mathbf{s}=\gamma\,\mathbf{s}+\mathbf{U}\,\boldsymbol{\Gamma}\,(\mathbf{U}^\top\mathbf{s}).
\]
Computing $\mathbf{U}^\top\mathbf{s}$ costs $\mathcal{O}(dr)$, multiplying by $\boldsymbol{\Gamma}\in\mathbb{R}^{r\times r}$ costs $\mathcal{O}(r^2)$, and multiplying by $\mathbf{U}$ costs $\mathcal{O}(dr)$.
All remaining operations (inner products, scalar branches, and the convex combination) are $\mathcal{O}(d)$.
Therefore Alg.~\ref{alg:damp} runs in
\[
\mathcal{O}(dr+r^2)
\]
time per call.
Storing the compact factors $(\mathbf{U},\boldsymbol{\Gamma})$ requires $\mathcal{O}(dr+r^2)$ memory, and the algorithm uses an additional $\mathcal{O}(d)$ workspace.

\subsubsection{UpdateB subroutine}

\begin{algorithm}[!h]
\caption{\textsc{UpdateB}$(\mathcal{M},\widehat{\gamma})$}
\label{alg:updateB}
\begin{algorithmic}[1]
\Require Memory queue $\mathcal{M}=\{(\mathbf{s}_i,\widehat{\mathbf{y}}_i,u_i,w_i)\}_{i=0}^{\ell-1}$ ordered from oldest to newest, with $\ell\le m$; Jacobi-scale initialization $\widehat{\gamma}>0$
\Ensure Compact form $(\gamma,\mathbf{U},\boldsymbol{\Gamma})$ such that $\mathbf{B}=\gamma\mathbf{I}+\mathbf{U}\boldsymbol{\Gamma}\mathbf{U}^\top$

\State $\gamma \Leftarrow \widehat{\gamma}$
\State $\mathbf{U} \Leftarrow [\,]$
\State $\boldsymbol{\Gamma} \Leftarrow [\,]$

\For{$i=0$ \textbf{to} $\ell-1$}
    \State $\mathbf{s} \Leftarrow \mathbf{s}_i$
    \State $\widehat{\mathbf{y}} \Leftarrow \widehat{\mathbf{y}}_i$
    \State $u \Leftarrow u_i,\quad w \Leftarrow w_i$
    \State $\rho \Leftarrow 1/(\widehat{\mathbf{y}}^\top\mathbf{s})$

    \If{$\mathbf{U}=[\,]$}
        \State $\mathbf{U} \Leftarrow [\,\mathbf{s},\widehat{\mathbf{y}}\,]$ 
        \Comment{Append columns}
        \State $\boldsymbol{\Gamma} \Leftarrow
        \begin{bmatrix}
          0 & -\gamma\rho\\
          -\gamma\rho & \rho+\rho^2\gamma\,\mathbf{s}^\top\mathbf{s}
        \end{bmatrix}$
    \Else
        \State $\mathbf{q} \Leftarrow \mathbf{U}^\top\mathbf{s}$
        \State $\mathbf{p} \Leftarrow \boldsymbol{\Gamma}\mathbf{q}$
        \State $\tau \Leftarrow \mathbf{q}^\top\mathbf{p}$
        \State $\mathbf{U} \Leftarrow [\,\mathbf{U},\mathbf{s},\widehat{\mathbf{y}}\,]$
        \State $\boldsymbol{\Gamma} \Leftarrow
        \begin{bmatrix}
          \boldsymbol{\Gamma} & \mathbf{0} & -\rho\,\mathbf{p}\\
          \mathbf{0}^\top & 0 & -\gamma\rho\\
          -\rho\,\mathbf{p}^\top & -\gamma\rho & \rho+\rho^2(\tau+\gamma\,\mathbf{s}^\top\mathbf{s})
        \end{bmatrix}$
    \EndIf

    \State $\boldsymbol{\Gamma} \Leftarrow u\,\boldsymbol{\Gamma}$
    \State $\gamma \Leftarrow u\,\gamma+w$
\EndFor

\State \Return $(\gamma,\mathbf{U},\boldsymbol{\Gamma})$
\end{algorithmic}
\end{algorithm}

\noindent\textbf{Low-rank reconstruction of the secant matrix.}
Alg.~\ref{alg:updateB} rebuilds the compact representation of the local Jacobian-induced matrix
$\mathbf{B}$ from the memory queue $\mathcal{M}$.
Starting from a Jacobi-scale initialization $\gamma\leftarrow\widehat{\gamma}$, it sequentially applies the
DFP-form secant updates using the stored pairs $(\mathbf{s}_i,\widehat{\mathbf{y}}_i)$, while accounting for
the time-dependent affine transform $\mathbf{B}\leftarrow u_i\mathbf{B}+w_i\mathbf{I}$ induced by the schedule.
At each iteration, the algorithm maintains $\mathbf{B}$ in the form
$\mathbf{B}=\gamma\mathbf{I}+\mathbf{U}\boldsymbol{\Gamma}\mathbf{U}^\top$, where $\mathbf{U}$ stores the
history of secant vectors (interleaving $\mathbf{s}$ and $\widehat{\mathbf{y}}$) and $\boldsymbol{\Gamma}$
stores the corresponding small dense coefficients.
The scalars $\rho=(\widehat{\mathbf{y}}^\top\mathbf{s})^{-1}$ and $\tau=\mathbf{q}^\top\mathbf{p}$ with
$\mathbf{q}=\mathbf{U}^\top\mathbf{s}$, $\mathbf{p}=\boldsymbol{\Gamma}\mathbf{q}$ implement the closed-form
recursion in Theorem~\ref{thm:Bk_compact_form}, enabling updates without forming any $d\times d$ matrices.

\noindent\textbf{Time complexity.}
Let $d$ be the latent dimension and let $\ell\le m$ denote the number of stored secant pairs.
At iteration $i$, the compact rank is $r_i=2i$.
The dominant costs are forming $\mathbf{q}=\mathbf{U}^\top\mathbf{s}$ in $\mathcal{O}(d r_i)$,
computing $\mathbf{p}=\boldsymbol{\Gamma}\mathbf{q}$ in $\mathcal{O}(r_i^2)$, and evaluating
$\tau=\mathbf{q}^\top\mathbf{p}$ in $\mathcal{O}(r_i)$.
All remaining operations (updating $\gamma$, assembling the new blocks, and scaling by $u$) are lower order.
Summing over $i=0,\dots,\ell-1$ yields
\[
\sum_{i=0}^{\ell-1}\mathcal{O}(d r_i + r_i^2)
=\mathcal{O}\!\left(d\sum_{i=0}^{\ell-1}2i+\sum_{i=0}^{\ell-1}(2i)^2\right)
=\mathcal{O}(d\ell^2+\ell^3).
\]
Since $\ell\le m$ and $r=2m$, this can be equivalently expressed as $\mathcal{O}(d m^2+m^3)$ in the worst case.
The memory footprint is dominated by storing $\mathbf{U}\in\mathbb{R}^{d\times 2\ell}$ and
$\boldsymbol{\Gamma}\in\mathbb{R}^{2\ell\times 2\ell}$, i.e., $\mathcal{O}(d\ell+\ell^2)$.

\subsubsection{SemiNumericalSqrt subroutine}
\begin{algorithm}[!h]
\caption{\textsc{SemiNumericalSqrt}$(\gamma,\mathbf{U},\boldsymbol{\Gamma})$}
\label{alg:semi_sqrt}
\begin{algorithmic}[1]
\Require $\gamma>0$; $\mathbf{U}\in\mathbb{R}^{d\times r}$; symmetric matrix $\boldsymbol{\Gamma}\in\mathbb{R}^{r\times r}$
\Ensure A linear map $\mathbf{L}$ such that 
$\mathbf{B}=\gamma\mathbf{I}+\mathbf{U}\boldsymbol{\Gamma}\mathbf{U}^\top=\mathbf{L}\mathbf{L}^\top$

\If{$r=0$}
    \State \Return $\mathbf{L}=\sqrt{\gamma}\mathbf{I}$
\EndIf

\State Compute the reduced QR decomposition 
$\mathbf{U}=\mathbf{Q}\mathbf{R}$, where $\mathbf{Q}\in\mathbb{R}^{d\times r}$ is orthonormal and $\mathbf{R}\in\mathbb{R}^{r\times r}$ is upper triangular

\State Form the small matrix
$\mathbf{C} \Leftarrow \gamma\mathbf{I}_r+\mathbf{R}\boldsymbol{\Gamma}\mathbf{R}^\top$

\State Stabilize 
$\mathbf{C} \Leftarrow (\mathbf{C}+\mathbf{C}^\top)/2$
and optionally 
$\mathbf{C} \Leftarrow \mathbf{C}+\epsilon\mathbf{I}_r$

\State Compute the Cholesky factorization
$\mathbf{C}=\mathbf{L}_{\mathbf{C}}\mathbf{L}_{\mathbf{C}}^\top$,
where $\mathbf{L}_{\mathbf{C}}$ is lower triangular

\State Define $\mathbf{L}$ implicitly by
\[
\mathbf{L}
=
\sqrt{\gamma}\mathbf{I}
+
\mathbf{Q}
\left(
\mathbf{L}_{\mathbf{C}}-\sqrt{\gamma}\mathbf{I}_r
\right)
\mathbf{Q}^\top
\]

\State \Return $\mathbf{L}$, implemented via
\[
\textsc{ApplyL}(\mathbf{x})
=
\sqrt{\gamma}\mathbf{x}
+
\mathbf{Q}
\left(
\mathbf{L}_{\mathbf{C}}-\sqrt{\gamma}\mathbf{I}_r
\right)
\left(
\mathbf{Q}^\top\mathbf{x}
\right)
\]
\end{algorithmic}
\end{algorithm}

\noindent\textbf{Semi-numerical square root in compact form.}
Alg.~\ref{alg:semi_sqrt} computes a numerically stable factorization of the low-rank matrix
$\mathbf{B}=\gamma\mathbf{I}+\mathbf{U}\boldsymbol{\Gamma}\mathbf{U}^\top$ without forming any $d\times d$
dense matrices.
The key idea is to isolate the nontrivial action of $\mathbf{B}$ to the $r$-dimensional subspace spanned by
the columns of $\mathbf{U}$.
Using the reduced QR factorization $\mathbf{U}=\mathbf{Q}\mathbf{R}$, we obtain
\[
\mathbf{B}
=\gamma\mathbf{I}+\mathbf{Q}\big(\mathbf{R}\boldsymbol{\Gamma}\mathbf{R}^\top\big)\mathbf{Q}^\top,
\]
so the factorization reduces to taking a Cholesky decomposition of the small matrix
$\mathbf{C}=\gamma\mathbf{I}_r+\mathbf{R}\boldsymbol{\Gamma}\mathbf{R}^\top$.
The resulting square-root map is represented as a rank-$r$ correction to $\sqrt{\gamma}\mathbf{I}$:
\[
\mathbf{L}=\sqrt{\gamma}\mathbf{I}+\mathbf{Q}\big(\mathbf{L}_{\mathbf{C}}-\sqrt{\gamma}\mathbf{I}_r\big)\mathbf{Q}^\top,
\qquad \mathbf{C}=\mathbf{L}_{\mathbf{C}}\mathbf{L}_{\mathbf{C}}^\top,
\]
which can be applied to vectors using only matrix--vector products with $\mathbf{Q}$ and $\mathbf{L}_{\mathbf{C}}$.
The symmetrization and optional diagonal jitter $\epsilon\mathbf{I}_r$ ensure numerical stability when $\mathbf{C}$
is close to singular due to finite-precision errors.

\noindent\textbf{Time complexity.}
Let $d$ be the latent dimension and let $r$ denote the compact rank (in our memory-queue implementation, $r=2\ell\le 2m$).
Computing the reduced QR factorization of $\mathbf{U}\in\mathbb{R}^{d\times r}$ costs $\mathcal{O}(dr^2)$.
Forming $\mathbf{C}=\gamma\mathbf{I}_r+\mathbf{R}\boldsymbol{\Gamma}\mathbf{R}^\top$ costs $\mathcal{O}(r^3)$
(e.g., one $r\times r$ multiply to form $\mathbf{R}\boldsymbol{\Gamma}$ and one to post-multiply by $\mathbf{R}^\top$),
and the Cholesky factorization of $\mathbf{C}\in\mathbb{R}^{r\times r}$ also costs $\mathcal{O}(r^3)$.
Therefore, the total preprocessing cost is
\[
\mathcal{O}(dr^2+r^3).
\]
In practice, $\mathbf{L}$ is never materialized as a dense $d\times d$ matrix.
To apply $\mathbf{L}$ to a vector $\mathbf{x}\in\mathbb{R}^d$, we compute
$\mathbf{q}=\mathbf{Q}^\top\mathbf{x}$, $\mathbf{u}=(\mathbf{L}_{\mathbf{C}}-\sqrt{\gamma}\mathbf{I}_r)\mathbf{q}$,
and then $\mathbf{Q}\mathbf{u}$, which costs $\mathcal{O}(dr+r^2)$ per application.
The memory footprint is $\mathcal{O}(dr+r^2)$ to store $\mathbf{Q}$ and the small factors
(e.g., $\mathbf{R}$, $\boldsymbol{\Gamma}$, $\mathbf{C}$, and $\mathbf{L}_{\mathbf{C}}$).

\subsubsection{SampleTiltedMean subroutine}
\begin{algorithm}[!h]
\caption{\textsc{SampleTiltedMean}}
\label{alg:sample_tilted_mean}
\begin{algorithmic}[1]
\Require One-step prediction $\mathbf{x}_1^{\mathrm{pred}}$; factor $\mathbf{L}$; time $t$; guide scale $\lambda$; cost $J(\cdot,\lambda)$; Monte Carlo size $S$; small constant $\varepsilon>0$; schedule $b_t$
\Ensure Tilted-mean estimate $\mathbf{g}$

\State Sample $\{\boldsymbol{\epsilon}^{(i)}\}_{i=1}^{S}\sim \mathcal{N}(\mathbf{0},\mathbf{I})$
\State $\sigma(t) \Leftarrow \dfrac{1-t+\varepsilon}{\sqrt{t+\varepsilon}}$

\For{$i=1$ \textbf{to} $S$}
    \State $\boldsymbol{\xi}^{(i)} \Leftarrow \mathbf{L}\boldsymbol{\epsilon}^{(i)}$
    \State $\mathbf{x}_1^{(i)} \Leftarrow \mathbf{x}_1^{\mathrm{pred}}+\boldsymbol{\xi}^{(i)}$
    \State $\ell^{(i)} \Leftarrow -J(\mathbf{x}_1^{(i)},\lambda)$
\EndFor

\State $\ell_{\max} \Leftarrow \max_{1\le i\le S}\ell^{(i)}$
\State $\omega^{(i)} \Leftarrow \exp\!\left(\ell^{(i)}-\ell_{\max}\right),\quad i=1,\ldots,S$
\State $Z \Leftarrow \sum_{i=1}^{S}\omega^{(i)}$
\State $\mathbf{g} \Leftarrow b_t\sum_{i=1}^{S}\dfrac{\omega^{(i)}}{Z}\,\boldsymbol{\xi}^{(i)}$

\State \Return $\mathbf{g}$
\end{algorithmic}
\end{algorithm}

\noindent\textbf{Tilted-mean estimator.}
Alg.~\ref{alg:sample_tilted_mean} approximates the \emph{tilted mean correction} that appears in our guidance term.
Given the one-step prediction $\mathbf{x}_1^{\mathrm{pred}}$ and a factor $\mathbf{L}$ satisfying
$\mathbf{B}=\mathbf{L}\mathbf{L}^\top$, we draw Gaussian perturbations
$\boldsymbol{\xi}^{(i)}=\mathbf{L}\boldsymbol{\epsilon}^{(i)}$ and evaluate the cost
$J(\mathbf{x}_1^{(i)},\lambda)$ at the perturbed samples $\mathbf{x}_1^{(i)}=\mathbf{x}_1^{\mathrm{pred}}+\boldsymbol{\xi}^{(i)}$.
The weights $\omega^{(i)}\propto \exp(-J(\mathbf{x}_1^{(i)},\lambda))$ form an empirical approximation of the
exponentially tilted distribution, and the returned vector
\[
\mathbf{g}\approx \mathbb{E}\!\left[\boldsymbol{\xi}\ \middle|\ \text{tilt by } \exp(-J)\right]
\]
is the weighted average of the perturbations.
We compute the weights in a numerically stable manner by subtracting $\ell_{\max}$ before exponentiation.

\noindent\textbf{Time complexity.}
Let $d$ be the latent dimension and let $S$ denote the Monte Carlo sample size.
We implement $\mathbf{L}$ via the implicit routine \textsc{ApplyL} (Alg.~\ref{alg:semi_sqrt}), so each sample requires
one application of $\mathbf{L}$ to a vector and one forward evaluation of the cost predictor.
If $\mathbf{L}$ is stored in compact form with rank $r$ (in our memory-queue setting, typically $r\le 2m$), then
\textsc{ApplyL} costs $\mathcal{O}(dr+r^2)$, and the total per-step complexity is
\[
\mathcal{O}\!\big(S\,(dr+r^2)\big)\;+\;\mathcal{O}\!\big(S\,C_{J}^{\mathrm{fwd}}\big),
\]
where $C_{J}^{\mathrm{fwd}}$ denotes the cost of a single forward pass of the cost predictor.
The remaining operations (computing $\ell_{\max}$, forming the normalized weights, and the weighted sum) contribute
$\mathcal{O}(S)$ scalar work and $\mathcal{O}(Sd)$ vector accumulations, which are dominated by the two terms above.
The memory footprint is $\mathcal{O}(Sd)$ if all samples (and their intermediate vectors) are stored; in practice, this can be reduced to
$\mathcal{O}(d)$ by streaming the computation (accumulating $\ell_{\max}$ and the weighted sum on the fly), at the cost of an additional pass over the $S$ samples.

\clearpage
\section{Experimental settings}
\label{sec:si_experiment_settings}

\subsection{2D Experimental Setting}
\label{sec:2d-setting}
We study a synthetic 2D problem on the bounded square domain
$\mathcal{X}=[-3.5,\,3.5]\times[-3.5,\,3.5]$.
All densities and costs are represented on a uniform Cartesian grid; we use a $256\times256$ grid to construct and store the underlying fields, while KL metrics for model selection are evaluated on a separate $N_{\text{grid}}\times N_{\text{grid}}$ grid (default $N_{\text{grid}}=250$) for a controllable accuracy--speed trade-off.

\paragraph{Data distribution.}
The data distribution $p_0(\mathbf{x})$ is defined by a smooth Gaussian random field (GRF) on the $256\times256$ grid.
Concretely, we build a grid potential and convert it into a truncated density over $\mathcal{X}$ by exponentiation and numerical normalization:
mass outside $\mathcal{X}$ is set to zero and the normalization constant is computed by summing the resulting discrete probabilities on the grid.

\paragraph{Sampling.}
Sampling is consistent with the truncated-and-renormalized $p_0$.
We first evaluate the unnormalized log-density on grid nodes, convert it into a discrete probability mass function (PMF), and draw grid indices via multinomial sampling.
To reduce lattice artifacts, we add a small uniform jitter within each selected grid cell and clamp samples back to $\mathcal{X}$.

\paragraph{Cost landscape.}
The cost function $C(\mathbf{x})$ is generated independently as a smooth GRF on the same $256\times256$ grid.
The field is standardized (zero mean and fixed standard deviation) and queried continuously via bilinear interpolation with border padding.
Optionally, we use a peak-shift variant $C(\mathbf{x})$.

\paragraph{Tilted target distribution and KL evaluation.}
Given $p$ and $C$, we define the tilted target distribution
$q(\mathbf{x}) \propto p_0(\mathbf{x})\exp\!\big(-\alpha\,C(\mathbf{x})\big)$
($\alpha$ denoted as \texttt{scale}).
For evaluation, we discretize both $p_0$ and $C$ on an $N_{\text{grid}}\times N_{\text{grid}}$ grid over $\mathcal{X}$, form a grid PMF for $q$, and compute divergences between the ground-truth $q$ and the model-induced $q$ obtained by replacing $C$ with the predicted cost $\hat{C}$:
$\mathrm{KL}(q\|\hat{q})$, $\mathrm{KL}(\hat{q}\|q)$, and their symmetric version
$\mathrm{SKL}=\tfrac{1}{2}\big(\mathrm{KL}(q\|\hat{q})+\mathrm{KL}(\hat{q}\|q)\big)$.
We track these quantities every training step and select checkpoints by the minimum $\mathrm{SKL}$.

\paragraph{Reproducibility.}
All random seeds are fixed, and the generated GRF fields (for both $p$ and $C$) can be serialized and reused, ensuring identical experimental settings across runs.

\subsection{Vehicle aerodynamic optimization and generation setting}
\label{subsec:vehicle_setting}

\paragraph{Dataset.}
We adopt DrivAerNet++~\cite{elrefaie2024drivaernet++}, a large-scale vehicle aerodynamics dataset based on the parametric DrivAer model.
DrivAerNet++ provides thousands of watertight vehicle surfaces together with CFD-derived aerodynamic quantities obtained with OpenFOAM simulations.
For each shape, the dataset includes integrated force coefficients such as the drag and lift coefficients $(C_d, C_l)$.

\begin{table}[!h]
\centering
\caption{Summary of DrivAerNet++ and the aerodynamic quantities used in our vehicle experiments.}
\small
\setlength{\tabcolsep}{5.5pt}
\begin{tabular}{l l}
\toprule
\textbf{Item} & \textbf{DrivAerNet++ setting used in our experiments} \\
\midrule
Base geometry & DrivAer parametric vehicle model \\
Scale & $\sim$8{,}000 vehicle shapes with CFD annotations \\
CFD solver & OpenFOAM \\
Outputs used in this paper & Integrated drag coefficient $C_d$ \\
Target in this paper & Minimize $C_d$ (lower is better) \\
Usage in this paper & Train cost predictor $J_\theta$ \\
\bottomrule
\end{tabular}
\vspace{2pt}
\label{tab:drivaernetpp}
\end{table}

\paragraph{Target.}
Our optimization target is the drag coefficient \(C_d\) reported by CFD.
We use the standard definition
\begin{equation}
C_d \;\triangleq\; \frac{D}{q_\infty A_f},
\qquad
q_\infty \;\triangleq\; \tfrac{1}{2}\rho_\infty U_\infty^2,
\end{equation}
where \(D\) is the streamwise aerodynamic drag force acting on the vehicle, \(A_f\) is the frontal reference area, \(U_\infty\) is the freestream speed, and \(\rho_\infty\) is the freestream air density.
Accordingly, we instantiate the cost as \(J(\mathbf{x}) = C_d(\mathbf{x})\).

\subsection{Aircraft aerodynamic optimization and generation setting}
\label{subsec:aircraft_setting}

\paragraph{Dataset.}
Following \cite{sung2025blendednet}, we adopt \textsc{BlendedNet}, a high-fidelity aerodynamic dataset for blended-wing-body (BWB) aircraft.
It contains \(999\) unique BWB geometries, each simulated under multiple flight conditions, resulting in \(8{,}830\) successfully converged CFD cases.
For each geometry--condition pair, the dataset provides integrated force coefficients, including the drag and lift coefficients \((C_d, C_l)\), computed using FUN3D with steady RANS.

\begin{table}[!h]
\centering
\caption{Summary of BlendedNet and the aerodynamic quantities used in our aircraft experiments.}
\small
\setlength{\tabcolsep}{3.0pt}
\begin{tabular}{l l}
\toprule
\textbf{Item} & \textbf{BlendedNet setting used in our experiments} \\
\midrule
Base geometry & Parametric blended-wing-body (BWB) aircraft family \\
Scale & 999 geometries; 8{,}830 converged CFD cases (multiple conditions per geometry) \\
Conditioning variables & Flight-condition parameters (e.g., Mach, AoA, altitude / Reynolds length) \\
CFD solver & FUN3D (steady RANS) \\
Outputs in this paper & Integrated $C_d$, $C_l$ \\
Target in this paper & Minimize $C_d/C_l$ (lower is better) \\
Usage in this paper & Train cost predictor $J_\theta$ \\
\bottomrule
\end{tabular}
\vspace{2pt}
\label{tab:blendednet}
\end{table}

\paragraph{Target.}
For each geometry under a specified flight condition, we use the standard aerodynamic definitions
\begin{equation}
C_d \;\triangleq\; \frac{D}{q_\infty S_{\mathrm{ref}}},
\qquad
C_l \;\triangleq\; \frac{L}{q_\infty S_{\mathrm{ref}}},
\qquad
q_\infty \;\triangleq\; \tfrac{1}{2}\rho_\infty U_\infty^2,
\end{equation}
where \(D\) and \(L\) denote the drag and lift forces, \(S_{\mathrm{ref}}\) is the reference area, and \(\rho_\infty\) and \(U_\infty\) are the freestream density and speed, respectively.
We evaluate aerodynamic performance using the drag-to-lift ratio \(C_d/C_l\) (equivalently maximizing \(L/D\)), and instantiate the cost as
\begin{equation}
J(\mathbf{x}) \;=\; \frac{C_d(\mathbf{x})}{C_l(\mathbf{x})},
\end{equation}
computed under the corresponding flight condition of each sample.

\clearpage
\section{Evaluation protocols}
\label{sec:si_evaluation_protocols}
\label{sec:evalution}
\subsection{Vehicle OpenFOAM setting}
\label{sec:cfd-setting}

\begin{figure*}[!h]
  \centering
  \includegraphics[width=0.78\textwidth]{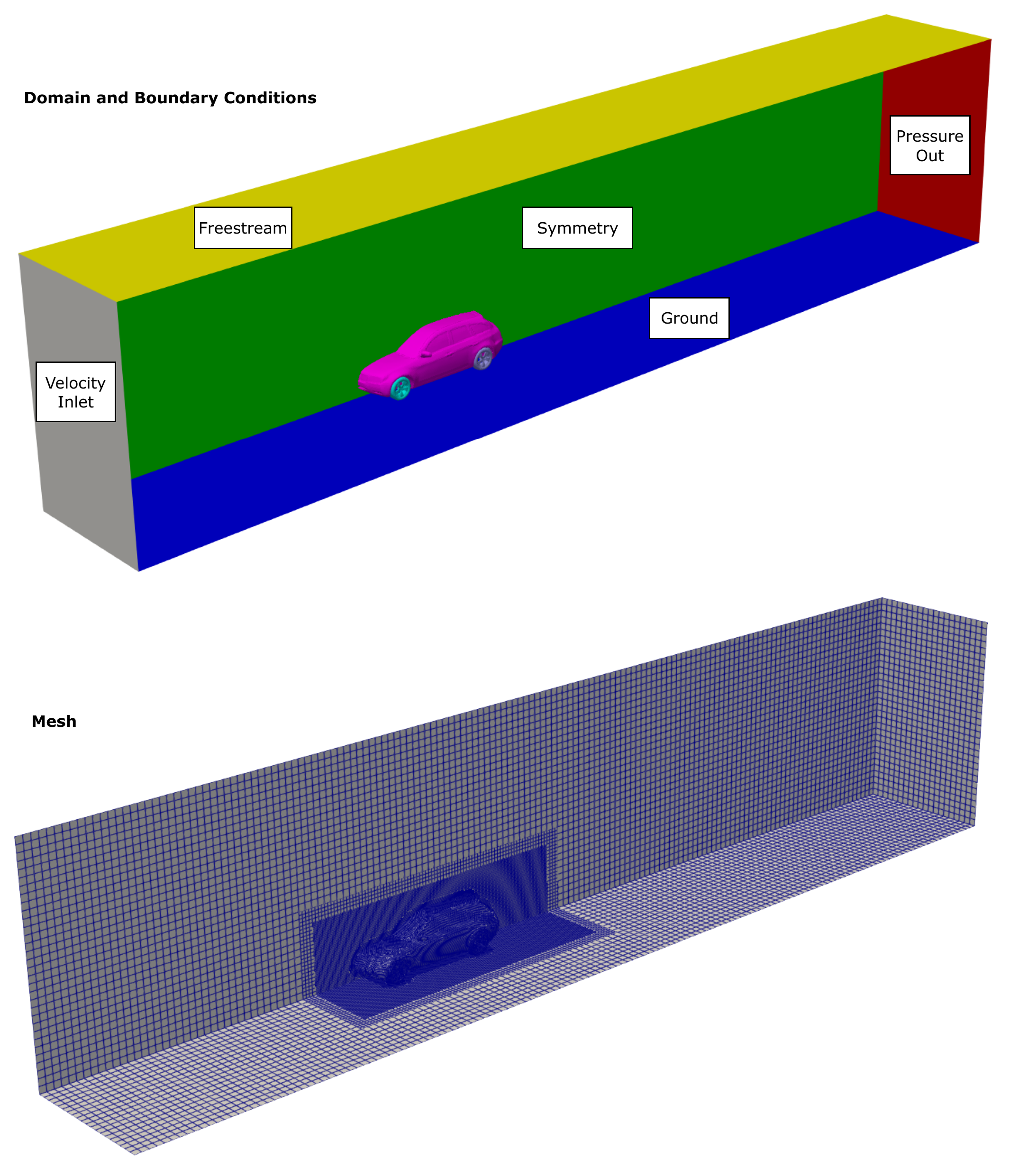}
  \caption{\textbf{CFD setup.}
  \textbf{Top:} computational domain and boundary conditions.
  \textbf{Bottom:} hex-dominant mesh with local refinement around the vehicle and near-wake.
  Boundary labels correspond to OpenFOAM patches: \emph{Velocity Inlet} (\texttt{ffminx}),
  \emph{Pressure Out} (\texttt{ffmaxx}),
  \emph{Symmetry} (\texttt{ffminy}),
  \emph{Freestream / far-field slip} (\texttt{ffmaxy}, \texttt{ffmaxz}),
  and \emph{Moving Ground} (\texttt{ffminz}).}
  \label{fig:driver_mesh}
\end{figure*}

\paragraph{Domain and Boundary Conditions.}
We solve a steady incompressible RANS problem using \texttt{simpleFoam}.
The inflow speed is set to $U_\infty=30~\mathrm{m/s}$.
At the inlet (\texttt{ffminx}), we impose a uniform velocity $\mathbf{U}=(U_\infty,0,0)$.
At the outlet (\texttt{ffmaxx}), we use an \texttt{inletOutlet} condition for $\mathbf{U}$ to robustly handle possible backflow.
A symmetry plane is applied on \texttt{ffminy}, while the remaining far-field boundaries (\texttt{ffmaxy}, \texttt{ffmaxz}) use slip conditions to approximate a freestream.
To mimic a moving-belt setup, the ground patch (\texttt{ffminz}) is prescribed with $\mathbf{U}=(U_\infty,0,0)$.
The vehicle body (\texttt{body2}) uses a no-slip wall with $\mathbf{U}=\mathbf{0}$.
The wheels (\texttt{ruotaant}, \texttt{ruotapost}) are modeled via \texttt{rotatingWallVelocity} with angular speed $\omega=-94~\mathrm{rad/s}$ about the lateral axis, using the wheel centers specified in the case files.
Aerodynamic coefficients are computed using \texttt{forceCoeffs} on the body, each wheel, and their union.

\begin{table}[H]
  \centering
  \caption{Boundary conditions for the velocity field $\mathbf{U}$.}
  \label{tab:cfd_bc_u}
  \setlength{\tabcolsep}{6pt}
  \begin{tabular}{lll}
    \toprule
    Patch & Physical meaning & $\mathbf{U}$ boundary condition \\
    \midrule
    \texttt{ffminx} & Velocity inlet & \texttt{fixedValue} $(U_\infty,0,0)$ \\
    \texttt{ffmaxx} & Pressure outlet & \texttt{inletOutlet} (inletValue $(U_\infty,0,0)$) \\
    \texttt{ffminy} & Symmetry plane & \texttt{symmetry} \\
    \texttt{ffmaxy} & Far-field & \texttt{slip} \\
    \texttt{ffminz} & Moving ground & \texttt{fixedValue} $(U_\infty,0,0)$ \\
    \texttt{ffmaxz} & Far-field & \texttt{slip} \\
    \texttt{body2} & Car body & \texttt{fixedValue} $(0,0,0)$ \\
    \texttt{ruotaant/ruotapost} & Wheels & \texttt{rotatingWallVelocity} ($\omega=-94~\mathrm{rad/s}$) \\
    \bottomrule
  \end{tabular}
\end{table}

\paragraph{Meshing.}
As shown in Fig.~\ref{fig:driver_mesh}, we employ a hex-dominant background mesh and apply a locally refined region enclosing the vehicle and its near-wake to better resolve separation and vortical structures.
Cell sizes transition smoothly from the refined block to the far-field boundaries to control computational cost while maintaining stability.
For diagnostics and post-processing, we additionally output standard quantities including $y^+$, $Q$-criterion, and wall shear stress.

\begin{figure*}[!h]
  \centering
  \includegraphics[width=0.78\textwidth]{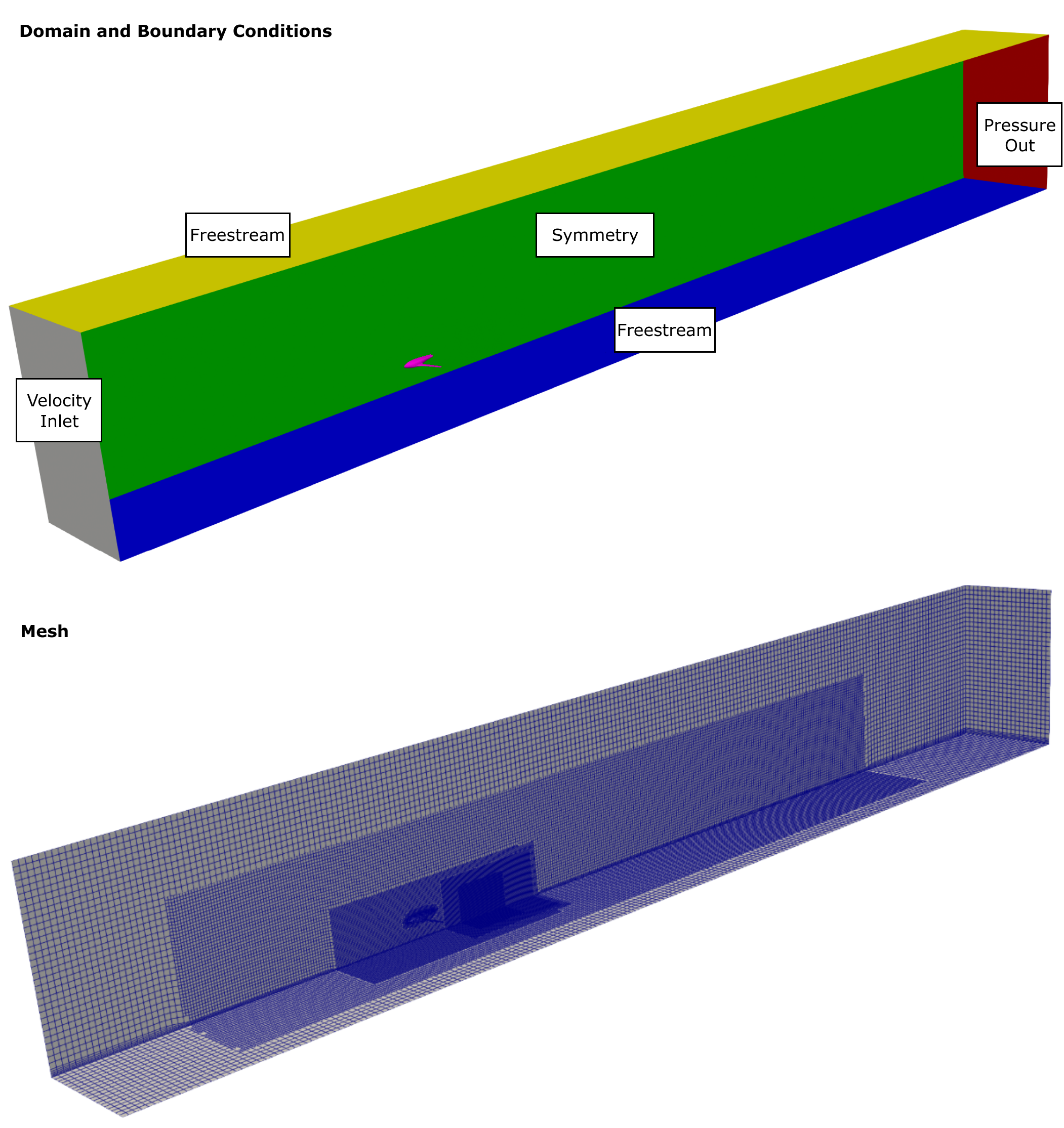}
  \caption{\textbf{CFD setup for the Blender car.}
  \textbf{Top:} computational domain and boundary conditions: \emph{Velocity Inlet} (\texttt{ffminx}),
  \emph{Pressure Out} (\texttt{ffmaxx}),
  \emph{Symmetry} (\texttt{ffminy}),
  and \emph{Freestream} slip walls (\texttt{ffmaxy}, \texttt{ffminz}, \texttt{ffmaxz}).
  \textbf{Bottom:} hex-dominant mesh with local refinement around the vehicle and near-wake.}
  \label{fig:blender_mesh}
\end{figure*}

\subsection{Blender-car OpenFOAM setting}
\label{sec:blender-cfd-setting}

\paragraph{Domain and Boundary Conditions.}
We run steady incompressible RANS simulations using \texttt{simpleFoam}.
The freestream speed is set to $|\mathbf{U}_\infty|=50~\mathrm{m/s}$.
To model an inclined inflow, we prescribe a uniform inlet velocity
$\mathbf{U}_\infty=(49.2404,\,0,\,8.6824)~\mathrm{m/s}$, corresponding to a $10^\circ$ incidence in the $x$--$z$ plane.
At the inlet (\texttt{ffminx}), $\mathbf{U}$ is enforced by a fixed-value condition.
At the outlet (\texttt{ffmaxx}), we use \texttt{inletOutlet} for $\mathbf{U}$ to improve robustness under possible backflow.
A symmetry plane is imposed on \texttt{ffminy}, and the remaining far-field boundaries (\texttt{ffmaxy}, \texttt{ffminz}, \texttt{ffmaxz}) use slip conditions to approximate a freestream.
The car surface (\texttt{body2}) is treated as a no-slip wall with $\mathbf{U}=\mathbf{0}$.
We compute aerodynamic coefficients using \texttt{forceCoeffs} on \texttt{body2}, with drag along $(1,0,0)$ and lift along $(0,0,1)$.

\begin{table}[H]
  \centering
  \caption{Boundary conditions for the velocity field $\mathbf{U}$ in the Blender car setup.}
  \label{tab:blender_bc_u}
  \setlength{\tabcolsep}{6pt}
  \begin{tabular}{lll}
    \toprule
    Patch & Physical meaning & $\mathbf{U}$ boundary condition \\
    \midrule
    \texttt{ffminx} & Velocity inlet & \texttt{fixedValue} $(49.2404,0,8.6824)$ \\
    \texttt{ffmaxx} & Pressure outlet & \texttt{inletOutlet} (inletValue $(49.2404,0,8.6824)$) \\
    \texttt{ffminy} & Symmetry plane & \texttt{symmetry} \\
    \texttt{ffmaxy} & Freestream boundary & \texttt{slip} \\
    \texttt{ffminz} & Freestream boundary & \texttt{slip} \\
    \texttt{ffmaxz} & Freestream boundary & \texttt{slip} \\
    \texttt{body2}  & blender body & \texttt{fixedValue} $(0,0,0)$ \\
    \bottomrule
  \end{tabular}
\end{table}

\paragraph{Meshing.}
As shown in Fig.~\ref{fig:blender_mesh}, we employ a hex-dominant background mesh and add a locally refined region that encloses the vehicle and its near-wake.
The mesh resolution is increased around the body to better capture separation and wake development, while coarser cells are used in the far field to control cost.
For analysis and visualization, we enable additional function objects, including $y^+$, $Q$-criterion, and wall shear stress, and compute time-averaged statistics after an initial transient.

\subsection{3D-printed car and miniature wind tunnel testing}
\label{sec:wind}

\begin{figure*}[!h]
  \centering
  \includegraphics[width=1\textwidth]{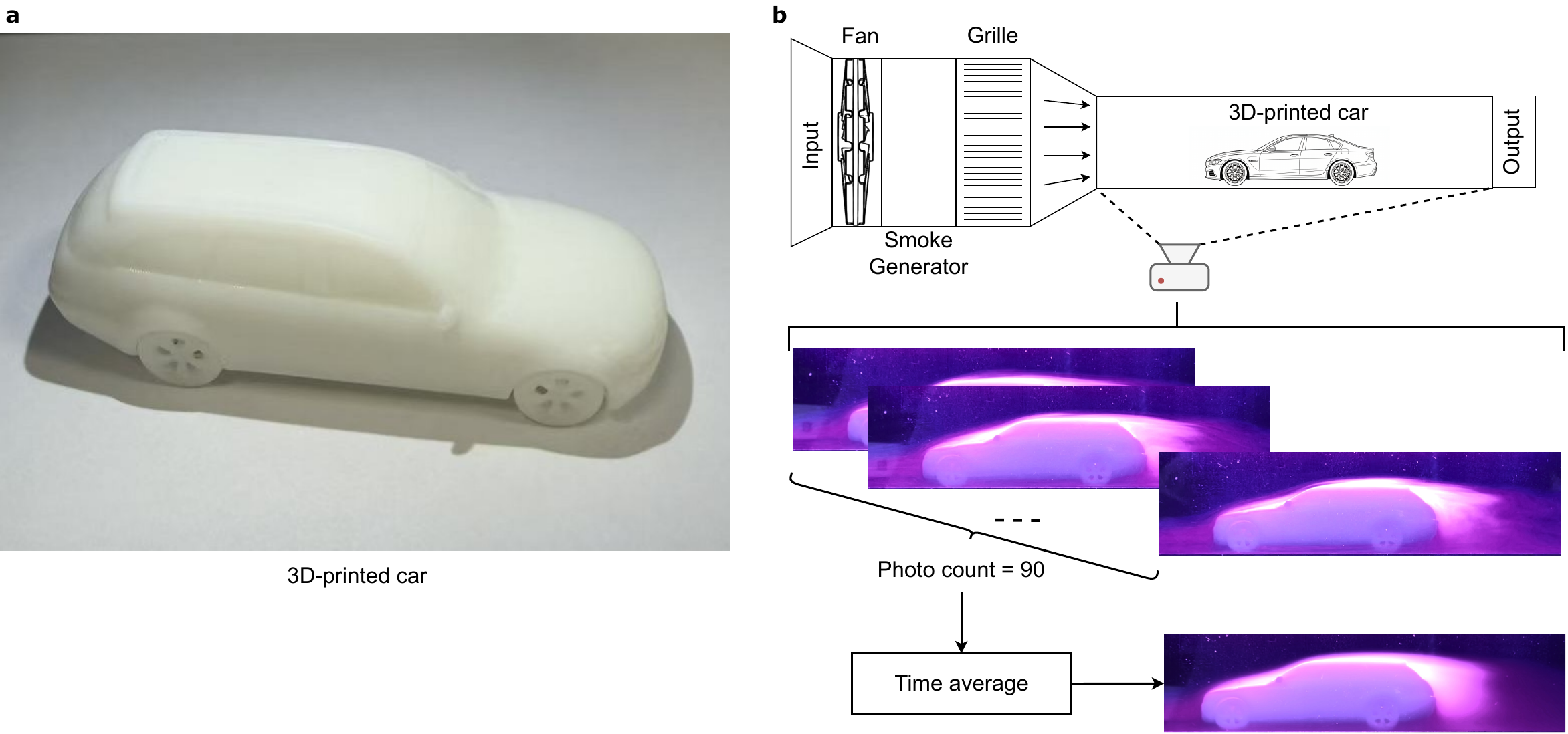}
  \caption{\textbf{Physical evaluation with a miniature wind tunnel.}
  (\textbf{a}) An example 3D-printed car used for testing.
  (\textbf{b}) Experimental pipeline: the fan-driven flow is seeded by a smoke generator and conditioned by a grille/flow straightener before entering the test section containing the 3D-printed car.
  For each design, we capture a sequence of smoke-visualization images from a fixed camera view (\emph{photo count} $=90$) and report both representative snapshots and their time-averaged image for comparison.}
  \label{fig:3d_print}
\end{figure*}

We complement CFD-based evaluation with physical tests in a miniature wind tunnel using 3D-printed vehicles.
For each design (before/after optimization, and generations with/without guidance), we fabricate a rigid scaled model and place it at a fixed location in the test section.
As illustrated in Fig.~\ref{fig:3d_print}b, the incoming flow is produced by a fan, seeded by a smoke generator, and then passes through a grille to reduce large-scale non-uniformity before interacting with the vehicle.
We record smoke-visualization images from a fixed viewpoint under consistent lighting.
To suppress instantaneous fluctuations and measurement noise, we acquire 90 frames per design and compute a time-averaged image, which serves as our primary qualitative indicator of the mean wake structure.

\clearpage
\section{Additional analysis of experimental results}
\label{sec:si_experimental_analysis}

\subsection{Why do guided methods with an MSE-trained predictor exhibit a consistent drift?}
\label{whyMSE}

\begin{figure*}[!h]
  \centering
  \includegraphics[width=0.86\textwidth]{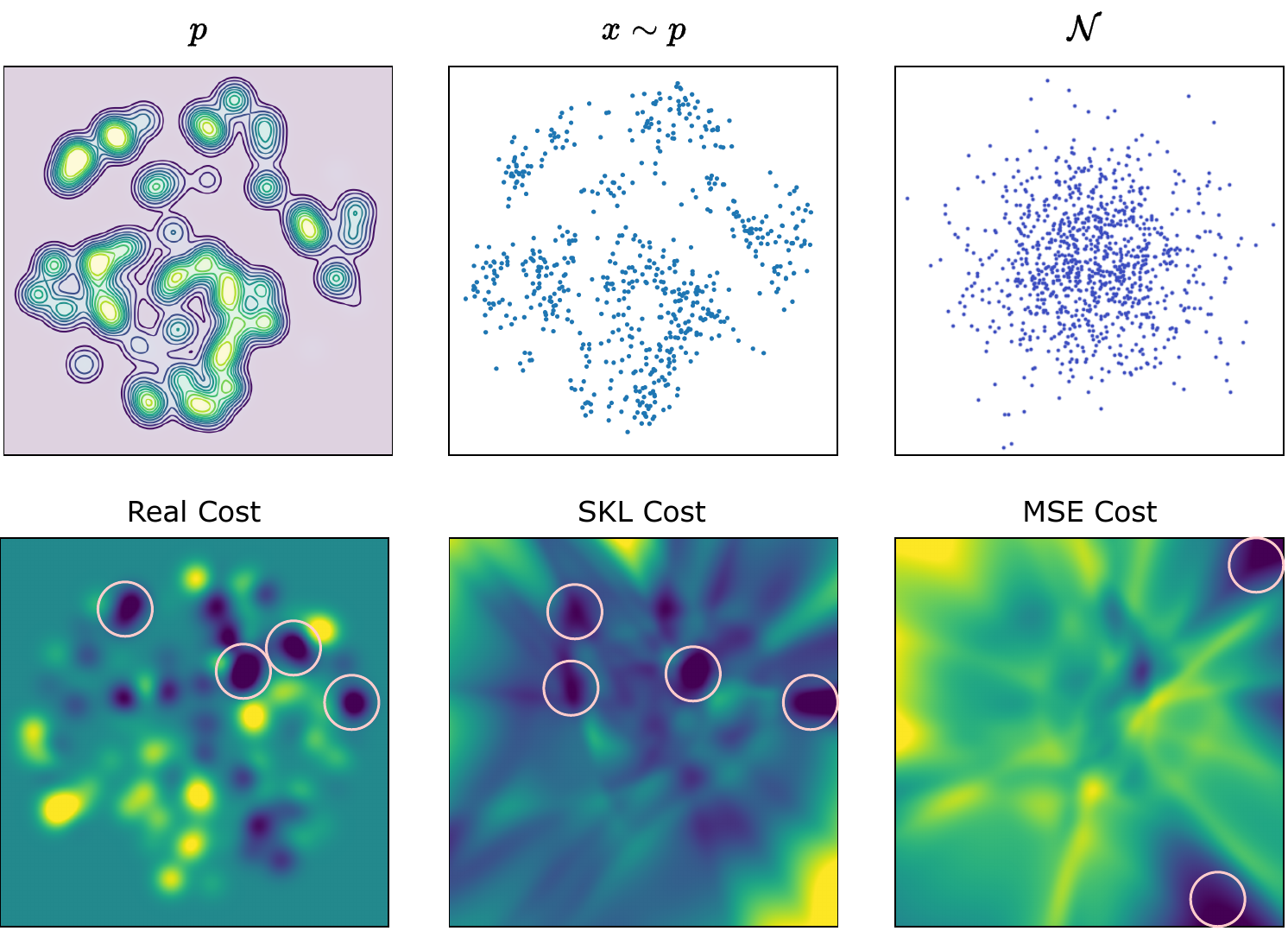}
  \caption{\textbf{Out-of-support generalization explains the drift of MSE-based guidance.}
  \textbf{Top:} the data density $p$ (contours), samples $x\sim p$, and the base noise distribution $\mathcal{N}$ used to initialize generation.
  \textbf{Bottom:} cost landscapes over the \emph{entire} domain: ground-truth (\emph{Real Cost}), SKL-trained predictor (\emph{SKL Cost}), and MSE-trained predictor (\emph{MSE Cost}).
  All cost maps are normalized for visualization, and circles mark the lowest few minima (global low-cost basins) in each map.}
  \label{fig:why_mse}
\end{figure*}

When training the cost predictor, supervision is only available on samples drawn from the data distribution $p$ (top row), i.e., the predictor is learned on the support of $p$.
Consequently, the learned predictor is under-constrained outside the data support and may extrapolate arbitrarily in low-density regions.
This distinction becomes crucial for guided generation: sampling is initialized from noise $\mathcal{N}$ and thus inevitably traverses regions far outside the support of $p$.
Therefore, effective guidance requires the predictor to match the \emph{global} low-cost basins of the true cost landscape over the full domain.

The second row visualizes this effect by plotting the cost fields over the entire space and highlighting the lowest few minima after normalization.
The SKL-trained predictor better preserves the locations of these global minima compared with the ground-truth cost, while the MSE-trained predictor produces mismatched or spurious minima in out-of-support regions.
As a result, guidance computed from the MSE predictor tends to pull trajectories consistently toward incorrect low-cost basins, manifesting as a systematic drift during generation.

\clearpage
\section{Additional reinforcement-learning experiments}
\label{sec:si_offline_rl}

\subsection{Experiment setting}
\label{sec:si_offline_rl_setting}
We follow the offline planning setup in \cite{feng2025guidance} to benchmark training-free guidance on offline reinforcement learning tasks.
Concretely, we use the D4RL locomotion suite (HalfCheetah, Hopper, Walker2d) under standard dataset regimes (Medium, Medium-Expert, and Medium-Replay), where demonstrations are non-expert mixtures or replay buffers and the goal is to recover high-return behavior from offline data. \cite{fu2020d4rl, janner2022planning}.

\paragraph{Planner formulation.}
Our evaluation adopts the Diffuser-style generative planning pipeline: a generative model proposes a length-$H$ state--action sequence, and a critic estimates the discounted return-to-go of the proposed plan; guidance then biases sampling toward higher predicted returns, i.e., sampling from an energy-weighted distribution proportional to $p(\tau)\exp(R(\tau))/Z$ \cite{levine2018reinforcement}.

\paragraph{Generative model and critic.}
The base planner is a conditional flow-matching model (CFM) or mini-batch OT-CFM with an affine path $\alpha_t=t,\beta_t=1-t$, implemented with an 8-layer Transformer backbone of hidden size 256.
It is trained for $10^{5}$ steps with batch size 32, learning rate $2\times10^{-4}$, and cosine-annealing scheduling.
The critic shares the same backbone and uses the last token as the value head; it is trained for $10^{4}$ steps with batch size 64 and learning rate $2\times10^{-4}$.
The discount factor is fixed to 0.99. \textbf{In experiments, we use the pretrained model checkpoints released by} \cite{feng2025guidance}.

\paragraph{Planning protocol.}
We plan with horizon $H=20$ and stride 1, and intentionally disable commonly used tricks to isolate the effect of guidance.
For a fair comparison, all guidance methods share the same pretrained base generative model during evaluation. \cite{feng2025guidance}.

\paragraph{Metric and hyperparameter tuning.}
We report D4RL normalized scores, where 100 corresponds to the expert policy return. \cite{fu2020d4rl} \cite{feng2025guidance}.
Hyperparameters are tuned per method following \cite{feng2025guidance}: covariance-gradient guidances sweep multiple $\lambda_t$ schedules (e.g., constant/decay families) and global scales; $g_{\text{MC}}$ sweeps the reward scaling and uses up to 128 Monte Carlo samples with a small numerical stabilizer; $g_{\text{sim-MC}}$ sweeps reward scaling and an assumed conditional standard deviation, and additionally tunes a lightweight schedule/scale on the estimated guidance. \cite{feng2025guidance}.

\paragraph{Baselines and Hyperparameter Tuning.}
We adopt the same four baselines as in the main text, including DPS~\cite{chung2022diffusion}, LGD-MC~\cite{song2023loss}, SIM-MC~\cite{feng2025guidance}, and our method SA-MC. Following \cite{feng2025guidance}, we perform a small hyperparameter sweep over: (i) the flow-matching variant (\texttt{cfm} or \texttt{ot\_cfm}); (ii) the number of Monte Carlo samples ($16, 64, 128$); (iii) the guidance scale ($0.1, 1$); and (iv) the guidance schedule (original or linear decay). Here \emph{original} refers to the default schedule used in each method. For DPS and LGD-MC, the schedule is ${b_t\sigma^2}/{\alpha_t}$, which equals $(1-t)/t$ in practice. For SIM-MC and SA-MC, the schedule is $b_t$, which equals $1/(1-t)$ in practice. The \emph{linear decay} schedule is set to $1-t$.

\begin{table}[!h]
\centering
\small
\setlength{\tabcolsep}{4pt}
\caption{\textbf{Offline RL results on D4RL locomotion.} 
We report the best normalized score (mean $\pm$ std.) for each method under our hyperparameter sweep described in Supplementary Section~\ref{sec:si_offline_rl_setting}. 
The last row reports the average over all nine tasks. 
Bold indicates the best result and underline indicates the second best.}
\label{tab:all}
\begin{tabular}{llccccc}
\toprule
 &  & \multicolumn{5}{c}{Overall} \\
\cmidrule(lr){3-7}
 &  & w/o $g$ & DPS & LGD-MC & SIM-MC & SA-MC \\
\midrule

\multirow{3}{*}{Medium-Expert} & HalfCheetah & 62.3 $\pm$ 17.9 & 41.8 $\pm$ 17.0 & \underline{66.5  $\pm$ 0.22} & 60.6 $\pm$ 16.5& \textbf{71.0 $\pm$ 14.6}\\
 & Hopper   & 82.8 $\pm$ 22.3 & 99.1 $\pm$ 13.3 & 100.4 $\pm$ 18.5 &  \underline{101.7 $\pm$ 18.1} & \textbf{102.1 $\pm$ 17.9} \\
 & Walker2d & 84.6 $\pm$ 12.6 & \textbf{107.5 $\pm$ 0.5} & \underline{103.2 $\pm$ 4.7} & 96.1 $\pm$ 12.7 & 97.7 $\pm$ 10.8 \\
\midrule
 
\multirow{3}{*}{Medium} & HalfCheetah & 42.5 $\pm$ 1.0 & \textbf{42.8 $\pm$ 1.1}& 42.6 $\pm$ 1.0 & \underline{42.7 $\pm$ 1.2} &42.4 $\pm$ 0.9\\
 & Hopper   &  76.2 $\pm$ 10.7 & \textbf{90.5 $\pm$ 14.0} &  74.6 $\pm$ 11.2 &  77.2 $\pm$ 15.9 & \underline{80.7 $\pm$ 11.4} \\
 & Walker2d & 68.8 $\pm$ 10.8 &  \textbf{79.7 $\pm$ 7.5} & 76.1 $\pm$ 10.7 &77.4 $\pm$ 2.8 & \underline{78.3 $\pm$ 4.1} \\
\midrule

\multirow{3}{*}{Medium-Replay} & HalfCheetah & 34.9 $\pm$ 3.0 & 31.6 $\pm$ 14.6 & \textbf{38.3 $\pm$ 1.5} & {36.0 $\pm$ 2.2}  & \underline{36.2 $\pm$ 1.7} \\
 & Hopper   & 44.1 $\pm$ 6.7 & \textbf{66.4 $\pm$ 6.0} & 55.1 $\pm$ 11.1 & 53.6 $\pm$ 15.6 &  \underline{57.6 $\pm$ 8.6} \\
 & Walker2d & 50.8 $\pm$ 28.3 & \underline{67.1 $\pm$ 24.3} & \textbf{70.4 $\pm$ 5.4} &  56.0 $\pm$ 18.7 & 62.1 $\pm$ 20.0 \\
\midrule

\multicolumn{2}{c}{Average} & 60.8 & 69.6 & \underline{69.7} & 66.8 & \textbf{69.8} \\
\bottomrule
\end{tabular}
\end{table}

\begin{table}[!h]
\centering
\small
\setlength{\tabcolsep}{4pt}
\caption{\textbf{OT-CFM with original schedule.}
D4RL normalized scores (mean $\pm$ std.) for each dataset and environment.
``w/o $g$'' denotes no guidance.
``--'' indicates runs that produced invalid scores under the corresponding setting.}
\label{tab:1}
\begin{tabular}{llccccc}
\toprule
 &  & \multicolumn{5}{c}{OT-CFM, Original Schedule} \\
\cmidrule(lr){3-7}
 &  & w/o $g$ & DPS & LGD-MC & SIM-MC & SA-MC \\
\midrule

\multirow{3}{*}{Medium-Expert} & HalfCheetah & 62.3 $\pm$ 17.9 & - & - & 60.6 $\pm$ 16.5& 71.0 $\pm$ 14.6\\
 & Hopper   & 82.8 $\pm$ 22.3 & - & 79.5 $\pm$ 19.4 &  101.0 $\pm$ 17.9 & 100.7 $\pm$ 18.7 \\
 & Walker2d & 84.6 $\pm$ 12.6 & - & - & 86.0 $\pm$ 20.7 & 87.0 $\pm$ 16.1 \\
\midrule

\multirow{3}{*}{Medium} & HalfCheetah & 41.3 $\pm$ 0.5 & 41.6 $\pm$ 0.1& 41.8 $\pm$ 0.6 & 41.4 $\pm$ 0.9 & 41.9 $\pm$ 0.3 \\
 & Hopper   &  54.4 $\pm$ 8.9 & - &  55.3 $\pm$ 4.5 &  53.0 $\pm$ 6.5 & 58.1 $\pm$ 6.0 \\
 & Walker2d & 68.8 $\pm$ 10.8 & - & - &66.6$\pm$ 15.7 & 77.0 $\pm$ 9.0 \\
\midrule

\multirow{3}{*}{Medium-Replay} & HalfCheetah & 22.4 $\pm$ 5.0 & - & 23.5 $\pm$ 6.4 & 24.4 $\pm$ 6.2 & 24.5 $\pm$ 8.0 \\
 & Hopper   & 44.1 $\pm$ 6.7 & - & - & 47.7 $\pm$ 3.4 &  55.3 $\pm$ 14.3 \\
 & Walker2d & 31.0 $\pm$ 2.0 & - & - &  39.3 $\pm$ 20.6 & 40.3 $\pm$ 20.1 \\
\bottomrule
\end{tabular}
\end{table}

\begin{table}[!h]
\centering
\small
\setlength{\tabcolsep}{4pt}
\caption{\textbf{CFM with original schedule.}
D4RL normalized scores (mean $\pm$ std.) for each dataset and environment.
``w/o $g$'' denotes no guidance.
``--'' indicates runs that produced invalid scores under the corresponding setting.}
\label{tab:2}
\begin{tabular}{llccccc}
\toprule
 &  & \multicolumn{5}{c}{CFM, Original Schedule} \\
\cmidrule(lr){3-7}
 &  & w/o $g$ & DPS & LGD-MC & SIM-MC & SA-MC \\
\midrule

\multirow{3}{*}{Medium-Expert} & HalfCheetah & 44.9 $\pm$ 3.6 & - & - & 57.5 $\pm$ 20.2& 51.6 $\pm$ 18.9\\
 & Hopper   & 61.1 $\pm$ 13.1 & - & 91.7 $\pm$ 17.4 &  91.1 $\pm$ 23.4 & 91.9 $\pm$ 24.7 \\
 & Walker2d & 67.5 $\pm$ 29.5 & - & - & 90.2 $\pm$ 23.9 & 94.1 $\pm$ 15.5 \\
\midrule

\multirow{3}{*}{Medium} & HalfCheetah & 42.5 $\pm$ 1.0 & 42.6 $\pm$ 0.4& 41.8 $\pm$ 0.7 & 41.9 $\pm$ 0.7 & 42.4 $\pm$ 0.9 \\
 & Hopper   &  76.2 $\pm$ 10.7 & - &  71.6 $\pm$ 10.6 &  72.4 $\pm$ 3.9 &  80.7 $\pm$ 11.4 \\
 & Walker2d & 68.7 $\pm$ 23.3 & - & 68.6 $\pm$ 14.5 & 71.9 $\pm$ 6.8 & 78.3 $\pm$ 4.1 \\
\midrule

\multirow{3}{*}{Medium-Replay} & HalfCheetah & 34.9 $\pm$ 3.0 & - & 33.1 $\pm$ 5.5 & 30.2 $\pm$ 13.6 & 34.9 $\pm$ 4.4 \\
 & Hopper   & 39.4 $\pm$ 2.6 & - & 48.0 $\pm$ 13.9 & 53.6 $\pm$ 15.6 &  52.6 $\pm$ 13.7 \\
 & Walker2d & 50.8 $\pm$ 28.3 & - & 35.6 $\pm$ 8.8 & 56.0 $\pm$ 18.7 & 57.8 $\pm$ 19.9 \\
\bottomrule
\end{tabular}
\end{table}

\begin{table}[!h]
\centering
\small
\setlength{\tabcolsep}{4pt}
\caption{\textbf{OT-CFM with linear-decay schedule.}
D4RL normalized scores (mean $\pm$ std.) for each dataset and environment.
``w/o $g$'' denotes no guidance.}
\label{tab:3}
\begin{tabular}{llccccc}
\toprule
 &  & \multicolumn{5}{c}{OT-CFM, Linear Decay Schedule} \\
\cmidrule(lr){3-7}
 &  & w/o $g$ & DPS & LGD-MC & SIM-MC & SA-MC \\
\midrule

\multirow{3}{*}{Medium-Expert} & HalfCheetah & 62.3 $\pm$ 17.9 & - & 64.5 $\pm$ 10.6 & 58.3 $\pm$ 18.5& 61.8 $\pm$ 16.0\\
 & Hopper   & 82.8 $\pm$ 22.3 & 85.8 $\pm$ 16.2 & 100.4 $\pm$ 18.5 &  101.7 $\pm$ 18.1 & 102.1 $\pm$ 17.9 \\
 & Walker2d & 84.6 $\pm$ 12.6 & 107.5 $\pm$ 0.5 & 93.1 $\pm$ 10.8 & 85.7 $\pm$ 20.2 & 86.1 $\pm$ 14.0 \\
\midrule

\multirow{3}{*}{Medium} & HalfCheetah & 41.3 $\pm$ 0.5 & 42.1 $\pm$ 0.7& 41.8 $\pm$ 0.6 & 41.5 $\pm$ 0.7 & 41.8 $\pm$ 1.2 \\
 & Hopper   &  54.4 $\pm$ 8.9 & 66.5 $\pm$ 1.7 &  68.3 $\pm$ 6.3 &  57.9 $\pm$ 11.6 & 57.3 $\pm$ 6.2 \\
 & Walker2d & 68.8 $\pm$ 10.8 & 75.4 $\pm$ 9.5 & 76.1 $\pm$ 10.7 &69.2$\pm$ 12.6 & 67.9 $\pm$ 10.3 \\
\midrule

\multirow{3}{*}{Medium-Replay} & HalfCheetah & 22.4 $\pm$ 5.0 & 24.6 $\pm$ 9.2 & 28.1 $\pm$ 4.4 & 22.1 $\pm$ 3.2 & 25.5 $\pm$ 5.2 \\
 & Hopper   & 44.1 $\pm$ 6.7 & 66.4 $\pm$ 6.0 & 49.8 $\pm$ 10.3& 52.3 $\pm$ 6.2 &  57.6 $\pm$ 8.6 \\
 & Walker2d & 31.0 $\pm$ 2.0 & 67.1 $\pm$ 24.3 & 43.5 $\pm$ 16.6 &  40.1 $\pm$ 19.2 & 45.1 $\pm$ 15.5 \\
\bottomrule
\end{tabular}
\end{table}

\begin{table}[!h]
\centering
\small
\setlength{\tabcolsep}{4pt}
\caption{\textbf{CFM with linear-decay schedule.}
D4RL normalized scores (mean $\pm$ std.) for each dataset and environment.
``w/o $g$'' denotes no guidance.}
\label{tab:4}
\begin{tabular}{llccccc}
\toprule
 &  & \multicolumn{5}{c}{CFM, Linear Decay Schedule} \\
\cmidrule(lr){3-7}
 &  & w/o $g$ & DPS & LGD-MC & SIM-MC & SA-MC \\
\midrule

\multirow{3}{*}{Medium-Expert} & HalfCheetah & 44.9 $\pm$ 3.6 & 41.8 $\pm$ 17.0 & 66.5  $\pm$ 0.22& 51.9 $\pm$ 19.7& 62.4 $\pm$ 24.0\\
 & Hopper   & 61.1 $\pm$ 13.1 & 99.1 $\pm$ 13.3 & 98.8 $\pm$ 18.6 &  94.4 $\pm$ 21.7 & 91.8 $\pm$ 19.0 \\
 & Walker2d & 67.5 $\pm$ 29.5 & 103.2 $\pm$ 7.6 & 103.2 $\pm$ 4.7 & 96.1 $\pm$ 12.7 & 97.7 $\pm$ 10.8 \\
\midrule

\multirow{3}{*}{Medium} & HalfCheetah & 42.7 $\pm$ 1.5 & 42.8 $\pm$ 1.1& 42.6 $\pm$ 1.0 & 42.7 $\pm$ 1.2 & 42.2 $\pm$ 0.7 \\
 & Hopper   &  76.2 $\pm$ 10.7 & 90.5 $\pm$ 14.0 &  74.6 $\pm$ 11.2 &  77.2 $\pm$ 15.9 &  74.9 $\pm$ 4.0 \\
 & Walker2d & 68.7 $\pm$ 23.3 & 79.7 $\pm$ 7.5 & 72.5$\pm$ 12.6 & 77.4 $\pm$ 2.8 & 77.5 $\pm$ 4.4 \\
\midrule

\multirow{3}{*}{Medium-Replay} & HalfCheetah & 34.9 $\pm$ 3.0 & 31.6 $\pm$ 14.6 & 38.3 $\pm$ 1.5 & 36.0 $\pm$ 2.2 & 36.2 $\pm$ 1.7 \\
 & Hopper   & 39.4 $\pm$ 2.6 & 49.5 $\pm$ 11.8 & 55.1 $\pm$ 11.1 & 52.2 $\pm$ 7.6 &  55.7 $\pm$ 19.5 \\
 & Walker2d & 50.8 $\pm$ 28.3 & 66.2 $\pm$ 27.8& 70.4 $\pm$ 5.4 & 50.9 $\pm$ 21.4 & 62.1 $\pm$ 20.0 \\
\bottomrule
\end{tabular}
\end{table}

\begin{figure*}[t]
  \centering
  \includegraphics[width=0.55\textwidth]{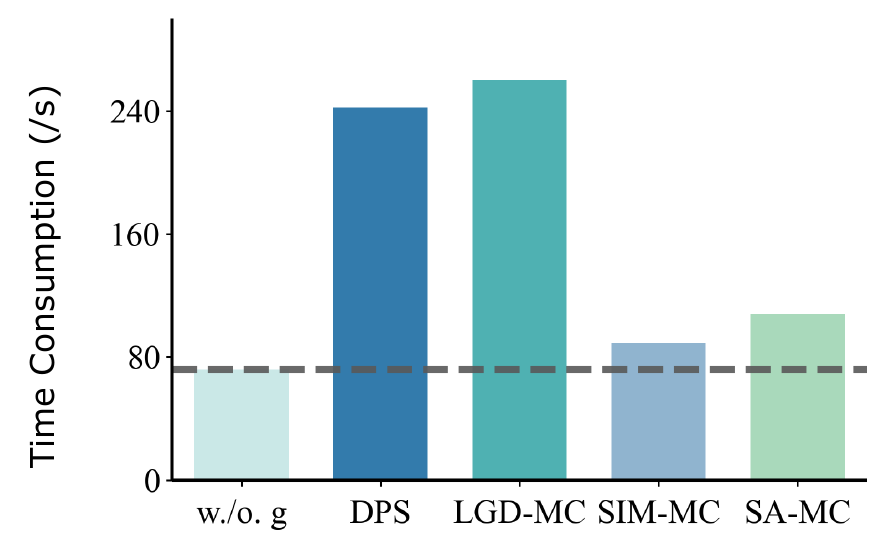}
  \caption{\textbf{Runtime overhead of guidance on D4RL locomotion.}
  Wall-clock time (seconds) to generate a 1{,}000-step trajectory.
  The dashed line denotes the unconditional baseline (w/o $g$). Lower is better.}
  \label{fig:offlin_RL_time}
\end{figure*}

\subsection{Results and analysis}

Following the experimental setting above, we sweep the hyperparameters listed in Supplementary Section~\ref{sec:si_offline_rl_setting} and report, for each method, the \emph{best} score achieved on each task under this sweep.
Table~\ref{tab:all} summarizes the resulting best-performance comparison across all configurations.
Overall, SA-MC achieves the best average normalized score among the compared methods.

Beyond the mean score, SA-MC exhibits stronger \emph{practical reliability}.
As shown by the detailed breakdown in Tables~\ref{tab:1}--\ref{tab:4}, DPS and LGD-MC can be sensitive to the choice of schedule and scale.
In particular, under the method-specific original schedules, DPS and LGD-MC may become numerically unstable for some settings and fail to return a valid evaluation score; such failed cases are marked by ``--'' in the tables.
Switching to the linear-decay schedule alleviates this issue in several tasks, but does not eliminate it entirely, indicating that these methods remain fragile under certain hyperparameter choices.
In contrast, SIM-MC and SA-MC yield valid results more consistently across the same sweep, and SA-MC maintains competitive or leading performance across datasets.

For clarity, Tables~\ref{tab:1} and \ref{tab:3} report results using OT-CFM under the original and linear-decay schedules, respectively.
Tables~\ref{tab:2} and \ref{tab:4} report the analogous results using CFM.
In all tables, ``w/o $g$'' denotes unconditional sampling without guidance, and reported values are mean $\pm$ standard deviation over evaluation runs.

In addition to accuracy and robustness, SA-MC is also more efficient in wall-clock time.
Figure~\ref{fig:offlin_RL_time} compares the time required to generate a 1{,}000-step trajectory on D4RL locomotion tasks.
Compared with the substantial overhead of DPS and LGD-MC relative to unconditional sampling, SA-MC incurs only a modest additional cost (similar to SIM-MC).
This efficiency becomes particularly important when scaling to large hyperparameter sweeps and extensive evaluations.

\end{document}